\documentclass[onefignum,onetabnum]{siamonline220329}

\usepackage{lipsum}
\usepackage{amsmath}
\usepackage{amsfonts}
\usepackage{graphicx}
\usepackage{epstopdf}
\usepackage{algorithmic}

\ifpdf
  \DeclareGraphicsExtensions{.eps,.pdf,.png,.jpg}
\else
  \DeclareGraphicsExtensions{.eps}
\fi

\usepackage[inline]{enumitem}
\setlist[enumerate]{leftmargin=.5in}
\setlist[itemize]{leftmargin=.5in}

\newsiamremark{remark}{Remark}
\newsiamremark{example}{Example}
\newsiamremark{hypothesis}{Hypothesis}
\crefname{hypothesis}{Hypothesis}{Hypotheses}
\newsiamthm{claim}{Claim}

\headers{Random Fourier Signature Features}{Cs. Tóth, H. Oberhauser, and Z. Szabó}

\title{Random Fourier Signature  Features
}

\author{Csaba T{\'o}th\thanks{Mathematical Institute, University of Oxford 
  (\texttt{\href{mailto:toth@maths.ox.ac.uk}{toth@maths.ox.ac.uk}}, \texttt{\href{mailto:oberhauser@maths.ox.ac.uk}{oberhauser@maths.ox.ac.uk}}).
}
\and Harald Oberhauser\footnotemark[2]
\and Zolt{\'a}n Szab{\'o}\thanks{Department of Statistics, London School of Economics (\texttt{\href{mailto:z.szabo@lse.ac.uk}{z.szabo@lse.ac.uk}}).}}

\usepackage{amsopn}

\DeclareMathOperator{\med}{med}

\usepackage[utf8]{inputenc} 
\usepackage{hyperref}       
\usepackage{url}            
\usepackage{booktabs}       
\usepackage{amsfonts}       
\usepackage{nicefrac}       
\usepackage{microtype}      
\usepackage{xcolor}         

\usepackage[T1]{fontenc}
\usepackage{inconsolata}

\usepackage{amsmath}
\usepackage{amssymb}
\usepackage{mathtools}
\usepackage{bm}
\usepackage{todonotes}
\usepackage[numbers]{natbib}
\usepackage{cases}
\usepackage{longtable}
\usepackage{booktabs}
\usepackage{cancel}
\usepackage{enumitem}
\usepackage[mathcal]{eucal}
\usepackage{adjustbox}
\usepackage{enumitem}
\usepackage{hyperref}
\usepackage{caption}
\usepackage{algorithm}
\usepackage{algorithmic}
\usepackage{forloop}

\usepackage{url}
\usepackage{tikz-cd} 
\usepackage{wrapfig}
\usepackage{autonum}

\allowdisplaybreaks

\usepackage{etoolbox}

\usepackage{xr-hyper}

\newcommand{\p}{\prime}

\newcommand{\SVM}{\texttt{SVM}}

\newcommand{\RKHS}{\texttt{RKHS}}
\newcommand{\RFF}{\texttt{RFF}}
\newcommand{\RBF}{\texttt{RBF}}
\newcommand{{\RFSF}}{{\texttt{RFSF}}}
\newcommand{\RFSFD}{\texttt{RFSF-DP}}
\newcommand{\RFSFT}{\texttt{RFSF-TRP}}
\newcommand{\KS}{\texttt{KSig}}
\newcommand{\KSP}{\texttt{KSigPDE}}
\newcommand{\GAK}{\texttt{GAK}}
\newcommand{\RWS}{\texttt{RWS}}

\newcommand{\TRP}{\texttt{TRP}}
\newcommand{\CP}{\texttt{CP}}
\newcommand{\pr}{\texttt{Pr}}

\newcommand{\bx}{\mathbf{x}}
\newcommand{\bX}{\mathbf{X}}
\newcommand{\by}{\mathbf{y}}
\newcommand{\bz}{\mathbf{z}}
\newcommand{\bu}{\mathbf{u}}
\newcommand{\bv}{\mathbf{v}}
\newcommand{\bw}{\mathbf{w}}
\newcommand{\bi}{\mathbf{i}}
\newcommand{\bj}{\mathbf{j}}
\newcommand{\bk}{\mathbf{k}}
\newcommand{\bl}{\mathbf{l}}
\newcommand{\ba}{\mathbf{a}}
\newcommand{\bb}{\mathbf{b}}
\newcommand{\bp}{\mathbf{p}}

\newcommand{\be}{\mathbf{e}}
\newcommand{\bq}{\mathbf{q}}

\newcommand{\bs}{\mathbf{s}}
\newcommand{\bt}{\mathbf{t}}

\newcommand{\bW}{\mathbf{W}}
\newcommand{\bP}{\mathbf{P}}
\newcommand{\bA}{\mathbf{A}}

\newcommand{\bh}{\mathbf{h}}

\newcommand{\cI}{\mathcal{I}}

\newcommand{\cX}{\mathcal{X}}
\newcommand{\cM}{\mathcal{M}}
\newcommand{\cN}{\mathcal{N}}
\newcommand{\cH}{\mathcal{H}}

\newcommand{\bbP}{\mathbb{P}}
\newcommand{\bbV}{\mathbb{V}}
\newcommand{\bbZ}{\mathbb{Z}}

\newcommand{\bbR}{\mathbb{R}}
\newcommand{\bbE}{\mathbb{E}}
\newcommand{\bbN}{\mathbb{N}}

\newcommand{\spc}{\hspace{5pt}}

\newcommand{\given}{\, : \,}
\newcommand{\cond}{\, \middle\vert \,}

\newcommand{\norm}[1]{\left\lVert #1 \right\rVert}
\newcommand{\expe}[1]{\bbE\left[ #1 \right]}
\newcommand{\prob}[1]{\bbP\left[ #1 \right]}
\newcommand{\abs}[1]{\left\lvert #1 \right\rvert}
\newcommand{\bracks}[1]{\left\lbrack #1 \right\rbrack}
\newcommand{\curls}[1]{\left\{ #1 \right\}}
\newcommand{\pars}[1]{\left( #1 \right)}
\newcommand{\inner}[2]{\left\langle #1, #2 \right\rangle}

\newcommand{\sumnolim}{\sum\nolimits}

\newcommand{\kernel}{\texttt{k}} 
\newcommand{\sigkernel}[1][]{\ifthenelse{\equal{#1}{}}{\kernel_\texttt{Sig}}{\kernel_{\texttt{Sig}_{#1}}}} 
\newcommand{\rffkernel}{\tilde{\kernel}}
\newcommand{\rffsigkernel}[1][]{\ifthenelse{\equal{#1}{}}{\rffkernel_\texttt{Sig}}{\rffkernel_{\texttt{Sig}_{#1}}}}
\newcommand{\rffsigkernelhat}[1][]{\ifthenelse{\equal{#1}{}}{\hat{\kernel}_\texttt{Sig}}{\hat{\kernel}_{\texttt{Sig}_{#1}}}}

\newcommand{\rffsigkernelDP}[1][]{\ifthenelse{\equal{#1}{}}{\rffkernel^{\texttt{DP}}_\texttt{Sig}}{\rffkernel^{\texttt{DP}}_{\texttt{Sig}_{#1}}}}
\newcommand{\rffsigkernelTRP}[1][]{\ifthenelse{\equal{#1}{}}{\rffkernel^\texttt{TRP}_\texttt{Sig}}{\rffkernel^\texttt{TRP}_{\texttt{Sig}_{#1}}}}

\newcommand{\test}[1][]{\ifthenelse{\equal{#1}{}}{given}{not given}}

\newcommand{\kernelfeatures}{\varphi}
\newcommand{\signature}[1][]{\ifthenelse{\equal{#1}{}}{\kernelfeatures_\texttt{Sig}}{\kernelfeatures_{\texttt{Sig}_{#1}}}}

\newcommand{\rff}{\tilde{\kernelfeatures}}

\newcommand{\rffsig}[1][]{\ifthenelse{\equal{#1}{}}{\rff_{\texttt{Sig}}}{\rff_{\texttt{Sig}_{#1}}}}
\newcommand{\rffsighat}[1][]{\ifthenelse{\equal{#1}{}}{\hat{\kernelfeatures}_{\texttt{Sig}}}{\hat{\kernelfeatures}_{\texttt{Sig}_{#1}}}}
\newcommand{\rffsigDP}[1][]{\ifthenelse{\equal{#1}{}}{\rff_{\texttt{Sig}}^{\texttt{DP}}}{\rff_{\texttt{Sig}_{#1}}^{\texttt{DP}}}}
\newcommand{\rffsigTRP}[1][]{\ifthenelse{\equal{#1}{}}{\rff_{\texttt{Sig}}^{\texttt{TRP}}}{\rff_{\texttt{Sig}_{#1}}^{\texttt{TRP}}}}

\newcommand{\dimRFF}{\tilde{d}}

\newcommand{\dimTRP}{\dimRFF_{\texttt{TRP}}}

\newcommand{\Dp}[1]{D^{(p)}}
\newcommand{\Ep}[1]{E^{(p)}}
\newcommand{\Hil}{\cH}
\newcommand{\HilT}{\Hil_{\texttt{Sig}}}
\newcommand{\HilRFF}{\tilde{\Hil}}

\newcommand{\HilRFFT}{\HilRFF_{\texttt{Sig}}}
\newcommand{\HilRFFTDP}{\HilRFF_{\texttt{Sig}}^{\texttt{DP}}}
\newcommand{\HilRFFTTRP}{\HilRFF_{\texttt{Sig}}^{\texttt{TRP}}}

\DeclareMathOperator{\seq}{\cX_{\texttt{seq}}}
\DeclareMathOperator{\seqH}{\Hil_{\text{seq}}}

\newcommand{\onevar}{{1\text{-}\mathrm{var}}}

\DeclareMathOperator{\Jac}{J}

\newcommand{\comment}[1]{}

\DeclareMathOperator{\Span}{span}
 
\renewcommand{\d}{\mathrm{d}} 
\newcommand{\R}{\bbR}         
\renewcommand{\b}{\mathbf}    
\newcommand{\iid}{\text{i.i.d.}} 

\usepackage{xcolor}

\ifpdf
\hypersetup{
  pdftitle={Random Fourier Signature Features},
  pdfauthor={Cs.~Tóth, H.~Oberhauser, and Z.~Szabó}
}
\fi

\begin{document}

\maketitle

\begin{abstract}
    Tensor algebras give rise to one of the most powerful measures of similarity for sequences of arbitrary length called the signature kernel accompanied with attractive theoretical guarantees from stochastic analysis.
    Previous algorithms to compute the signature kernel scale quadratically in terms of the length and number of the sequences. To mitigate this severe computational bottleneck, we develop a random Fourier feature-based acceleration of the signature kernel acting on the inherently non-Euclidean domain of sequences. We show uniform approximation guarantees for the proposed unbiased estimator of the signature kernel, while keeping its computation linear in the sequence length and  number. In addition, combined with recent advances on tensor projections, we derive two even more scalable time series features with favourable concentration properties and computational complexity both in time and memory. Our empirical results show that the reduction in computational cost comes at a negligible price in terms of accuracy on moderate size datasets, and it enables one to scale to large datasets up to a million time series. We release the code publicly available at \texttt{\href{https://github.com/tgcsaba/ksig}{https://github.com/tgcsaba/ksig}}.
\end{abstract}



\section{Introduction} \label{sec:introduction}
Machine learning has successfully been applied to tasks that require learning from complex and structured data types on non-Euclidean domains. Feature engineering on such domains is often tackled by exploiting the geometric structure and symmetries existing within the data \cite{bronstein2021geometric}.
Learning from sequential data (such as video, text, audio, time series, health data, etc.) is a classic, but an ongoing challenge due to the following properties:
\begin{itemize}
\item \emph{Non-Euclidean data.} The data domain is nonlinear since there is no obvious and natural way of adding sequences of different length. 

\item \emph{Time-space patterns.} 
  Statistically significant patterns can be distributed over time and space, that is, capturing the order structure in which ``events'' arise is crucial. 

\item \emph{Time-warping invariance.} The meaning of many sequences is often invariant to reparametrization also frequently called time-warping, at least to an extent; e.g.~a sentence spoken quicker or slower contains (essentially) the same information. 

\item \emph{Discretization and irregular sampling.}
  Sequences often arise by sampling along an irregularly spaced grid of an underlying continuous time process. A general methodology should be robust as the sampling gets finer, sequences approximate paths (continuous-time limit), or as the discretization grid varies between sequences.

\item \emph{Scalability.} Sequence datasets can quickly become massive, so the computational complexity should grow subquadratically, in terms of all of the state-space dimension, and the length and number of sequences. 
\end{itemize}
The signature kernel $\sigkernel$ is the state-of-the-art kernel for sequential data \cite{toth2020bayesian,salvi2021signature,lemercier2021scaling} that addresses the first 4 of the above questions and can rely on the modular and powerful framework of kernel learning \cite{scholkopf2002learning}.
Its construction is motivated by classic ideas from stochastic analysis that give a structured description of a sequence by developing it into a series of tensors.
We refer to \cite{lee2023signature} for a recent overview of its various constructions and applications. In the real-world, various phenomena are  well-modelled by systems of differential equations. The path signature arises naturally in the context of controlled differential equations. The role of the signature here is to provide a basis for the effects of a driving signal on systems of controlled differential equations. In essence, it captures the interactions of a controlling signal with a nonlinear system. This explains the widespread applicability of signatures to various problems across the sciences \cite{lyons2007differential}. There is also geometric intuition behind signatures, see Section 1.2.4 in \cite{chevyrev2016primersignaturemethodmachine}.
\paragraph{Features vs Kernel/Primal vs Dual}
Kernel learning circumvents the costly evaluation of a high- or infinite-dimensional feature map by replacing it with the computation of a Gram matrix which contains as entries the inner products of features between all pairs of data points.
This can be very powerful since the inner product evaluation can often be done cheaply by the celebrated "kernel trick", even for infinite-dimensional feature spaces, but the price is that now the computational cost is quadratic in the number of samples, and downstream algorithms further often incur a cubic cost usually in the form of a matrix inversion. On the other hand, when finite-dimensional features can used for learning, the primal formulation of a learning algorithm can perform training and inference in a cost that is linear with respect to the sample size assuming that the feature dimension is fixed. This motivates the investigation of finite-dimensional approximations to kernels that mimic their expressiveness at a lower computational cost. It is an interesting question how the feature dimension should scale with the dataset size to maintain a given (optimal) learning performance in downstream tasks,
which is investigated for instance by \cite{rudi17generalization,carratino18learning,sun2018but,li2019towards,sriperumbudur22approximate,lanthaler2023error}.

\paragraph{Computational Cost of the Signature kernel}
In the context of the signature kernel, one data point is itself a whole sequence. 
Hence, given a data set $\bX$ consisting of $N \in \bbZ_+$ sequences where each sequence $\bx=(\bx_1,\ldots,\bx_{\ell_\bx})$ is of maximal length $\ell_\bx \leq \ell \in \bbZ_+$ and has sequence entries $\bx_i$ in a state-space of dimension $d$, then the existing algorithms to evaluate the Gram matrix of the $\sigkernel$ scale quadratically, i.e.~as $O(N^2 \ell^2 d)$, both in sequence length $\ell$ and number of sequences $N$.
So far this has only been addressed by subsampling (either directly the sequence elements to reduce the length or by column subsampling via the Nyström approach \cite{williams2000nystrom}), which can lead to crude approximations and performance degradation on large-scale datasets.

\paragraph{Contribution}
Random Fourier Features (\RFF) \cite{rahimi2007random} is a classic technique to enjoy both the benefits of the primal and dual approach. 
Here, a low-dimensional and random feature map is constructed, which although does not approximate the feature map of a translation-invariant kernel, its inner product is with high probability close to the kernel itself.
The main contribution of this article is to carry out such a construction for the signature kernel.
Concretely, we construct a random feature map on the domain of sequences called Random Fourier Signature Features (\RFSF), such that its inner product is a random kernel $\rffsigkernel$ for sequences that is both 
\begin{enumerate*}[label=(\roman*)]
    \item an unbiased estimator for $\sigkernel$, and 
    \item has analogous probabilistic approximation guarantees to the classic \RFF{} kernel.
\end{enumerate*}
The challenge is that a direct application of the classic \RFF{} technique is not feasible since this relies on Bochner's theorem which does not apply since the sequence domain is not even a linear space and the feature domain is non-Abelian, which makes the use of (generalizations of \cite{fukumizu2008characteristic}) Bochner's theorem difficult due to the lack of sufficiently explicit representations.
We tackle this challenge by combining the algebraic structure of signatures with probabilistic concentration arguments; a careful analysis of the error propagation yields uniform concentration guarantees similar to the \RFF{} on $\R^d$.
Then, we introduce dimensionality reduction techniques for random tensors further approximating $\rffsigkernel$ to define the extremely scalable variants $\rffsigkernelDP$ and $\rffsigkernelTRP$ called \RFSFD{} and \RFSFT{} saving considerable amounts of computation time and memory by low-dimensional projection of the feature set of the \RFSF{}.

Hence, analogously to the classic \RFF{} construction, the random kernels $\rffsigkernel, \rffsigkernelDP, \rffsigkernelTRP$ simultaneously enjoy the expressivity of an infinite-dimensional feature space as well as linear complexity in sequence length.
This overcomes the arguably biggest drawback of the signature kernel, which is the quadratic complexity in sample size and sequence length; the price for reducing the complexities by an order is that this approximation only holds with high probability. As in the case of the classic \RFF, our experiments show that this is in general a very attractive tradeoff. Concretely, we demonstrate in the experiments that the proposed random features \begin{enumerate*}[label=(\arabic*)] \item provide comparable performance on moderate sized datasets to full-rank (quadratic time) signature kernels, \item outperform other random feature approaches for time series on both moderate- and large-scale datasets, \item allow scaling to datasets of a million time series\end{enumerate*}.

\paragraph{Related Work}
The signature kernel has found many applications; for example, it is used in ABC-Bayes \cite{dyer2023approximate}, economic scenario validation \cite{andres2023signaturebased}, amortised likelihood estimation \cite{dyer2022amortised}, the analysis of RNNs \cite{fermanian2021framing}, analysis of trajectories in Lie groups \cite{lee2020path}, metrics for generative modelling \cite{buehler2021generating,kidger2021neural}, or dynamic analysis of topological structures \cite{giusti2023signatures}. 
For a general overview see \cite{lee2023signature}.
All of these applications can benefit from a faster computation of the signature kernel with theoretical guarantees.
Previous approaches address the quadratic complexity of the signature kernel only by subsampling in one form or another:
\cite{kiraly2019kernels} combine a structured Nystr{\"o}m type-low rank approximation to reduce complexity in dimension of samples and sequence length, \cite{toth2020bayesian} combine this with inducing point and variational methods, \cite{salvi2021signature} uses sequence-subsampling, \cite{lemercier2021scaling} use diagonal approximations to Gram matrices in a variational setting.
Related to this work is also the random nonlinear projections in \cite{lyons2017sketching}; further, \cite{morrill2021generalised} combine linear dimension projection in a general pipeline and \cite{cuchiero2021discrete} use signatures in reservoir computing. 
Directly relevant for this work is recent progress on tensorized random projections \cite{sun2021tensor,rakhshan2020tensorized}.
Random Fourier Features \cite{rahimi2007random,rahimi2008weighted} are well-understood theoretically \cite{sutherland2015error, sriperumbudur2015optimal,sriperumbudur22approximate,liao2020random,avron2017random, szabo2019kernel, chamakh2020orlicz, ullah2018streaming,chamakh2020orlicz}. 
In particular, its generalization properties are studied in e.g.~\cite{bach2013sharp,li2019towards, sun2018but, lanthaler2023error}, where it is shown that the feature dimension need only scale sublinearly in the dataset size for supervised learning, and a similar result also holds for kernel principal component analysis \cite{sriperumbudur22approximate}. 
Several variations have been proposed over the years \cite{le2013fastfood, feng2015random, choromanski2016recycling, yu2016orthogonal, choromanski2017unreasonable, choromanski2018geometry, choromanski2022hybrid}, even finding applications in deep learning \cite{tancik2020fourier}. Alternative random feature approaches for polynomial and Gaussian kernels based on tensor sketching have been proposed in e.g.~\cite{wacker2022complex, wacker2022improved, wacker2022local}. Gaussian sketching has also been applied in the RKHS for kernel approximation \cite{kpotufe2020gaussian}.  For a survey, the reader is referred to \cite{chamakh2020orlicz,liu2021random}.

\paragraph{Outline} Section \ref{sec:prereq} provides background on the prerequisites of our work: Random Fourier Features, and Signature Features/Kernels. Section \ref{sec:RFSF} contains our proposed methods with theoretical results; it introduces Random Fourier Signature Features (\RFSF) $\rffsig[\leq M]$,
\RFSF{} kernels $\rffsigkernel[\leq M]$ (where $M \in \bbZ_+$ is the truncation level introduced later), and most importantly their theoretical guarantees. Theorem \ref{thm:main} quantifies the approximation $\sigkernel[\leq M](\bx,\by) \approx \rffsigkernel[\leq M](\bx,\by)$ uniformly. Then, we discuss additional variants: the \RFSFD{} kernel $\rffsigkernelDP[\leq M]$ and the \RFSFT{} kernel $\rffsigkernelTRP[\leq M]$, which build on the previous construction using dimensionality reduction with corresponding concentration results in Theorems \ref{thm:main2} and \ref{thm:main3}. Section \ref{sec:experiments} compares the performance of the proposed scalable signature kernels against popular approaches on \SVM{} multivariate time series classification, which demonstrates that the proposed kernel not only significantly improves the computational complexity of the signature kernel, it also provides comparable performance, and in some cases even improvements in accuracy as well. Hence, we take the best of both worlds: linear batch, sequence, and state-space dimension complexities, while approximately enjoying the expressivity of an infinite-dimensional \RKHS{} with high probability.

\section{Prerequisites} \label{sec:prereq}

\paragraph{Notation}
We denote the real numbers by $\bbR$, natural numbers by $\bbN \coloneqq \curls{0, 1, 2, \dots}$, positive integers by $\bbZ_+ \coloneqq \curls{1, 2, 3, \dots}$, the range of positive integers from $1$ to $n \in \bbZ_+$ by $\bracks{n} \coloneqq \curls{1, 2, \dots, n}$. Given $a, b \in \bbR$, we denote their maximum by $a \vee b \coloneqq \max(a, b)$ and their minimum by $a \wedge b \coloneqq \min(a, b)$. We define the collection of all ordered $m$-tuples with non-repeating entries starting from $1$ up to $n$ including the endpoints by 
\begin{align} \label{eq:delta_m_def}
    \Delta_m(n) \coloneqq \curls{1 \leq i_1 < i_2 < \cdots < i_m \leq n \given i_1, i_2, \dots, i_n \in [n]}.
\end{align}

In general, $\cX$ refers to a subset of the input domain, where the various objects are defined, generally taken to be a subset $\bbR^d$ unless otherwise stated. For a vector $\bx \in \bbR^d$, we denote its $\ell_p$ norm by $\norm{\bx}_p \coloneqq \pars{\sum_{i=1}^d \abs{x_i}^p}^{1/p}$. For a matrix $\bA \in \bbR^{d \times e}$, we denote the spectral and the Frobenius norm
by $\norm{\bA}_2 \coloneqq \sup_{\norm{\bx}_2 = 1} \norm{\bA \bx}_2$ and $\norm{\bA}_F \coloneqq \pars{\sum_{i=1}^e \norm{\bA \be_i}_2^2}^{1/2}$, where $\{\be_1, \dots, \be_e\}$ is the canonical basis of $\bbR^e$. The transpose of a matrix $\b A$ is denoted by $\b A^\top$. For a differentiable $f: \bbR^d \to \bbR$, we denote its gradient at $\bx \in \bbR^d$ by $\nabla f(\bx) \coloneqq \pars{\nicefrac{\partial f(\bx)}{\partial x_i}}_{i=1}^d$, and its collection of partial derivatives with respect to $\bs \coloneqq (x_{i_1}, \dots, x_{i_k})$ by $\partial_\bs f(\bx) \coloneqq \pars{\nicefrac{\partial f(\bx)}{\partial x_{i_j}}}_{j=1}^k$.

$\seq$ refers to sequences of finite, but unbounded length with values in the set $\cX$:
\begin{align}
\seq \coloneqq \{\bx=(\bx_1,\ldots,\bx_L): \bx_i \in \cX, L \in \bbZ_+   \}.
\end{align}
We denote the length of a sequence $\bx = (\bx_1, \dots, \bx_L) \in \seq$ by $\ell_\bx \coloneqq L$, and define the 1\textsuperscript{st}-order forward differencing operator as $\delta \bx_i \coloneqq \bx_{i+1} - \bx_i$. We define the $1$-variation functional of a sequence $\bx \in \seq$ as $\norm{\bx}_{\onevar} \coloneqq \sum_{i=1}^{\ell_{\bx} - 1} \norm{\delta \bx_i}_2$ as a measure of sequence complexity.

\paragraph{Random Fourier Features} Kernel methods allow to implicitly use an infinite-dimensional feature map $\kernelfeatures:\cX \to \Hil$ by evaluation of the inner product $\kernel(\bx,\by)= \langle \kernelfeatures(\bx), \kernelfeatures(\by) \rangle_\Hil$, when $\cH$ is a Hilbert space. This inner product can often be evaluated without direct computation of $\kernelfeatures(\bx)$ and $\kernelfeatures(\by)$ via the kernel trick. Although this makes them a powerful tool due to the resulting flexibility, the price of this flexibility is a trade-off in complexity with respect to the number of samples $N \in \bbZ_+$. Disregarding the price of evaluating the kernel $\kernel(\bx, \by)$ momentarily, kernel methods require the computation of a Gram matrix with $O(N^2)$ entries, that further incurs an $O(N^3)$ computational cost by most downstream algorithms, such as \texttt{KRR} \cite{shawe2004kernel}, \texttt{GP} \cite{rasmussen2006gaussian}, and \texttt{SVM} \cite{scholkopf2002learning}. Several techniques reduce this complexity, and the focal point of this article is the Random Fourier Feature (\RFF) technique of \cite{rahimi2007random, rahimi2008uniform, rahimi2008weighted}, which can be applied to any continuous, bounded, translation-invariant kernel on $\bbR^{d}$.\footnote{A kernel is called translation-invariant if $\kernel(\bx, \by) = \kernel(\bx+\bz, \by+\bz)$ for any $\bx, \by, \bz \in \bbR^{d}$.} Throughout, we write with some abuse of notation $\kernel(\bx-\by) \equiv \kernel(\bx, \by)$. Next, we outline the \RFF{} construction.

A corollary of Bochner's theorem \cite{rudin2017fourier} is that any continuous, bounded, and translation-invariant kernel $\kernel: \bbR^{d} \times \bbR^{d} \to \bbR$ can be represented as the Fourier transform of a non-negative finite measure $\Lambda$ called the spectral measure associated to $\kernel$, i.e. for $\bx, \by \in \cX$
\begin{align} \label{eq:bochner}
    \kernel(\bx - \by) = \int_{\bbR^d} \exp(i \bw^\top(\bx - \by)) \d\Lambda(\bw).
\end{align}
We may, without loss of generality, assume that $\Lambda$ is a probability measure such that $\Lambda(\bbR^{d}) = 0$, which amounts to working with the kernel $\kernel(\bx - \by) / \kernel(\mathbf{0})$. \cite{rahimi2007random} proposed to draw $\dimRFF \in \bbZ_+$ $\iid$ 
samples from $\Lambda$, $\bw_1,\ldots,\bw_{\dimRFF} \stackrel{\iid}{\sim} \Lambda$, to define the random feature map for $\bx \in \cX$ by
\begin{align} \label{eq:rff_def}
    \rff: \cX \to \HilRFF \coloneqq \bbR^{2\dimRFF}, \quad \rff(\bx) \coloneqq \frac{1}{\sqrt{\tilde d}}\pars{\cos\pars{\bW^\top \bx}, \sin{\pars{\bW^\top \bx}}},     
\end{align}
where $\bW = \pars{\bw_i}_{i=1}^{\dimRFF} \in \bbR^{d \times \dimRFF}$. Then, the corresponding random kernel is defined for $\bx, \by \in \cX$ as 
\begin{align} \label{eq:rffkernel_def}
    \rffkernel: \cX \times \cX \to \bbR, \quad \rffkernel(\bx, \by) = \inner{\rff(\bx)}{\rff(\by)}_{\HilRFF} = \frac{1}{\dimRFF} \sum_{i=1}^{\dimRFF} \cos\pars{\bw_i^\top(\bx - \by)} 
\end{align}
to provide a probabilistic approximation to $\kernel$. Indeed, it is a straightforward exercise to check that $\kernel(\bx, \by) = \expe{\rffkernel(\bx, \by)} \approx\rffkernel(\bx, \by)$.
This approximation converges exponentially fast in $\tilde d$ and uniformly over compact subsets of $\bbR^d$ as proven in \cite[Claim~1]{rahimi2007random}. This bound was later tightened and extended to the derivatives of the kernel in the series of works \cite{sriperumbudur2015optimal,szabo2019kernel, chamakh2020orlicz}, and we provide an adapted version under Theorem~\ref{thm:rff_derivative_approx} in the supplement.

\paragraph{Tensors and the tensor product} First, we provide a brief overview of tensors and tensor products of Hilbert spaces, which we will use to construct our feature space called the \emph{free algebra over a Hilbert space}.
The construction we adapt was first proposed by \cite{murray1936rings}.


Let $\cH_1, \dots, \cH_m$ be Hilbert spaces. To each element $(h_1, \dots, h_m) \in \cH_1 \times \cdots \cH_m$, associate the multi-linear operator $h_1 \otimes \cdots \otimes h_m$ defined for each $(f_1, \dots, f_m) \in \cH_1 \times \cdots \times \cH_m$ by
\begin{align}
    (h_1 \otimes \cdots \otimes h_m)(f_1, \dots, f_m) \coloneqq \inner{h_1}{f_1}_{\cH_1} \cdots \inner{h_m}{f_m}_{\cH_m}.
\end{align}
Take the linear span of all such multi-linear operators to build the space
\begin{align}
    \cH_1 \otimes^\p  \cdots \otimes^\p  \cH_m \coloneqq \Span\curls{h_1 \otimes \cdots \otimes h_m \given h_1 \in \cH_1, \dots, h_m \in \cH_m},
\end{align}
and endow $\cH_1 \otimes^\p  \cdots \otimes^\p  \cH_m$ with an inner product via
\begin{align} \label{eq:inner_tensor}
    \inner{h_1 \otimes \cdots h_m}{f_1 \otimes \cdots \otimes f_m}_{\cH_1 \otimes^\p  \cdots \otimes^\p  \cH_m} \coloneqq \inner{h_1}{f_1}_{\cH_1} \cdots \inner{h_m}{f_m}_{\cH_m}
\end{align}
for all $h_1, f_1 \in \cH_1, \dots, h_m, f_m \in \cH_m$, and extend by linearity to $\cH_1 \otimes^\p  \cdots \otimes^\p  \cH_m$. Taking the topological completion of this space under this inner product gives a Hilbert space denoted by $\cH_1 \otimes  \cdots \otimes \cH_m$ called the tensor product of the Hilbert spaces $\cH_1, \dots, \cH_m$. For more details about the tensor product from an algebraic point of view, see \cite{lang2002algebra}.

\paragraph{Free algebras} Now we introduce our feature space $\HilT$. That is, we show how to embed a Hilbert space $\Hil$ into a bigger Hilbert space $\HilT$ which is also an associative algebra\footnote{An algebra $A$ is a vector space $A$, where one can multiply elements together, i.e.~$\ba \bb \in A$ for $\ba, \bb \in A$.} using a so-called free construction. 
Since the tensor product is associative, we can unambiguously take tensor powers of the vector space $\cH$. 
Denoting $\cH^{\otimes m} \coloneqq \cH \otimes \cdots \otimes \cH$, we define the free algebra over $\cH$ as the set of sequences of tensors indexed by their degree $m \in \bbN$,
\begin{align} \label{eq:free_alg_def}
   \bigoplus_{m \geq 0} \Hil^{\otimes m} = \curls{(t_0, \bt_1, \bt_2, \dots) \given \bt_m \in \cH^{\otimes m} \text{ for~} m \in \bbN, \exists n \in \bbN \text{ s.t.~} N \geq n, \bt_N = \b 0},
\end{align}
where $\bigoplus$ is the direct sum operation, $\otimes$ is the tensor product.
For example, if $\Hil = \bbR^d$, then the degree-$1$ component is a $d$-dimensional vector, the degree-$2$ component is a $d \times d$ matrix, the degree-$3$ component is an array of shape $d \times d \times d$. The space $ \bigoplus_{m \geq 0} \Hil^{\otimes m}$ is a vector space with addition and scalar multiplication defined for $\lambda \in \bbR$, $\bs, \bt \in  \bigoplus_{m \geq 0} \Hil^{\otimes m}$ as
\begin{align}
    \bs + \bt \coloneqq \pars{\bs_m + \bt_m}_{m \geq 0}, \quad \lambda \bs \coloneqq \pars{\lambda \bs_m}_{m \geq 0},
\end{align}
and $\Hil$ is a linear subspace of $\bigoplus_{m \geq 0} \Hil^{\otimes m}$ by identifying $\bv \in \Hil$ 
as $(0, \bv, 0, 0, \dots) \in  \bigoplus_{m \geq 0} \Hil^{\otimes m}$. 
Further, $\bigoplus_{m \geq 0} \Hil^{\otimes m}$ is also an associative algebra since it is endowed with a (noncommutative\footnote{Noncommutative means that $\ba \bb \neq \bb \ba$ in general for elements $\ba, \bb \in V$ of the algebra.}) product defined for tensors $\bs, \bt \in \bigoplus_{m \geq 0} \Hil^{\otimes m}$ as
\begin{align}
    \bs \bt = \pars{\sum_{i=0}^m \bs_i \otimes \bt_{m-i}}_{m \geq 0} \in  \bigoplus_{m \geq 0} \Hil^{\otimes m}.
\end{align}
This process of turning $\Hil$ into an algebra $\bigoplus_{m \geq 0} \Hil^{\otimes m}$ is a free construction; informally this means that \eqref{eq:free_alg_def} is the minimal structure that turns $\Hil$ into an algebra; for more details about free algebras, see \cite{yokonuma1992tensor, reutenauer2003free}.
We now define for $\bs, \bt \in \bigoplus_{m \geq 0} \Hil^{\otimes m}$ their inner product as
\begin{align}\label{ref:hilbert inner product}
    \inner{\bs}{\bt}_{\bigoplus_{m \geq 0} \Hil^{\otimes m}} = \sum_{m \geq 0} \inner{\bs_m}{\bt_m}_{\cH^{\otimes m}},
\end{align}
where the inner product $\inner{\bs_m}{\bt_m}_{\cH^{\otimes m}}$ on $\cH^{\otimes m}$ is as in \eqref{eq:inner_tensor}.
Finally, the completion of $\bigoplus_{m \geq 0} \Hil^{\otimes m}$ in this inner product gives a Hilbert space $\HilT$, which is equivalently defined  as
\begin{align} \label{eq:hilt_def}
    \HilT = \{ \bt = (t_0,\bt_1,\bt_2,\ldots): \bt_m \in \Hil^{\otimes m}, \,\langle \bt, \bt \rangle_{\HilT} < \infty\}.
\end{align}

\paragraph{Path Signatures} A classic way to obtain a structured and hierarchical description of a path $\bx:[0,T] \to \R^d$ is by computing a sequence of iterated integrals called the path signature of $\bx$ given as tensors of increasing degrees $m \in \bbN$ such that the degree-$m$ object is
\begin{align} \label{eq:pathsig}
    S_m(\bx) \coloneqq \idotsint\limits_{0<t_1<\cdots<t_m<T} \d\bx(t_1) \otimes \cdots \otimes \d\bx(t_m) = \idotsint\limits_{0<t_1<\cdots<t_m<T} \dot \bx(t_1) \otimes \cdots \otimes\dot \bx(t_m) \d t_1\cdots \d t_m.
\end{align}

Formally, we refer to the map that takes a path to its iterated integrals, $S: \texttt{Paths} \to \HilT$, $S(\bx) \coloneqq \pars{1, S_1(\bx), S_2(\bx), \ldots}$
as the path signature map. The domain of $S$ is a space of paths that are regular enough such that the integrals are well-defined. 
Its feature space is given by applying the above construction of $\HilT$ in \eqref{eq:hilt_def} to $\Hil=\R^d$ with the Euclidean inner product. 

Among the attractive properties of $S$ is that it linearizes nonlinear functions of paths, that is for any continuous function $f$ one can find a linear functional $\bw$ of $S$ such that
\begin{align} \label{eq:approx}
f(\bx) \approx \langle \bw, S(\bx) \rangle
\coloneqq 
\sum_{m{\ge 0}} \sum_{i_1,\ldots,i_m {\in [d]}} w_{i_1,\ldots,i_m} \int \dot \bx^{i_1}(t_1) \cdots \dot \bx^{i_m}(t_m) \d t_1 \cdots \d t_m,
\end{align}
where \eqref{eq:approx} $w_1, \ldots, w_d, w_{1,1}, \ldots, w_{d,d}, \ldots, w_{d,\ldots,d} \in \R$ denote the coordinates of $\bw$, and the approximation holds uniformly on compacts \cite[Theorem II.5]{fliess1981fonctionnelles} whenever the path $\bx$ includes time as a coordinate\footnote{{This means that} $\bx^i(t)=t$ for some $i{\in[d]}$; more generally, a strictly increasing coordinate is sufficient.}.
The same results generalize to paths without an increasing coordinate up to reparametrization (i.e.~time-warping) and backtracking, formally called ``tree-like'' equivalence, see \cite{hambly2010uniqueness}.
Moreover, these iterated integrals can be well-defined beyond the setting of smooth paths; for example, the same results extend to Brownian motion, semimartingales, and even rougher paths. 
Rough path theory provides a systematic study that comes with a rich toolbox, that combines analytic and algebraic estimates, rich enough to cover the trajectories of large classes of stochastic processes; see \cite{lyons2014rough,friz2020course} for an introduction.
Informally, iterated integrals of paths can be seen as a generalization of classical monomials and from this perspective, the approximation \eqref{eq:approx} can be regarded as the extension of classic polynomial regression to path-valued data.
Thus at least informally it is not surprising, that vanilla signature features suffer from similar drawbacks as classic monomial features; for example, if classic monomials are replaced by other nonlinearities this often drastically improves the approximations; see e.g.~\cite{toth2020bayesian, salvi2021signature}, where precomposing the signature with the \texttt{RBF} kernel increases learning performance.

\paragraph{Signature Features for Sequential Data}
A challenge in machine learning when constructing feature maps for datasets of sequences is that the sequence length can vary from instance to instance; the space of sequences $\seq = \curls{(\bx_i)_{i=1}^\ell \given \bx_1, \dots, \bx_\ell \in \cX \spc \text{and} \spc \ell \in \bbZ_+}$ includes sequences of various lengths, and they should all get mapped to the same feature space, while preserving the information about the elements themselves and their ordering. A concatenation property of path signatures called Chen's identity \cite[Thm.~2.9]{lyons2007differential}
turns concatentation into multiplication provides a principled approach to construct features for sequences. Below we recall the construction of discrete-time signatures based on \cite{toth2021seq2tens}.

The key idea is to define the discrete-time signature of $1$-step increments, and then glue features together by algebra multiplication to guarantee that the Chen identity holds by construction.
Now assume we are given a static feature map $\kernelfeatures: \cX \to \Hil$ into some Hilbert space $\Hil$. Our task is to construct from this feature map for elements of $\cX$, a feature map for sequences of arbitrary length in $\cX$.
A natural first step is to apply the feature map $\kernelfeatures$ elementwise to a sequence $\bx \in \seq$ to lift it to a sequence into the feature space $\Hil$ of $\kernelfeatures$, $\kernelfeatures(\bx) \coloneqq \pars{\kernelfeatures(\bx_i)}_{i=1}^{\ell_\bx} \in \seqH$. The challenge is now to construct a feature map for sequences in $\Hil$.
Simple aggregation of the individual features fails; e.g.~summation of the individual features $\kernelfeatures(\bx_i)$ would lose the order information, vectorization $\pars{\kernelfeatures(\bx_1), \ldots, \kernelfeatures(\bx_{\ell_\bx})} \in \Hil^{\ell_\bx}$ would make sequences of different length not comparable.
It turns out that multiplication is well-suited for this task in a suitable algebra.

Fortunately, there is a natural way to embed any Hilbert space $\Hil$ into a larger Hilbert space $\HilT$ that is also a non-commutative algebra.
First, we take the 1\textsuperscript{st}-order differences,
\begin{align} \label{eq:seq_diff}
\bx \mapsto \delta \kernelfeatures(\bx) \coloneqq \pars{\kernelfeatures(\bx_{i+1}) - \kernelfeatures(\bx_i)}_{i=1}^{\ell_\bx - 1} \in \Hil^{\ell_\bx-1},\quad \text{where} \spc \bx \in \seq
\end{align}
since it is more natural to keep track of changes rather than absolute values.
Then we identify $\Hil$ as a subset of $\HilT$. 
The simplest choice given the above construction of $\HilT$ is
\begin{align} \label{eq:embed_1px}
\iota: \bh \mapsto (1, \bh, \mathbf{0}, \mathbf{0}, \ldots) \in \HilT \quad \text{where} \spc \bh \in \Hil.
\end{align}

A direct calculation shows that composing the maps \eqref{eq:seq_diff}, \eqref{eq:embed_1px}, and multiplying the individual entries in $\HilT$ results in a sequence summary using all non-contiguous subsequences, since in each multiplication step a sequence entry is either selected once or not at all.
This gives rise to the discretized signatures $\signature: \seq \to \HilT$ for $\bx \in \seq$ with $\ell_\bx \geq 2$:
\begin{align} \label{eq:sig_feat}
    \signature(\bx) \coloneqq \prod_{i=1}^{\ell_\bx-1} \iota(\delta\kernelfeatures(\bx_i)) = \pars{\sum_{\bi \in \Delta_m(\ell_\bx - 1)} \delta \kernelfeatures(\bx_{i_1}) \otimes \cdots \otimes \delta \kernelfeatures(\bx_{i_m})}_{m \geq 0},
\end{align}
where $\Delta_m: \bbZ_+ \to \bbZ_+^m$
is as defined in \eqref{eq:delta_m_def} and $\bi = (i_1, \dots, i_m)$. Thus, the sequence feature is itself a sequence, however, now a sequence of tensors indexed by their degree $m \in \bbN$ in contrast to being indexed by the time index $i \in [\ell_\bx]$.
These sequence features are invariant to a natural transformation of time series called time-warping, but can also be made sensitive to it by including time as an extra coordinate with the mapping $\bx = (\bx_i)_{i=1}^{\ell_\bx} \mapsto (t_i, \bx_i)_{i=1}^{\ell_\bx}$. It also possesses similar approximation properties to path signatures in \eqref{eq:approx}, i.e.~uniform approximation of functions of sequences on compact sets; see Appendices A and B in \cite{toth2021seq2tens}.

Despite the abstract derivation, the resulting feature map $\signature$ is---in principle---explicitly computable when $\Hil=\R^d$; see \cite{kidger2021signatory} for details.
However, when the static feature map $\kernelfeatures$ is high- or infinite-dimensional, this is not feasible and we discuss a kernel trick further below.

\begin{remark} \label{remark:truncation}
We used the map $\iota$, as defined in \eqref{eq:embed_1px}, to embed $\Hil$ into $\HilT$.
Other choices are possible, for example one could use the embedding $\hat\iota: \Hil \to \HilT$ for $\bh \in \Hil$
\begin{align}\label{eq:embed_exp}
  \hat\iota(\bh) \coloneqq \pars{1, \bh, \frac{\bh^{\otimes 2}}{2!}, \frac{\bh^{\otimes 3}}{3!},\ldots} {\in\HilT}.
\end{align}
This embedding is actually the classical choice in mathematics, but different choices of the embedding lead to, besides potential improvements in benchmarks, mildly different computational complexities and interesting algebraic questions  \citep{diehl2023generalized,toth2021seq2tens, toth2022capturing}.
\end{remark}

Finally, it can be useful to only consider the first $M \in \bbZ_+$ tensors in the series $\signature(\bx)$ analogously to using the first $M$ moments in classic polynomial regression to avoid overfitting. Hence, we define the $M$-truncated signature features for $M \in \bbZ_+$ as
\begin{align}\label{eq:sig trunc}
 \signature[\leq M](\bx) \coloneqq \pars{1,\signature[1](\bx),\ldots,\signature[M](\bx), \mathbf{0}, \mathbf{0}, \dots } \quad \text{for} \spc \bx \in \seq,
\end{align}
where $\signature[m](\bx)$ is the projection of $\signature(\bx)$ onto $\Hil^{\otimes m}$.
In practice, we regard $M \in \bbZ_+$, and the choice of the embedding as hyperparameters to optimize.

\paragraph{Signature Kernels}
The signature is a powerful feature set for nonlinear regression on paths and sequences. 
A computational bottleneck associated with it is the dimensionality of the feature space $\HilT$. As we are dealing with tensors, for $\Hil$ finite-dimensional $\signature[m](\bx)$ is a tensor of degree-$m$ which has $\pars{\dim \Hil}^{m}$ coordinates that need to be computed. This can quickly become computationally expensive. For infinite-dimensional $\Hil$, e.g.~when $\Hil$ is a reproducing kernel Hilbert space (\RKHS), which is one of the most interesting settings due to the modelling flexibility, it is infeasible to directly compute $\signature$. 
In \cite{kiraly2019kernels}, the signature kernel was introduced, and it was shown that a kernel trick allows to compute the inner product of signature features up to a given degree $M \in \bbZ_+$ using dynamic programming, even when $\Hil$ is infinite-dimensional. 
Subsequently, \cite{salvi2021signature} proposed a PDE-based algorithm to approximate the untruncated signature kernel, which was further extended in \cite{cass2021general}, 
and we refer to \cite{lee2023signature} for a recent overview of signature kernels. 
Here, we focus on discrete-time, and our starting point is the approach of \cite{kiraly2019kernels} combined with the non-geometric approximation \cite{diehl2023generalized} resulting in the features \eqref{eq:sig_feat}.

Above we described a generic way to turn a static feature map $\kernelfeatures: \cX \to \Hil$ into a feature map $\signature[\leq M](\bx)$ for sequences, see \eqref{eq:sig_feat}. The signature kernel is a powerful formalism that allows to transform any static kernel on $\cX$ into a kernel for sequences that evolve in $\cX$. Let $\kernel: \cX \times \cX \to \bbR$ be a continuous and bounded kernel, and from now on, let $\Hil$ denote its \RKHS, and $\kernelfeatures(\bx) \coloneqq \kernel_\bx \equiv \kernel(\bx, \cdot)$ the associated reproducing kernel lift for $\bx \in \cX$. We define the $M$-truncated (discretized) signature kernel $\sigkernel[\leq M]: \seq \times \seq \to \bbR$ for $M \in \bbZ_+$ as
\begin{align}
    \sigkernel[\leq M](\bx, \by)
    &\coloneqq
    \inner{\signature[\leq M](\bx)}{\signature[\leq M](\by)}_{\HilT}
    =
    \sum_{m=0}^M \inner{\signature[m](\bx)}{\signature[m](\by)}_{\Hil^{\otimes m}}
    \\
    &=
    \sum_{m=0}^M \sigkernel[m](\bx, \by)
    =
    \sum_{m=0}^M \sum_{\substack{\bi \in \Delta_m(\ell_\bx - 1)\\\bj \in \Delta_m(\ell_\by-1)}} \delta^2_{i_1, j_1} \kernel(\bx_{i_1}, \by_{j_1}) \cdots \delta^2_{i_m, j_m} \kernel(\bx_{i_m}, \by_{j_m}), \label{eq:sigkernel_def}
\end{align}
where we defined the level-$m$ (discretized) signature kernel $\sigkernel[m]: \seq \times \seq \to \bbR$ for $m \in [M]$ as $\sigkernel[m](\bx, \by) \coloneqq \inner{\signature[m](\bx)}{\signature[m](\by)}_{\Hil^{\otimes m}}$, and $\delta^2$ denotes a 2\textsuperscript{nd}-order cross-differencing operator such that $\delta^2_{i, j} \kernel (\bx_i, \by_j) \coloneqq \kernel(\bx_{i+1}, \by_{j+1}) - \kernel(\bx_{i+1}, \by_{j}) - \kernel(\bx_{i}, \by_{j+1}) + \kernel(\bx_{i}, \by_{j})$ for $i \in [\ell_\bx-1]$ and $j \in [\ell_\by-1]$.
The key insight by \cite{kiraly2019kernels} is equation \eqref{eq:sigkernel_def}, i.e.~that $\sigkernel[\leq M]$ can be computed\footnote{The computation can be carried out exactly for finite $M$ and approximately for $M=\infty$.} without computing $\signature[\leq M]$ itself by a kernel trick that only uses kernel evaluations.

The kernel hyperparameters are the choice of the static kernel $\kernel$, for which there is a wide range of options, e.g.~for $\cX = \bbR^d$ the Gaussian, exponential or Mat{\'e}rn family of kernels; any hyperparameters that $\kernel$ comes with, such as the bandwidth; the truncation level $M \in \bbZ_+$; the choice of the algebra embedding, e.g. \eqref{eq:embed_1px} or \eqref{eq:embed_exp}; and the choice of kernel normalization \cite{chevyrev2022signature} that scales each level $\sigkernel[m]$ appropriately. It also comes with nice theoretical guarantees such as analytic estimates when sequences converge to paths, its maximum mean discrepancy (MMD) metrizes classic topologies for stochastic processes, and can lead to robust statistics in the classic statistical sense (B-robustness); see \cite{chevyrev2022signature} for details.

Although \eqref{eq:sigkernel_def} looks expensive to compute, \cite{kiraly2019kernels} applies dynamic programming to efficiently compute $\sigkernel[\leq M]$ using a recursive algorithm; an alternative algorithm is the above mentioned approach of approximating the (untruncated) signature kernel $\sigkernel$ using PDE-discretization. Importantly, \eqref{eq:sigkernel_def} avoids computing tensors, and only depends on the entry-wise evaluations of the static kernel $\kernel(\bx_i, \by_j)$. Indeed, this leads to a computational cost of $O((M + d) \ell_\bx \ell_\by)$, which is feasible for sequences evolving in high-dimensional state-spaces, but only with moderate sequence length. Note that the same bottleneck applies to PDE-based approaches. In part, the aim of this article is to alleviate this quadratic cost in sequence length, while approximately enjoying the modelling capability of working within an infinite-dimensional \RKHS. 

\section{Random Fourier Signature Features} \label{sec:RFSF}
The goal of this section is to build random features for sequences, that enjoy the benefit of linear sequence length and low-dimensional feature complexity with theoretical guarantees that the corresponding inner product is close to the $M$-truncated (discretized) signature kernel $\sigkernel[\leq M]$ with high probability. We construct these random features in a two step process: firstly, we reduce the feature space from infinite to finite (but high) dimensionality through a careful construction using random Fourier features (RFFs), and in the second step we apply further dimensionality reduction to reduce the complexity to an even lower dimensional space in order to aid in scalability. Although we present this construction as conceptually distinct steps, the steps are coupled during the computation, and the features can be computed directly without going through the initial step.

\paragraph{From infinite to finite dimensions} In Section~\ref{sec:prereq}, we recalled the \RFF{} construction, which associates to a continuous, bounded, translation-invariant kernel $\kernel: \cX \times \cX \to \bbR$ on $\cX$ a spectral measure $\Lambda$, and approximates $\kernel$ by drawing samples from $\Lambda$ to define the random features $\rff: \cX \to \HilRFF$ \eqref{eq:rff_def}, and the random kernel $\rffkernel: \cX \times \cX \to \bbR$ \eqref{eq:rffkernel_def}. Afterwards, we presented a generic way to turn any such static features $\rff: \cX \to \HilRFF$ for elements of $\cX$ into sequence features for sequences that evolve in $\cX$ via $\signature[\leq M]: \seq \to \HilT$. Applying this construction with the \RFF{} as feature map on $\cX$ would already result in a random feature map for sequences, i.e.~a map from $\seq$ into $\HilRFFT$.
Taking the inner product in $\HilRFFT$ of this new random feature map for sequences would, however, only yield a biased estimator for the truncated signature kernel $\sigkernel[\leq M]$.
We correct for this bias by revisiting our previous construction, and build an unbiased approximation to $\sigkernel[\leq M]$ using independent \RFF{} copies in each tensor multiplication step. Then, we show in Theorem \ref{thm:main} that this random estimator comes with good probabilistic guarantees.

The probabilistic construction procedure is outlined in the following definition.

\begin{definition} \label{def:rffsig_def}
Let $\bW^{(1)}, \dots, \bW^{(M)} \stackrel{\iid}{\sim} \Lambda^{\dimRFF}$ be $\iid$ random matrices sampled from $\Lambda^{\dimRFF}$ for \RFF{} dimension $\dimRFF \in \bbZ_+$, and define the independent \RFF{} maps $\rff_m: \cX \to \HilRFF$ as in \eqref{eq:rff_def}, i.e.~$\rff_m(\bx) = \frac{1}{\sqrt{\dimRFF}}\pars{\cos({{\bW^{(m)}}^\top} \bx), \sin({\bW^{(m)}}^\top \bx)}$ for $m \in \bracks{M}$ and $\bx \in \cX$. The $M$-truncated Random Fourier Signature Feature (\RFSF) map $\rffsig[\leq M]: \seq \to \HilRFFT$ from sequences in $\cX$ into the free algebra over $\HilRFF$ is defined for truncation level $M \in \bbZ_+$ and $\bx \in \seq$ as
\begin{align}
\rffsig[\leq M](\bx) \coloneqq \pars{\sum_{\bi \in \Delta_m(\ell_\bx-1)} \delta \rff_1(\bx_{i_1}) \otimes \cdots \otimes \delta \rff_m(\bx_{i_m})}_{m=0}^M.
\label{eq:rffsigdef}
\end{align}
Further, the \RFSF{} kernel $\rffsigkernel[\leq M]: \seq \times \seq \to \bbR$ can be computed for $\bx, \by \in \seq$ as
\begin{align}
    \rffsigkernel[\leq M](\bx,\by) &\coloneqq \inner{\rffsig[\leq M](\bx)}{\rffsig[\leq M](\by)}_{\HilRFFT}
    =
    \sum_{m=0}^M \inner{\rffsig[m](\bx)}{\rffsig[m](\by)}_{\HilRFF^{\otimes m}}
    \\
    &=
    \sum_{m=0}^M \rffsigkernel[m](\bx, \by)
    =
    \sum_{m=0}^M \sum_{\substack{\bi \in \Delta_m(\ell_\bx -1)\\\bj \in \Delta_m(\ell_\by-1)}} \delta^2_{i_1, j_1} \tilde\kernel_1(\bx_{i_1}, \by_{j_1}) \cdots \delta^2_{i_m, j_m} \tilde\kernel_m(\bx_{i_m}, \by_{j_m}), \label{eq:rffsigkernel_def}
\end{align}
where we defined the level-$m$ \RFSF{} kernel $\rffsigkernel[m]: \seq \times \seq \to \bbR$ for $m \in \bbN$ as $\rffsigkernel[m](\bx, \by) \coloneqq \inner{\rffsig[m](\bx)}{\rffsig[m](\by)}_{\HilRFF^{\otimes m}}$ with the convention that $\rffsigkernel[0] \equiv 1$, and $\rffkernel_1, \dots, \rffkernel_M: \cX \times \cX \to \bbR$ are independent \RFF{} kernels defined as in \eqref{eq:rffkernel_def} with the random weights $\bW^{(1)}, \dots, \bW^{(M)} \in \bbR^{d \times \dimRFF}$.
\end{definition}

Since the feature map $\rffsig[\leq M]$ can be directly evaluated in the feature space recursively, $\rffsigkernel[\leq M]$ has linear complexity in the sequence length. However, it requires computing high-dimensional tensors, where the degree-$m$ component $\rffsig[m](\bx) \in \HilRFF^{\otimes m}$ has $(\dim \HilRFF)^m = (2\dimRFF)^m$ coordinates, making it infeasible for large $m, \dimRFF \in \bbZ_+$. Remark \ref{remark:alg_RFSF} discusses the computational complexity in detail. Further, note that the kernel can be evaluated by means of a kernel trick exactly analogously to the evaluation of \eqref{eq:sigkernel_def}, but in this case there are no computational gains compared to the infinite-dimensional signature kernel $\sigkernel[\leq M](\bx, \by)$.

Next, we provide a theoretical analysis to show that the random kernel $\rffsigkernel[\leq M](\bx, \by)$ converges to the ground truth signature kernel $\sigkernel[\leq M](\bx, \by)$ exponentially fast and uniformly over compact state-spaces $\cX \subseteq \bbR^d$, generalizing the result \cite[Claim~2]{rahimi2007random} to this non-Euclidean domain of sequences.
Throughout the analysis, we need certain regularity properties of $\Lambda$ in order to invoke quantitative versions of the law of large numbers, i.e.~properties such as boundedness, existence of the moment-generating function, moment-boundedness, or belonging to certain Orlicz spaces of random variables. Boundedness of the spectral measure is too restrictive an assumption, since a continuous, bounded, translation-invariant kernel $\kernel: \cX \times \cX \to \bbR$ is characteristic if and only if the support of its spectral measure is $\bbR^d$, see \cite[Prop.~8]{sriperumbudur2010relation}. Hence, we instead work with the assumption that its moments are well-controllable, i.e.~the tails of the distribution are not ``too heavy''. Specifically, we assume the Bernstein moment condition that
\begin{align} \label{eq:w_cond_main}
    \bbE_{\bw \sim \Lambda}\bracks{w_i^{2m}} \leq \frac{m! S^2 R^{m-2}}{2} \quad \text{for all} \spc i \in [d]
\end{align}
for some $S, R > 0$. We show in the Supplementary Material under Lemmas \ref{lem:alpha_bernstein_cond} and \ref{lem:alpha_exp_norm_to_bernstein}, in a more general context, that this is equivalent to $\Lambda$ being a sub-Gaussian probability measure; see e.g.~\cite[Sec~2.3]{boucheron2013concentration} and \cite[Sec.~2.5]{vershynin2018high} about sub-Gaussianity. This of course includes the spectral measure of the Gaussian kernel defined for bandwidth $\sigma > 0$ and $\bx, \by \in \cX$ $\kernel(\bx, \by) = \exp\pars{-\nicefrac{\norm{\bx - \by}_2^2}{2\sigma^2}}$, which has a Gaussian spectral distribution $\bw \sim \cN\pars{0, \nicefrac{1}{\sigma^2} \b I_d}$, and therefore calculation gives $\bbE_{w \sim \cN\pars{0, \nicefrac{1}{\sigma^2}}}\bracks{w^{2m}} = \frac{2^m \Gamma\pars{m + \frac{1}{2}})}{ \sigma^{2m}\sqrt{\pi}} < \frac{m!}{2} \pars{\frac{2\sqrt{2}}{\sigma^2 \sqrt[4]{\pi}}}^2\pars{\frac{2}{\sigma^2}}^{m-2}$,
since $\Gamma\pars{m + \nicefrac{1}{2}} < \Gamma(m + 1) = m!$. Hence $\Lambda$ satisfies condition \eqref{eq:w_cond_main} with $S, R$ as given here.
Now we state our approximation theorem regarding $\rffsigkernel[m]$, which quantifies that it is a (sub-)exponentially good estimator of $\sigkernel[m]$ with high probability and uniformly.

\begin{theorem} \label{thm:main}
    Let $\kernel: \bbR^d \times \bbR^d \to \bbR$ be a continuous, bounded, translation-invariant
    kernel with spectral measure $\Lambda$, which satisfies \eqref{eq:w_cond_main}.
    Let $\cX \subset \bbR^d$ be compact and convex with diameter $\abs{\cX}$, $\cX_\Delta \coloneqq \{\bx - \by : \bx, \by \in \cX \}$. Then, the following quantities are finite: $\sigma_\Lambda^2 \coloneqq \bbE_{\bw \sim \Lambda}\bracks{\norm{\bw}_2^2}$, $L \coloneqq \norm{\bbE_{\bw \sim \Lambda}\bracks{\bw \bw^\top}}_2^{1/2}$, $E_{i,j} \coloneqq \bbE_{{\bw \sim \Lambda}}\bracks{\abs{w_i w_j} \norm{\bw}_2}$ and $D_{i, j} \coloneqq \sup_{\bz \in \cX_\Delta} \norm{\nabla \bracks{\frac{\partial^2\kernel(\bz)}{\partial z_i \partial z_j}}}_2$ for $i, j \in [d]$. Further, for any max.~sequence $1$-var $V>0$, and signature level $m \in \mathbb{Z}_+$, for $\epsilon > 0$
    \begin{align}
        \bbP & \bracks{\sup_{\substack{\bx, \by \in \seq \\ \norm{\bx}_\onevar, \norm{\by}_\onevar \leq V}} \abs{\sigkernel[m](\bx, \by) - \rffsigkernel[m](\bx, \by)} \geq \epsilon } \le
        \\
        &\leq
        m
        \begin{cases}
        \pars{C_{d, \cX} \pars{\frac{\beta_{d, m, V}}{\epsilon}}^\frac{d}{d+1} + d}
        \exp\pars{-\frac{\dimRFF}{2(d+1)(S^2 + R)} \pars{\frac{\epsilon}{\beta_{d, m, V}}}^{2}} \,&\text{ for }\, \epsilon < \beta_{d, m, V}  \\
        \pars{C_{d, \cX} \pars{\frac{\beta_{d, m, V}}{\epsilon}}^{\frac{d}{(d+1)m}} + d}
        \exp\pars{-\frac{\dimRFF}{2(d+1)(S^2 + R)} \pars{\frac{\epsilon}{\beta_{d, m,v}}}^{\frac{1}{m}} } \,&\text{ for }\, \epsilon \geq \beta_{d, m, V},
        \end{cases}
        \label{eq:mainthm_bound}
    \end{align}
    where $C_{d, \cX} \coloneqq 2^\frac{1}{d+1} 16 \abs{\cX}^\frac{d}{d+1} \sum_{i,j=1}^d (D_{i,j} + E_{i,j})^\frac{d}{d+1}$ and $\beta_{d, m, V} \coloneqq m \pars{2 V^{2} \pars{L^2 \vee 1} \pars{\sigma_\Lambda^2 \vee d}}^m$.
\end{theorem}

The proof is provided in the supplement under Theorem \ref{thm:rfsf_approx}. The result shows that the random kernel $\rffsigkernel[m]$ approximates the signature kernel $\sigkernel[m]$ uniformly over subsets of $\seq$ of sequences $\bx \in \seq$ with maximal $1$-variation $V$, $\norm{\bx}_\onevar \leq V$, assuming that the state-space $\cX \subset \bbR^d$ is a convex and compact domain.
The error bound is analogous to the classic \RFF{} bounds, in the sense that the tail probability decreases exponentially fast as a function of the \RFF{} dimension $\dimRFF$. The functional form of the bound is inherited from Theorem \ref{thm:rff_derivative_approx}, which provides an analogous result for the derivatives of \RFF. This link follows from Lemma \ref{lem:RFSF_approx}, which connects the concentration of the \RFSF{} kernel to the second derivatives of \RFF.

The main difference from the classic case, i.e.~\cite[Claim~1]{rahimi2007random} and Theorem \ref{thm:rff_derivative_approx}, is the appearance of $\beta_{d,m,V}$ which controls a regime change in the tail behaviour. Concretely, for $\epsilon < \beta_{d, m, V}$ \eqref{eq:mainthm_bound} has a polynomial plus a sub-Gaussian tail, while for $\epsilon > \beta_{d, m, V}$ has a $\pars{\nicefrac{1}{m}}$-subexponential tail.
This is not surprising as the inner summand in \eqref{eq:rffsigkernel_def} is the $m$-fold tensor product of $m$ independent \RFF{} kernels, which makes the tail heavier exactly by an exponent of $\nicefrac{1}{m}$.
The constant itself, $\beta_{d,m,V}$, depends on \begin{enumerate*}[label=(\roman*)] \item the maximal sequence 1-variation $V$, which measures a notion of time-warping invariant sequence complexity;
\item the Lipschitz constant of the kernel $L$ (see Examples \ref{example:2ndmoment_lip} and \ref{example:rff_lip}); \item the trace of the second moment of $\Lambda$, $\sigma_\Lambda^2 = \bbE_{\bw \sim \Lambda}\bracks{\norm{\bw}_2^2}$; \item the state-space dimension $d$; \item and the signature level $m$ itself. \end{enumerate*} 

\begin{remark} \label{remark:alg_RFSF}
    Algorithm \ref{alg:rfsf} demonstrates the computation of the \RFSF{} map $\rffsig[\leq M]$ given a dataset of sequences $\bX = (\bx_i)_{i=1}^{N} \subset \seq$. Upon inspection, we can deduce that the algorithm has a computational complexity of $O\pars{N \ell (M d \dimRFF + 1 + \dimRFF + \ldots + \dimRFF^M)}$. Importantly, it is linear in $\ell$, the sequence length, although scales polynomially in the \RFF{} sample size $\dimRFF^M$.
\end{remark}

\paragraph{Dimensionality Reduction: Diagonal Projection} Previously, we introduced a featurized approximation $\rffsigkernel[\leq M]$ to the signature kernel $\sigkernel[\leq M]$, called the \RFSF{} kernel, which reduces the computation from the infinite-dimensional \RKHS{} to a finite-dimensional feature space using random tensors. Although this makes the computation in the feature space viable of the \RFSF{} map $\rffsig[\leq M]$, it is still tensor-valued, which incurs a computational cost of $O(\dimRFF + \dimRFF^2 + \cdots + \dimRFF^m)$ in the \RFF{} dimension $\dimRFF \in \bbZ_+$. Now, we take another step towards scalability and apply further dimensionality reduction. By examining the structure of these tensors, we introduce a diagonally projected variant called \RFSFD{} that considerably reduces their sizes. We emphasize that the above \RFSF{} construction is the crucial step: it approximates the inner product in an infinite-dimensional space, and now we further approximate it in an even lower dimensional space. The benefit is that one does not have to go through the computation of the initial \RFSF{} map, but only the selected degrees of freedom have to be computed from the beginning.

As a first observation, we notice that the computation of \eqref{eq:rffsigkernel_def} can be reformulated, due to \eqref{eq:rffkernel_def} and linearity of the differencing operator, in the following way:
\begin{align}
    \rffsigkernel[m](\bx, \by)
    =
    \frac{1}{\dimRFF^m} \sum_{q_1, \dots, q_m = 1}^{\dimRFF}  \sum_{\substack{\bi \in \Delta_m(\ell_\bx-1) \\ \bj \in \Delta_m(\ell_\by-1)}} \prod_{p=1}^m \delta^2_{i_p, j_p} \cos\pars{{\bw_{q_p}^{(p)}}^\top (\bx_{i_p} - \by_{j_p})} \label{eq:rffsigkernel_explicit}
\end{align}
by spelling out the definition of the \RFF{} kernel, where $\bw_{1}^{(1)}, \dots, \bw_{\dimRFF}^{(m)} \stackrel{\iid}{\sim} \Lambda$, such that $\bW^{(p)} = \pars{\bw_1^{(p)}, \dots, \bw_{\dimRFF}^{(p)}} \in \bbR^{d \times \dimRFF}$ as defined in Def.~\ref{def:rffsig_def}. Now, we may observe that there is a dependency structure among the samples being averaged in \eqref{eq:rffsigkernel_explicit}, since the outer summation is over the Cartesian product $(q_1, \dots, q_m) \in [\dimRFF]^{\times m}$, which suggests that we might be able to drastically reduce the degrees of freedom by restricting this summation to only go over an independent set of samples.
One way to do this is to restrict to multi-indices of the form $\cI \coloneqq \curls{(q, \dots, q) \in [\dimRFF]^{\times m} \given q \in [\dimRFF]}$, i.e.~we diagonally project the index set, motivating the name of the approach stated in the following definition.
\begin{definition} \label{def:rffsigdp}
    Let $\bw^{(1)}_1, \dots, \bw^{(M)}_{\dimRFF} \stackrel{\iid}{\sim} \Lambda$ for $\dimRFF \in \bbZ_+$, and define $\hat\kernelfeatures_{m,q}: \cX \to \hat\Hil \coloneqq \bbR^2$ with sample size $\hat d = 1$ for $q \in [\dimRFF]$ and $m \in [M]$, such that $\hat\kernelfeatures_{m,q}(\bx) = \pars{\cos({\bw^{(m)}_q}^\top \bx), \sin({\bw^{(m)}_q}^\top \bx)}$ for $\bx \in \cX$. The $M$-truncated Diagonally Projected Random Fourier Signature Feature (\RFSFD) map $\rffsigDP[\leq M]: \seq \to \HilRFFTDP \coloneqq \bigoplus_{m = 0}^M \pars{{\hat\Hil}^{\otimes m}}^{\dimRFF}$ is defined for truncation $M \in \bbZ_+$ and $\bx \in \seq$ as
    \begin{align}
        \rffsigDP[\leq M](\bx) \coloneqq \frac{1}{\sqrt{\dimRFF}} \pars{\pars{\sum_{\bi \in \Delta_m(\ell_\bx-1)} \delta \hat\kernelfeatures_{1,q}(\bx_{i_1}) \otimes \cdots \otimes \delta \hat\kernelfeatures_{m,q}(\bx_{i_m})}_{q=1}^{\dimRFF}}_{m=0}^M.
    \end{align}
Then, the \RFSFD{} kernel can be directly computed for $\bx, \by \in \seq$ via
\begin{align}
    \rffsigkernelDP[\leq M](\bx, \by)
    &\coloneqq
    \inner{\rffsigDP[\leq M](\bx)}{\rffsigDP[\leq M](\by)}_{\HilRFFTDP} 
    =
    \sum_{m=0}^M \inner{\rffsigDP[m](\bx)}{\rffsigDP[m](\by)}_{\pars{{\hat\Hil}^{\otimes m}}^{\dimRFF}}
    \\
    &=
    \sum_{m=0}^M \rffsigkernelDP[m](\bx, \by)
    =
    \frac{1}{\dimRFF} \sum_{m=0}^M \sum_{q=1}^{\dimRFF}  \mathrlap{\sum_{\substack{\bi \in \Delta_m(\ell_{\bx}-1)\\\bj \in \Delta_m(\ell_{\by}-1)}}} \hspace{38pt} \delta^2_{i_1, j_1} \hat\kernel_{1,q}(\bx_{i_1}, \by_{j_1}) \cdots \delta^2_{i_m, j_m} \hat\kernel_{m,q}(\bx_{i_m}, \by_{j_m}), \label{eq:rffsigdpkernel_def}
\end{align}
    where we defined the level-$m$ \RFSFD{} kernel $\rffsigkernelDP[m]: \seq \times \seq \to \bbR$ for $m \in \bbN$ and $\bx, \by \in \seq$ as $\rffsigkernel(\bx, \by) \coloneqq \inner{\rffsigDP[m](\bx)}{\rffsigDP[m](\by)}_{\pars{{\hat\Hil}^{\otimes m}}^{\dimRFF}}$ with the convention that $\rffsigkernelDP[0] \equiv 1$, and $\hat\kernel_{m,q}: \cX \times \cX \to \bbR$ are independent \RFF{} kernels with sample size $\hat d = 1$ defined for $\bx, \by \in \cX$ as $\hat\kernel_{m,q}(\bx, \by) \coloneqq \inner{\hat\kernelfeatures_{m,q}(\bx)}{\hat\kernelfeatures_{m,q}(\by)}_{\hat\Hil}$ with the random weights $\bw^{(m)}_q \in \bbR^d$ for $q \in [\dimRFF], m \in [M]$.
\end{definition}

Note that by the definition of the \RFF{} kernels in \eqref{eq:rffsigdpkernel_def}, we may substitute that $\hat\kernel_{p,q}(\bx, \by) = \cos({\bw^{(p)}_q}^\top(\bx - \by))$ for $\bx, \by \in \cX$, so \eqref{eq:rffsigdpkernel_def} is equivalently written for $\bx, \by \in \seq$ as
\begin{align}
    \rffsigkernelDP[m](\bx, \by) = \frac{1}{\dimRFF} \sum_{q=1}^{\dimRFF} \mathrlap{\sum_{\substack{\bi \in \Delta_m(\ell_{\bx}-1)\\\bj \in \Delta_m(\ell_{\by}-1)}}} \hspace{40pt} \delta^2_{i_1, j_1} \cos({\bw^{(1)}_q}^\top(\bx_{i_1} - \by_{j_1})) \cdots \delta^2_{i_m, j_m} \cos({\bw^{(m)}_q}^\top(\bx_{i_m} - \by_{j_m})), 
\end{align}
which is what we set out to do in the above paragraph; that is, restrict the outer summation onto the diagonal projection of the index set.
Another way to look at Definition \ref{def:rffsigdp} is that the \RFSFD{} kernel in \eqref{eq:rffsigdpkernel_def} is constructed by defining $\dimRFF$ independent \RFSF{} kernels, each with internal \RFF{} sample size $\hat d = 1$, and then taking their average; the concatenation of their corresponding features are then the features of the \RFSFD{} map. Note that for \RFF{} sample size $1$, each \RFF{} map has dimension $2$, i.e.~$\hat\Hil = \bbR^2$, and hence, the corresponding \RFSF{} kernels have dimension $1 + 2 + \cdots + 2^M = (2^{M+1} - 1)$, which by concatenation results in the overall dimensionality of the \RFSFD{} kernel being $\dim \HilRFFTTRP = \dimRFF\pars{2^{M+1} - 1}$. This relates to the computational complexity of the \RFSFD{} map; for details see Remark \ref{remark:alg_rfsf_dp}.

Next, we state our concentration result regarding the level-$m$ \RFSFD{} kernel $\rffsigkernelDP[m](\bx, \by)$.
\begin{theorem} \label{thm:main2}
    Let $\kernel: \bbR^d \times \bbR^d \to \bbR$ be a continuous, bounded, translation-invariant kernel with spectral measure $\Lambda$, which satisfies \eqref{eq:w_cond_main}. Then, for level $m \in \bbZ_+$, $\bx, \by \in \seq$, and $\epsilon > 0$ 
    \begin{align} \label{eq:mainthm2_bound}
        \bbP\bracks{\abs{\rffsigkernelDP[m](\bx, \by) - \sigkernel[m](\bx, \by)} \geq \epsilon}
        \leq
        2\exp\pars{-\frac{1}{4}\min\curls{
        \pars{\frac{\sqrt{\dimRFF} \epsilon}{2C_{d, m, \bx, \by}}}^2, 
        \pars{\frac{\dimRFF \epsilon}{\sqrt{8}C_{d, m, \bx, \by}}}^{\frac{1}{m}}
        }},
     \end{align}
    where $L \coloneqq \norm{\bbE_{\bw \sim \Lambda}\bracks{\bw \bw^\top}}$ is the Lipschitz constant of $\kernel$, and $C_{d, m, \bx, \by} > 0$ is bounded by
    \begin{align}
    C_{d, m, \bx, \by}
    \leq
    \sqrt{8} e^4 (2\pi)^{1/4} e^{1/24} (4e^3\norm{\bx}_\onevar \norm{\by}_\onevar /m)^m \pars{\pars{2d\max(S, R)}^m + \pars{L^2/\ln 2}^m}.
    \end{align}
\end{theorem}
The proof is provided in the supplement under Theorem \ref{thm:rfsf_dp_approx}. The result shows that the \RFSFD{} kernel converges for any two sequences $\bx, \by \in \seq$ with a $\pars{\nicefrac{1}{m}}$-subexponential convergence rate with respect to the sample size $\dimRFF \in \bbZ_+$. 
Similarly to Theorem \ref{thm:main}, the bound has a phase transition, where for small values of $\epsilon$, it has a sub-Gaussian tail, while for larger values, it has a $\pars{\nicefrac{1}{m}}$-subexponential tail. A crucial difference from the previous bound is that now the phase transition happens at $\epsilon^\star = C_{d, m, \bx, \by} 2^\frac{2m-3/2}{2m-1} \dimRFF^{\frac{1-m}{2m-1}}$, which depends on the sample size $\dimRFF$. This means that for fixed value of $\epsilon > 0$, the phase transition always happens eventually as $\dimRFF$ gets large enough, hence the convergence rate with respect to $\dimRFF$ is $\pars{\nicefrac{1}{m}}$-subexponential regardless of the value of $\epsilon$. The slightly reduced rate of convergence compared to the \RFSF{} kernel in Theorem \ref{thm:main} is to be expected, since the sample size of the \RFSFD{} kernel is analogously reduced by an exponent of $\pars{\nicefrac{1}{m}}$ with respect to $\dimRFF$ in comparison.
The constant $C_{d,m,\bx, \by}$, similarly to \eqref{eq:mainthm_bound}, depends on \begin{enumerate*}[label=(\roman*)] \item the 1-variation of sequences $\norm{\bx}_\onevar, \norm{\by}_\onevar$ that measure the complexity of the sequences; \item $L > 0$, the Lipschitz constant of the kernel $\kernel$ (see Examples \ref{example:2ndmoment_lip},  \ref{example:rff_lip}); \item the moment bound parameters $S, R > 0$ from condition \eqref{eq:w_cond_main}; \item the state-space dimension $d$;  and \item the signature level $m$. \end{enumerate*}

\begin{remark} \label{remark:alg_rfsf_dp}
    Algorithm \ref{alg:rsfsdp} demonstrates the computation of the \RFSFD{} map $\rffsigDP[\leq M]$ given a dataset of sequences $\bX = (\bx_i)_{i=1}^N \subset \seq$. Upon counting the operations, we deduce that the algorithm has a computational complexity $O\pars{N \ell \dimRFF (M d + 2^{M})}$. Crucially, it is linear in both $\ell$, the maximal sequence length, and $\dimRFF$, the sample size of the random kernel.
\end{remark}
\paragraph{Dimensionality Reduction: Tensor Random Projection} Previously, we built the \RFSFD{} map by subsampling an independent set from the samples that constitute \RFSF{} kernel. Here, we propose an alternative dimensionality reduction technique that starts again from the \RFSF{} map, and uses random projections to project this generally high-dimensional tensor onto a lower dimension. Random projections are a classic technique in data science for reducing the data dimension, while preserving its important structural properties. They are built upon the celebrated Johnson-Lindenstrauss lemma \cite{johnson1986extensions}, which states that a set of points in a high-dimensional space can be embedded into a space of much lower dimension, while approximately preserving their 
geometry. Exploiting this property, we construct a tensor random projected (\TRP) variant of our random kernel called \RFSFT{}, such that the computation is coupled between the \RFSF{} and \TRP{} maps, similarly to a kernel trick.

Tensorized random projections \cite{sun2021tensor, rakhshan2020tensorized} construct random projections for tensors with concise parametrization that respects their tensorial nature. Given tensors $\bs, \bt \in \pars{\bbR^d}^{\otimes m}$ for $m \in \bbZ_+$, the \TRP{} map with \CP{} (\texttt{CANDECOMP/PARAFAC} \cite{kolda2009tensor})  rank-$1$ is built via a random functional 
$\pr: \pars{\bbR^d}^{\otimes m} \to \bbR$ such that $\pr(\bs) = \inner{\bp_1 \otimes \cdots \otimes \bp_m}{\bs}_{\pars{\bbR^d}^{\otimes m}}$, where $\bp_1, \dots, \bp_m \stackrel{\iid}{\sim} \cN(\b 0, \b I_d)$ are $d$-dimensional component vectors sampled from a standard normal distribution. Then, the inner product can be estimated as $\pr(\bs) \pr(\bt) \approx \expe{\pr(\bs) \pr(\bt)} = \inner{\bs}{\bt}_{\pars{\bbR^d}^{\otimes m}}$. Variance reduction is achieved by stacking $n \in \bbZ_+$ such random projections, each with $\iid$ component vectors $\bp_1^{(1)}, \dots, \bp_m^{(n)} \stackrel{\iid}{\sim} \cN(\b 0, \b I_d)$. Hence, the \TRP{} operator is defined as
\begin{align} \label{eq:trp_def}
    \TRP: \pars{\bbR^d}^{\otimes m} \to \bbR^n, \quad  \TRP(\bs) \coloneqq \frac{1}{\sqrt{n}}\pars{\inner{\bp_1^{(i)} \otimes \cdots \otimes \bp_m^{(i)}}{\bs}}_{i=1}^n.
\end{align}
On the one hand, this allows to represent the random projection map onto $\bbR^n$ using only $O(n m d)$ parameters as opposed to the $O(n d^m)$ parameters in a densely parametrized random projection; and on the other, it allows for downstream computations to exploit the low-rank structure of the operator, as we shall do so in the definition stated below.
\begin{definition} \label{def:rffsigtrp_def}
Let $\bW^{(1)}, \dots, \bW^{(M)} \stackrel{\iid}{\sim} \Lambda^{\dimRFF}$ be $\iid$ random matrices sampled from $\Lambda^{\dimRFF}$ for \RFF{} dimension $\dimRFF \in \bbZ_+$, define the independent \RFF{} maps $\rff_m: \cX \to \HilRFF$ as in \eqref{eq:rff_def}, i.e.~$\rff_m(\bx) = \nicefrac{1}{\sqrt{\dimRFF}}\pars{\cos({{\bW^{(m)}}^\top} \bx), \sin({\bW^{(m)}}^\top \bx)}$ for $m \in \bracks{M}$ and $\bx \in \cX$, and let $\bP^{(1)}, \dots, \bP^{(M)} \stackrel{\iid}{\sim} \cN^{\dimRFF}\pars{\b 0, \b I_{2 \dimRFF}}$
be random matrices with $\iid$ standard normal entries. The $M$-truncated Tensor Random Projected Random Fourier Signature Feature (\RFSFT) map $\rffsigTRP[\leq M] \seq \to \HilRFFTTRP = \bbR^{M \dimRFF}$ is defined for truncation level $M \in \bbZ_+$ and $\bx \in \seq$ as
\begin{align}
    \rffsigTRP[\leq M](\bx) &\coloneqq
    \frac{1}{\sqrt{\dimRFF}} \pars{\pars{\sum_{\bi \in \Delta_m(\ell_\bx-1)} \inner{\bp^{(1)}_q}{\delta \rff_1(\bx_{i_1})} \cdots \inner{\bp^{(m)}_q}{\delta \rff_m(\bx_{i_m})}}_{q=1}^{\dimRFF}}_{m=0}^M
    \\
    &=
    \frac{1}{\sqrt{\dimRFF}} \pars{\sum_{\bi \in \Delta_m(\ell_\bx-1)} \pars{{\bP^{(1)}}^\top \delta \rff_1(\bx_{i_1})} \odot \cdots \odot \pars{{\bP^{(m)}}^\top \delta \rff_m(\bx_{i_m})}}_{m=0}^M,
\label{eq:rffsigtrpdef}
\end{align}
where $\bP^{(m)} = \pars{\bp_q^{(m)}}_{q=1}^{\dimRFF} \in \bbR^{2\dimRFF \times \dimRFF}$, and $\odot$ denotes the Hadamard product\footnote{The Hadamard product stands for component-wise multiplication of the vectors $\bx, \by \in \bbR^n$, $\bx \odot \by = \pars{x_i y_i}_{i=1}^n$.}.
The \RFSFT{} kernel $\rffsigkernelTRP[\leq M]: \seq \times \seq \to \bbR$ can then be directly computed for sequences $\bx, \by \in \seq$ by
\begin{align}
    \rffsigkernelTRP[\leq M](\bx,\by) &\coloneqq \inner{\rffsigTRP[\leq M](\bx)}{\rffsigTRP[\leq M](\by)}_{\HilRFFT}
    =
    \sum_{m=0}^M \inner{\rffsigTRP[m](\bx)}{\rffsigTRP[m](\by)}_{\HilRFF^{\otimes m}}
    \\
    &=
    \sum_{m=0}^M \rffsigkernelTRP[m](\bx, \by)
    =
    \frac{1}{\dimRFF} \sum_{m=0}^M \sum_{q=1}^{\dimRFF} \sum_{\substack{\bi \in \Delta_m(\ell_\bx -1)\\\bj \in \Delta_m(\ell_\by-1)}} \prod_{p=1}^m \inner{\bp^{(p)}_q}{\delta \rff_p(\bx_{i_p})} \inner{\bp^{(p)}_q}{\delta \rff_p(\by_{j_p})}, \label{eq:rffsigtrpkernel_def}
\end{align}
where we defined the level-$m$ \RFSFT{} kernel $\rffsigkernelTRP[m]: \seq \times \seq \to \bbR$ for $m \leq M$ as $\rffsigkernelTRP[m](\bx, \by) \coloneqq \inner{\rffsigTRP[m](\bx)}{\rffsigTRP[m](\by)}_{\HilRFF^{\otimes m}}$ with the convention that $\rffsigkernelTRP[0] \equiv 1$.
\end{definition}

We remark that \eqref{eq:rffsigtrpdef} is equivalent to the \TRP{} operator \eqref{eq:trp_def} applied to the \RFSF{} map \eqref{eq:rffsigdef} by exploiting bilinearity of the inner product, and using that it factorizes over the tensor components, as described in \eqref{eq:inner_tensor}. Then, the unbiasedness of \eqref{eq:rffsigtrpkernel_def} follows from the fact that the \TRP{} operator is an isometry under expectation, which is applied to the \RFSF{} tensor $\rffsig[m]$, therefore $\rffsigkernelTRP[m]$ kernel is conditionally an unbiased estimator of $\rffsigkernel[m]$ given the \RFSF{} weights $\bW^{(1)}, \dots, \bW^{(m)} \in \bbR^{d \times \dimRFF}$.
By the tower rule for expectations, $\rffsigkernelTRP[m]$ is an unbiased estimator of $\sigkernel[m]$. The approximation quality is then governed by two factors: \begin{enumerate*}[label=(\roman*)] \item \label{it:trp1} how well the \TRP{} projected kernel $\rffsigkernelTRP[m]$ approximates $\rffsigkernel[m]$; \item \label{it:trp2} the quality of the approximation of $\rffsigkernel[m]$ with respect to $\sigkernel[m]$. \end{enumerate*} Note that \ref{it:trp2} has already been discussed in Theorem \ref{thm:main}. Here, we state the following theoretical result which quantifies \ref{it:trp1}. Combining these two results by means of triangle inequality and union bounding quantifies that $\rffsigkernelTRP[m]$ is a good estimator of $\sigkernel[m]$.

\begin{theorem} \label{thm:main3}
    Let $\kernel: \bbR^d \times \bbR^d \to \bbR$ be a continuous, bounded, translation-invariant kernel with spectral measure $\Lambda$, which satisfies \eqref{eq:w_cond_main}. Then, the following bound holds for \RFSFT{} kernel for signature level $m \in \bbZ_+$ sequences $\bx, \by \in \seq$ and $\epsilon > 0$
    \begin{align} \label{eq:mainthm3_bound}
        \bbP\bracks{\abs{\rffsigkernelTRP[m](\bx, \by) - \rffsigkernel[m](\bx, \by)} \geq \epsilon}
        \leq
        C_{d, \Lambda}
        \exp\pars{- \pars{\frac{m^2 \dimRFF^{\frac{1}{2m}} \epsilon^{\frac{1}{m}}}{2\sqrt{2}e^3 R \norm{\bx}_\onevar \norm{\by}_\onevar}}^\frac{1}{2}},
    \end{align}
    where the absolute constant is defined as $C_{d, \Lambda} \coloneqq 2\pars{1 + \frac{S}{2R} + \frac{S^2}{4R^2}}^d$.
\end{theorem}

The proof is given in the supplement under Theorem \ref{thm:rfsf_trp_approx} utilizing the hypercontractivity of Gaussian polynomials \cite{janson1997gaussian} that is used to quantify the concentration of the \TRP{} estimator. The concentration of the \RFSFT{} kernel is then governed by Theorems \ref{thm:main} and \ref{thm:main3} combined. Together, they show that for smaller values of $\epsilon$ (i.e.~the regime change as discussed below Theorem \ref{thm:main}), the probability has a polynomial plus a sub-Gaussian tail, while for large $\epsilon$, it has a $\pars{\frac{1}{2m}}$-subexponential tail due to \eqref{eq:mainthm3_bound}, and the dominant convergence rate with respect to $\dimRFF$ is $\pars{\frac{1}{4m}}$-subexponential. This means that in terms of convergence, \RFSFT{} is the slowest among the 3 variations introduced so far. However, it is also the most efficient in terms of overall dimension, hence downstream computational complexity as well, since $\HilRFFTTRP = \bbR^{M \dimRFF}$.
Remark \ref{remark:alg_rfsf_trp} discusses the computational complexity in detail.

\begin{remark} \label{remark:alg_rfsf_trp}
    Algorithm \ref{alg:rsfstrp} demonstrates the computation of the \RFSFT{} map $\rffsigTRP[\leq M]$ given a dataset of sequences $\bX = (\bx_i)_{i=1}^N \subset \seq$. Counting the operations, here we can deduce that the algorithm has an $O\pars{M  N  \ell  \dimRFF  (d + \dimRFF)}$ computational complexity. This variation is also linear in $\ell$, the maximal sequence length, although it is quadratic in $\dimRFF$.
\end{remark}

\paragraph{Numerical evaluation}
\begin{wrapfigure}{r}{0.5\textwidth} 
\vspace{-10pt}
\hspace{-10pt}\includegraphics[width=0.5\textwidth]{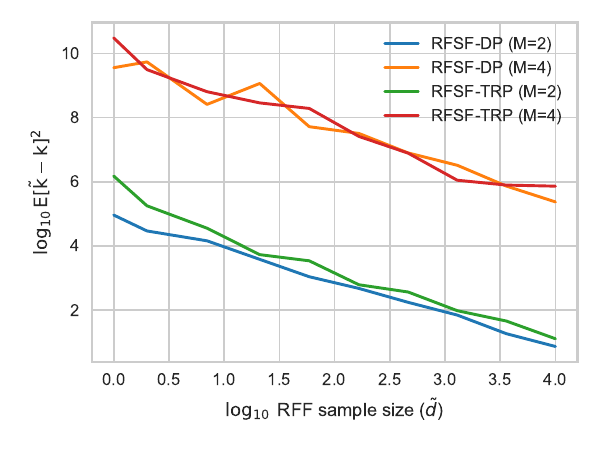}
\vspace{-12.5pt}
\caption{Approximation error of random kernels against \RFF{} sample size on $\log$-$\log$ plot.}
\label{fig:approx}
\vspace{-20pt}
\end{wrapfigure}
Here, we numerically evaluate the approximation error of the proposed scalable kernels, that is, \RFSFD{} and \RFSFT{}. We do not include \RFSF{} since its dimensionality shows polynomial explosion in the base sample size $\dimRFF$ due to its tensor-based representation, which makes its computation infeasible for reasonable values of $\dimRFF$. We generate $d$-dimensional synthetic time series of length-$\ell$ using a \texttt{VAR}(1) process $\tilde\bx \in \seq$, such that $\tilde\bx_0 = \b 0$ and $\tilde\bx_{t+1} = \nicefrac{1}{\sqrt{d}} A \tilde\bx_t + \b\epsilon_t$, where $A \sim \cN^{d \times d}(0, 1)$ and $\b\epsilon_t \sim \cN(\b 0, \sigma^2 \b I_d)$. Then, we compute the normalized version $\bx \in \seq$ of $\tilde\bx$, which is rescaled to have $1$-variation $V > 0$, i.e. $\bx_t = V \tilde\bx_t / \norm{\tilde\bx}_{\onevar}$. We set $d = 10$, $\ell = 100$, $\sigma = 0.1$ and $V = 100$. We compute the squared deviation between the groundtruth signature kernel and the randomized approximations for two randomly sampled time series in this way. This process is repeated for $100$ randomly sampled time series and $100$ times resampled random kernel evaluations, giving rise to overall $10000$ evaluations for each value of $\dimRFF$. Figure \ref{fig:approx} shows the average approximation error plotted against values of $\dimRFF$ on a $\log$-$\log$ plot. We can observe that both \RFSFD{} and \RFSFT{} have approximately the same error curves for a given value of truncation level $M$, and the steepness appears to be the same across different levels of $M$. This means that \RFSFT{} is slightly more efficient in terms of dimensionality, since its dimension is $M\dimRFF$ as opposed to $2^{M+1} \dimRFF$ in \RFSFD{}. We also observe that both curves are close to being linear, which indicates that the approximation error scales approximately as $O(\dimRFF^{-\alpha})$ for some value of $\alpha > 0$.


\section{Experiments}\label{sec:experiments}

\paragraph{Time series classification}
We perform multivariate time series classification to investigate the performance of the scalable \RFSF{} variants compared to the full-rank signature kernel and other quadratic time baseline kernels, and further, to demonstrate the scalability to large-scale datasets, where the quadratic sample complexity becomes prohibitive.
We use support vector machine (\SVM) \citep{steinwart2008support} classification for classifying multivariate time series on datasets of various sizes. For quadratic time kernels, the dual \SVM{} formulation is used, while for kernels with feature representations, we use the primal formulation that has linear complexity in the size of the dataset $n \in \bbZ_+$ aiding in scalability to truly large-scale datasets. For each considered kernel/feature, we use a \texttt{GPU}-based implementation provided in the \KS{} library\footnote{\href{https://github.com/tgcsaba/KSig}{https://github.com/tgcsaba/KSig}}. For large-scale experiments with the featurized kernels, linear \SVM{} implementation is used from the \texttt{cuML} library \cite{raschka2020machine}, while the dual \SVM{} on moderate-scale datasets uses the \texttt{sklearn} library \cite{scikit-learn}. For multi-class problems, we use the one-vs-one classification strategy. This study is also the largest scale comparison of signature kernels to date which extends the datasets considered in \cite{salvi2021signature}.
The hardware used was 2 computer clusters equipped with overall 8 NVIDIA 3080 Ti \texttt{GPU}s.
\begin{table}
 \caption{Computational complexities of kernels in our experiments; $N \in \bbZ_+$ is the number of time series, $\ell \in \bbZ_+$ is their length, $d \in \bbZ_+$ is their state-space dimension, $M \in \bbZ_+$ is the signature truncation level, $\dimRFF \in \bbZ_+$ is the \texttt{RF} dimension, $W \in \bbZ_+$ is the warping length in \texttt{RWS}.}
 \label{table:complexity}
 \resizebox{\textwidth}{!}{
\begin{tabular}{cccccccc}
\toprule
 \texttt{RFSF-DP} & \texttt{RFSF-TRP} & \texttt{KSig} & \texttt{KSigPDE} & \texttt{RWS} & \texttt{GAK} & \texttt{RBF} & \texttt{RFF}
 \\
 \midrule
 $O\pars{N\ell\tilde d\pars{Md + 2^M}}$ & $O\pars{N\ell M\tilde d\pars{d + \tilde d}}$ & $O\pars{N^2 \ell^2 \pars{M + d}}$ & $O\pars{N^2 \ell^2 d}$ & $O\pars{N\ell W d}$ & $O\pars{N^2 \ell^2 d}$ & $O\pars{N^2 \ell d}$ & $O\pars{N \ell d\dimRFF}$
 \\
\bottomrule
\end{tabular}
}
\end{table}

\paragraph{Methods}
We compare the proposed variants \RFSFD{} and \RFSFT{} to the baselines described here: 
\begin{enumerate*}[label=(\arabic*)] \item the $M$-truncated Signature Kernel \cite{kiraly2019kernels} \KS{} formulated via the kernel trick, and is a quadratic time baseline; \item the Signature-PDE Kernel \cite{salvi2021signature} \KSP{}, which uses the 2\textsuperscript{nd}-order PDE solver and also has quadratic complexity; \item the Global Alignment Kernel \cite{cuturi2011fast} \GAK{}, one of the most popular sequence kernels to day and can can be related to the signature kernel, see \cite[Sec.~5]{kiraly2019kernels}; \item Random Warping Series \cite{wu2018random} \RWS{}, which produces features by \texttt{DTW} alignments between the input and randomly sampled time series; \item the \RBF{} kernel, which treats the whole time series as a vector of length $\bbR^{\ell d}$, \item Random Fourier Reatures \cite{rahimi2007random} \RFF{}, which also treats the time series as a long vector. We excluded \RFSF{} from the comparison, as it is unfeasible to compute it with reasonable sample sizes $\dimRFF$ due to the polynomial explosion of dimensions in its tensor-based representation. The complexities are compared in Table \ref{table:complexity}.
\end{enumerate*}

\paragraph{Hyperparameter selection}
For each dataset-kernel, we perform cross-validation to select the optimal hyperparameters that are optimized over the Cartesian product of the following options. For each method that requires a static kernel, we use the \RBF{} kernel with bandwidth hyperparameter $\sigma > 0$. This is specified in terms of a rescaled median heuristic, i.e.~
\begin{align} \label{eq:alpha_select}
    \sigma = \alpha \med \curls{\norm{\bx_i - \bx_j^\prime}_2 / 2 \given i \in [\ell_\bx], j \in [\ell_{\bx^\prime}], \bx, \bx^\prime \in \bX}, \quad \text{for} \spc \alpha > 0,
\end{align}
where $\alpha$ is chosen from $\alpha \in \{10^{-3}, \dots, 10^3\}$ on a logarithmic grid with $19$ steps. For each kernel that is not normalized by default (i.e.~the \GAK{} and \RBF{} kernels are normalized, the former is because without normalization it blows up) , we select whether to normalize to unit norm in feature space via $\kernel(\bx, \by) \mapsto \kernel(\bx, \by) / \sqrt{\kernel(\bx, \bx) \kernel(\by, \by)}$. The \SVM{} hyperparameter $C > 0$ is selected from $C \in \{10^0, 10^1, 10^2, 10^3, 10^4\}$. Further, motivated by previous work that investigates the effect of path augmentations in the context of signature methods \cite{morrill2021generalised}, we chose 3 augmentations to cross-validate over. First is parametrization encoding, which gives the classifier the ability to remove the warping invariance of a given sequence kernel, adding the time index as an additional coordinate, i.e. for each time series in the dataset $\bx \in \bX$, we augment it via $\bx = (\bx_i)_{i=1}^{\ell_\bx} \mapsto (\beta i / \ell_\bx, \bx_i)_{i=1}^{\ell_\bx}$, where $\beta > 0$ is the parametrization intensity chosen from $\beta \in \{10^0, 10^1, 10^2, 10^3, 10^4\}$. The second augmentation is the basepoint encoding, the role of which is to remove the translation invariance of signature features. Note that when the static base kernel is chosen to be a nonlinear kernel other than the Euclidean inner product, the signature kernel is not completely translation-invariant due to the state-space nonlinearities, but it is close being that by the $L$-Lipschitz property in Lemma \ref{example:2ndmoment_lip} valid for of the static kernels considered in this work. The basepoint encoding adds an initial $\mathbf{0}$ step at the beginning of each time series, i.e.~for $\bx \in \bX$, $\bx=(\bx_1, \dots, \bx_{\ell_\bx}) \mapsto (\mathbf{0}, \bx_1, \dots, \bx_{\ell_\bx})$. The third augmentation is the lead-lag map, which is defined as $\bx = (\bx_1, \dots, \bx_{\ell_\bx}) \mapsto \pars{(\bx_1, \bx_1), (\bx_2, \bx_1), (\bx_2, \bx_2), \dots, (\bx_{\ell_\bx}, \bx_{\ell_\bx-1}), (\bx_{\ell_\bx}, \bx_{\ell_\bx})}$. For the truncation-based signature kernels, we select the truncation level $M \in \bbZ_+$ from $M \in \{2, 3, 4, 5\}$. For \RWS{}, we select the warping length from $W \in \{10, 20, \dots, 100\}$ as suggested by the authors. This makes \RWS{} the most expensive feature-based kernel, and so as to fit within the same resource limitations, we omit cross-validating over the path augmentations. We select the standard deviation $\sigma > 0$ of the warping series from the same grid as $\alpha$ in \eqref{eq:alpha_select}. For all \texttt{RF} approaches, we set the \texttt{RF} dimension $\tilde d \in \bbZ_+$, so the overall dimension is $1000$. Note that for \RFSFD{} and \RFSFT{} this is respectively $2^{M+1} \dimRFF$ and $M \dimRFF$, where $\dimRFF$ is the base \RFF{} sample size; for \RWS{} it is the number of warping series $\dimRFF$; while for \RFF{} it is twice the number of samples $2 \dimRFF$.

\paragraph{Datasets: UEA Archive}
The UEA archive \cite{dau2019ucr} is a collection of overall 30 datasets for benchmarking classifiers on multivariate time series classification problems containing both binary and multi-class tasks.
The data modality ranges from various sources e.g.~human activity recognition, motion classification, ECG classification, EEG/MEG classification, audio spectra recognition, and others. The sizes of the datasets in terms of number of time series range from moderate ($\leq 1000$ examples) to large ($\leq 30000$), and includes various lengths between $8$ and $18000$. A summary of the dataset characteristics can be found in Table~2 in \cite{dau2019ucr}. Pre-specified train-test splits are provided for each dataset, which we follow. We evaluate all considered kernels on the moderate datasets ($\leq 1000$ time series), but because the non-feature-based become very expensive computationally beyond these sizes, we only evaluate feature-based approaches on medium and large datasets ($\geq 1000$ time series). Each featurized approach is trained and evaluated $5$ times on each dataset in order to account for the randomness in the hyperparameter selection procedure and evaluation.

\begin{table}[!t]
\begin{minipage}{\textwidth}
 \centering
 \captionof{table}{Comparison of \SVM{} test accuracies on moderate multivariate time series classification datasets. For each row, the best result is highlighted in \textbf{bold}, and the second best in \textit{italic}.}
 \label{table:results}
 \resizebox{\textwidth}{!}{
\begin{tabular}{lcccccccc}
\toprule
 & \texttt{RFSF-DP} & \texttt{RFSF-TRP} & \texttt{KSig} & \texttt{KSigPDE} & \texttt{RWS} & \texttt{GAK} & \texttt{RBF} & \texttt{RFF}
\\
\midrule
ArticularyWordRecognition & $0.984$ & $0.981$ & $\mathbf{0.990}$ & $0.983$ & $\mathit{0.987}$ & $0.977$ & $0.977$ & $0.978$
\\ 
AtrialFibrillation & $0.373$ & $0.320$ & $\mathit{0.400}$ & $0.333$ & $\mathbf{0.427}$ & $0.333$ & $0.267$ & $0.373$
\\ 
BasicMotions & $\mathbf{1.000}$ & $\mathbf{1.000}$ & $\mathbf{1.000}$ & $\mathbf{1.000}$ & $\mathit{0.995}$ & $\mathbf{1.000}$ & $0.975$ & $0.860$
\\ 
Cricket & $0.964$ & $0.964$ & $0.958$ & $\mathit{0.972}$ & $\mathbf{0.978}$ & $0.944$ & $0.917$ & $0.886$
\\ 
DuckDuckGeese & $0.636$ & $\mathit{0.664}$ & $\mathbf{0.700}$ & $0.480$ & $0.492$ & $0.500$ & $0.420$ & $0.372$
\\ 
ERing & $0.921$ & $0.936$ & $0.841$ & $\mathit{0.941}$ & $\mathbf{0.945}$ & $0.926$ & $0.937$ & $0.915$
\\ 
EigenWorms & $\mathit{0.817}$ & $\mathbf{0.837}$ & $0.809$ & $0.794$ & $0.623$ & $0.511$ & $0.496$ & $0.443$
\\ 
Epilepsy & $\mathbf{0.949}$ & $\mathit{0.942}$ & $\mathbf{0.949}$ & $0.891$ & $0.925$ & $0.870$ & $0.891$ & $0.777$
\\ 
EthanolConcentration & $0.457$ & $0.439$ & $\mathbf{0.479}$ & $\mathit{0.460}$ & $0.284$ & $0.361$ & $0.346$ & $0.325$
\\ 
FingerMovements & $0.608$ & $0.624$ & $\mathbf{0.640}$ & $\mathit{0.630}$ & $0.612$ & $0.500$ & $0.620$ & $0.570$
\\ 
HandMovementDirection & $\mathit{0.573}$ & $0.568$ & $\mathbf{0.595}$ & $0.527$ & $0.403$ & $\mathbf{0.595}$ & $0.541$ & $0.454$
\\ 
Handwriting & $0.434$ & $0.400$ & $0.479$ & $0.409$ & $\mathbf{0.591}$ & $\mathit{0.481}$ & $0.307$ & $0.249$
\\ 
Heartbeat & $0.717$ & $0.712$ & $0.712$ & $\mathbf{0.722}$ & $0.714$ & $0.717$ & $0.717$ & $\mathit{0.721}$
\\ 
JapaneseVowels & $0.978$ & $0.978$ & $\mathbf{0.986}$ & $\mathbf{0.986}$ & $0.955$ & $\mathit{0.981}$ & $\mathit{0.981}$ & $0.979$
\\ 
Libras & $0.898$ & $\mathbf{0.928}$ & $\mathit{0.922}$ & $0.894$ & $0.837$ & $0.767$ & $0.800$ & $0.800$
\\ 
MotorImagery & $\mathit{0.516}$ & $\mathbf{0.526}$ & $0.500$ & $0.500$ & $0.508$ & $0.470$ & $0.500$ & $0.482$
\\ 
NATOPS & $0.906$ & $0.908$ & $0.922$ & $\mathbf{0.928}$ & $\mathit{0.924}$ & $0.922$ & $0.917$ & $0.900$
\\ 
PEMS-SF & $0.800$ & $0.808$ & $0.827$ & $\mathit{0.838}$ & $0.701$ & $\mathbf{0.855}$ & $\mathbf{0.855}$ & $0.770$
\\ 
RacketSports & $0.874$ & $0.861$ & $\mathbf{0.921}$ & $\mathit{0.908}$ & $0.878$ & $0.849$ & $0.809$ & $0.755$
\\ 
SelfRegulationSCP1 & $0.868$ & $0.856$ & $\mathit{0.904}$ & $\mathit{0.904}$ & $0.829$ & $\mathbf{0.915}$ & $0.898$ & $0.885$
\\ 
SelfRegulationSCP2 & $0.489$ & $0.510$ & $\mathit{0.539}$ & $\mathbf{0.544}$ & $0.481$ & $0.511$ & $0.439$ & $0.492$
\\ 
StandWalkJump & $0.387$ & $0.333$ & $\mathit{0.400}$ & $\mathit{0.400}$ & $0.347$ & $0.267$ & $\mathbf{0.533}$ & $0.267$
\\ 
UWaveGestureLibrary & $0.882$ & $0.881$ & $\mathbf{0.912}$ & $0.866$ & $\mathit{0.897}$ & $0.887$ & $0.766$ & $0.846$
\\ 
\midrule
Avg.~acc. & $\mathit{0.740}$ & $0.738$ & $\mathbf{0.756}$ & $0.735$ & $0.710$ & $0.702$ & $0.692$ & $0.656$
\\ 
Avg.~rank & $3.609$ & $3.739$ & $\textbf{2.348}$ & $\mathit{2.957}$ & $3.957$ & $4.174$ & $4.913$ & $5.913$
\\
\bottomrule
\end{tabular}
}
\end{minipage}
\vspace{10pt}
\resizebox{\textwidth}{!}{
\begin{minipage}{0.51\textwidth}
    \vspace{40pt}
        \centering
        \includegraphics[width=\textwidth]{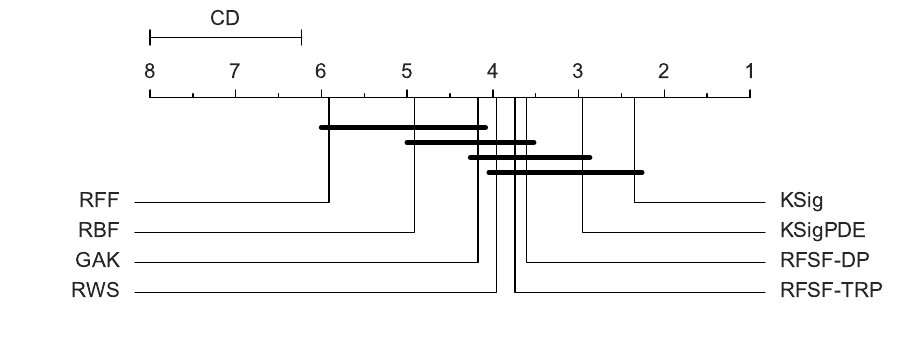}
        \captionof{figure}{Critical difference diagram comparison on moderate datasets of considered approaches using two-tailed Nemenyi test \cite{demsar2006statistical}.}
        \label{fig:cd_diagram}
\end{minipage}
\hspace{1.5pt}
\begin{minipage}{0.49\textwidth}
\vspace{10pt}
\captionof{table}{Comparison of accuracies on large-scale datasets of random features.}
\label{table:large_scale_results}
\resizebox{1.\textwidth}{!}{
\begin{tabular}{lcccc}
\toprule
  & \RFSFD{} & \RFSFT{} & \texttt{RWS} & \texttt{RFF}
\\ 
\midrule
CharacterTrajectories & $\mathit{0.990}$ & $\mathit{0.990}$ & $\mathbf{0.991}$ & $0.989$
\\ 
FaceDetection & $\mathit{0.653}$ & $\mathbf{0.656}$ & $0.642$ & $0.572$
\\ 
InsectWingbeat & $\mathit{0.436}$ & $\mathbf{0.459}$ & $0.227$ & $0.341$
\\ 
LSST & $0.589$ & $\mathit{0.624}$ & $\mathbf{0.631}$ & $0.423$
\\ 
PenDigits & $\mathit{0.983}$ & $0.982$ & $\mathbf{0.989}$ & $0.980$
\\ 
PhonemeSpectra & $\mathit{0.204}$ & $\mathit{0.204}$ & $\mathbf{0.205}$ & $0.083$
\\ 
SITS1M & $\mathbf{0.745}$ & $\mathit{0.740}$ & $0.610$ & $0.718$
\\ 
SpokenArabicDigits & $\mathbf{0.981}$ & $\mathit{0.980}$ & $\mathbf{0.981}$ & $0.964$
\\ 
fNIRS2MW & $\mathbf{0.659}$ & $\mathit{0.658}$ & $0.621$ & $0.642$
\\
\midrule
Avg.~acc. & $\mathit{0.693}$ & $\mathbf{0.699}$ & $0.655$ & $0.635$
\\ 
Avg.~rank & $\mathbf{1.778}$ & $\mathit{1.889}$ & $2.222$ & $3.333$
\\
\bottomrule
\end{tabular}
}
\end{minipage}
}
\end{table}

\paragraph{Datasets: Mental Workload Intensity Classification}
We evaluate featurized approaches on a large-scale brain-activity recording data set called \texttt{fNIRS2MW}.\footnote{\href{https://github.com/tufts-ml/fNIRS-mental-workload-classifiers}{https://github.com/tufts-ml/fNIRS-mental-workload-classifiers}} This dataset contains brain activity recordings collected from overall $68$ participants during a $30$-$60$ minute experimental session, where they were asked to carry out tasks of varying intensity. The collected time series are sliced into $30$ second segments using a sliding window, and each segment is labelled with an intensity level ($0$-$3$), giving rise to overall $\sim 100000$ segments, which we split in a ratio of $80-20$ for training and testing. We convert the task into a binary classification problem by assigning a label whether the task is low ($0$ or $1$) or high ($2$ or $3$) intensity.

\paragraph{Datasets: Satellite Image Classification} As a massive scale task, we use a satellite imagery dataset\footnote{ \href{https://cloudstor.aarnet.edu.au/plus/index.php/s/pRLVtQyNhxDdCoM}{https://cloudstor.aarnet.edu.au/plus/index.php/s/pRLVtQyNhxDdCoM}} of $N = 10^6$ time series. Each length $\ell = 46$ time series corresponds to a vegetation index calculated from remote sensing data, and the task is to classify land cover types \cite{petitjean2012satellite} by mapping vegetation profiles to various types of crops and forested areas corresponding to $24$ classes. We split the dataset in a ratio of $90$-$10$ for training and testing.

\subsection{Results} 
Table \ref{table:results} compares test accuracies on moderate size multivariate time series classification datasets with $N \leq 1000$ from the UEA archive. \KS{} provides state-of-the-art performance among all sequence kernels with taking the highest aggregate score in terms of all of average accuracy, average rank, and number of first places. Our proposed random feature variants \RFSFD{} and \RFSFT{} provide comparable performance on most of the datasets in terms of accuracy, and they are only outperformed by \KS{} and \KSP{} with respect to average accuracy and rank. Interestingly, \RFSFT{} has more first place rankings, but \RFSFD{} performs slightly better on average. This shows that on datasets of these sizes, using either of \RFSFD{} and \RFSFT{} does not sacrifice model performance - even leading to improvements in some cases, potentially due to the implicit regularization effect of restricting to a finite-dimensional feature space - and it can already provide speedups. We visualize the critical difference diagram comparison of all considered approaches in Figure \ref{fig:cd_diagram}.

Table \ref{table:large_scale_results} demonstrates the performance of scalable approaches, i.e.~\RFSFD{}, \RFSFT{}, \RWS{} and \RFF{} on the remaining UEA datasets ($N \geq 1000$), the dataset \texttt{fNIRS2MW} ($N=10^5$), and the satellite dataset \texttt{SITS1M} ($N = 10^6$). We find it infeasible to perform full cross-validation for quadratic time kernels on these datasets due to expensive kernel computations and downstream cost of dual \SVM{}. The results show that both variants \RFSFD{} and \RFSFT{} perform significantly better on average with respect to accuracy and rank then both \RWS{} and \RFF{}. Note when \RWS{} takes first place, it only improves over our approach marginally, however, when it underperforms, it often does so severely. This is not surprising as both \RFSFD{} and \RFSFT{} approximate the signature kernel, which is a universal kernel on time series; it is theoretically capable of learning from any kind of time series data as supported by its best overall performance above.

\section{Conclusion}
We constructed a random kernel $\rffsigkernel[\leq M]$ for sequences that benefits from
\begin{enumerate*}[label=(\roman*)]
  \item lifting the original sequence to an infinite-dimensional \RKHS{} $\Hil$,
  \item linear complexity in sequence length,
  \item being with high probability close to the signature kernel $\sigkernel$.
\end{enumerate*}
Thereby it combines the strength of the signature kernel $\sigkernel$ which is to implicitly use the iterated integrals of a sequence that has an infinite-dimensional \RKHS{} $\Hil$ as state-space with the strength of (unkernelized) signature features $\signature$ that only require linear time complexity.
Our main theoretical result extends the theoretical guarantees for translation-invariant kernels on linear spaces to the signature kernel defined on the nonlinear domain $\seq$; however, the proofs differ from the classic case and require to analyse the error propagation in tensor space.
A second step is more straightforward, and combines this approach with random projections in finite-dimensions for tensors to reduce the complexity in memory further.
The advantages and disadvantages of the resulting approach are analogous to the classic \RFF{} technique on $\R^d$, namely a reduction of computational complexity by an order for the price of an approximation that only holds with high probability.
As in the classic \RFF{} case, our experiments indicate that this is in general a favourable tradeoff.

In the future, it would be interesting both theoretically and empirically to replace the vanilla Monte Carlo integration in the \RFF{} construction by block-orthogonal random matrices as done in \cite{yu2016orthogonal}.
Further, our random features can also be used to define an unbiased appoximation to the inner product of expected signatures, which has found usecases, among many, in nonparametric hypothesis testing and market regime detection \cite{chevyrev2022signature, horvath2023non}, training of generative models \cite{ni2021sig, issa2023non}, and graph representation learning \cite{toth2022capturing}.  

\section*{Acknowledgements}
CT was supported by the Mathematical Institute Award by the University of Oxford. HO was supported by the Hong Kong Innovation and Technology Commission (InnoHK Project CIMDA) and by the EPSRC grant Datasig [EP/S026347/1].


\newpage

\bibliographystyle{siamplain}
\bibliography{BIB/clean}

\begin{thebibliography}{10}

\bibitem{andres2023signaturebased}
{\sc H.~Andrès, A.~Boumezoued, and B.~Jourdain}, {\em Signature-based
  validation of real-world economic scenarios}, 2023,
  \url{https://arxiv.org/abs/2208.07251}.

\bibitem{avron2017random}
{\sc H.~Avron, M.~Kapralov, C.~Musco, C.~Musco, A.~Velingker, and A.~Zandieh},
  {\em Random {F}ourier features for kernel ridge regression: Approximation
  bounds and statistical guarantees}, in International Conference on Machine
  Learning, 2017, pp.~253--262.

\bibitem{bach2013sharp}
{\sc F.~Bach}, {\em Sharp analysis of low-rank kernel matrix approximations},
  in Conference on Learning Theory, 2013, pp.~185--209.

\bibitem{boucheron2013concentration}
{\sc S.~Boucheron, G.~Lugosi, and P.~Massart}, {\em Concentration Inequalities:
  A Nonasymptotic Theory of Independence}, Oxford University Press, 2013.

\bibitem{bronstein2021geometric}
{\sc M.~M. Bronstein, J.~Bruna, T.~Cohen, and P.~Veličković}, {\em Geometric
  deep learning: Grids, groups, graphs, geodesics, and gauges}, 2021,
  \url{https://arxiv.org/abs/2104.13478}.

\bibitem{buehler2021generating}
{\sc H.~Buehler, B.~Horvath, T.~Lyons, I.~Arribas, and B.~Wood}, {\em
  Generating financial markets with signatures}, 2021.
\newblock \url{https://ssrn.com/abstract=3657366}.

\bibitem{carratino18learning}
{\sc L.~Carratino, A.~Rudi, and L.~Rosasco}, {\em Learning with {SGD} and
  random features}, in Advances in Neural Information Processing Systems, 2018,
  pp.~10192--10203.

\bibitem{cass2021general}
{\sc T.~Cass, T.~Lyons, and X.~Xu}, {\em General signature kernels}, 2021,
  \url{https://arxiv.org/abs/2107.00447}.

\bibitem{chamakh2020orlicz}
{\sc L.~Chamakh, E.~Gobet, and Z.~Szab\'o}, {\em Orlicz random {F}ourier
  features}, Journal of Machine Learning Research, 21 (2020), pp.~1--37.

\bibitem{chevyrev2016primersignaturemethodmachine}
{\sc I.~Chevyrev and A.~Kormilitzin}, {\em A primer on the signature method in
  machine learning}, 2016, \url{https://arxiv.org/abs/1603.03788},
  \url{https://arxiv.org/abs/1603.03788}.

\bibitem{chevyrev2022signature}
{\sc I.~Chevyrev and H.~Oberhauser}, {\em Signature moments to characterize
  laws of stochastic processes}, Journal of Machine Learning Research, 23
  (2022), pp.~1--42.

\bibitem{choromanski2018geometry}
{\sc K.~Choromanski, M.~Rowland, T.~Sarlos, V.~Sindhwani, R.~Turner, and
  A.~Weller}, {\em The geometry of random features}, in International
  Conference on Artificial Intelligence and Statistics, 2018, pp.~1--9.

\bibitem{choromanski2017unreasonable}
{\sc K.~Choromanski, M.~Rowland, and A.~Weller}, {\em The unreasonable
  effectiveness of structured random orthogonal embeddings}, in International
  Conference on Neural Information Processing Systems, 2017, pp.~218--227.

\bibitem{choromanski2016recycling}
{\sc K.~Choromanski and V.~Sindhwani}, {\em Recycling randomness with structure
  for sublinear time kernel expansions}, in International Conference on Machine
  Learning, 2016, pp.~2502--2510.

\bibitem{choromanski2022hybrid}
{\sc K.~M. Choromanski, H.~Lin, H.~Chen, A.~Sehanobish, Y.~Ma, D.~Jain,
  J.~Varley, A.~Zeng, M.~S. Ryoo, V.~Likhosherstov, D.~Kalashnikov,
  V.~Sindhwani, and A.~Weller}, {\em Hybrid random features}, in International
  Conference on Learning Representations, 2022,
  \url{https://openreview.net/forum?id=EMigfE6ZeS}.

\bibitem{cuchiero2021discrete}
{\sc C.~Cuchiero, L.~Gonon, L.~Grigoryeva, J.-P. Ortega, and J.~Teichmann},
  {\em Discrete-time signatures and randomness in reservoir computing}, IEEE
  Transactions on Neural Networks and Learning Systems, 33 (2022),
  pp.~6321--6330.

\bibitem{cucker2002mathematical}
{\sc F.~Cucker and S.~Smale}, {\em On the mathematical foundations of
  learning}, Bulletin of the American mathematical society, 39 (2002),
  pp.~1--49.

\bibitem{cuturi2011fast}
{\sc M.~Cuturi}, {\em Fast global alignment kernels}, in International
  Conference on Machine Learning, 2011, pp.~929--936.

\bibitem{dau2019ucr}
{\sc H.~A. Dau, A.~Bagnall, K.~Kamgar, C.-C.~M. Yeh, Y.~Zhu, S.~Gharghabi,
  C.~A. Ratanamahatana, and E.~Keogh}, {\em The {UCR} time series archive},
  IEEE/CAA Journal of Automatica Sinica, 6 (2019), pp.~1293--1305.

\bibitem{demsar2006statistical}
{\sc J.~Dem{\v{s}}ar}, {\em Statistical comparisons of classifiers over
  multiple data sets}, Journal of Machine Learning Research, 7 (2006),
  pp.~1--30.

\bibitem{diehl2023generalized}
{\sc J.~Diehl, K.~Ebrahimi-Fard, and N.~Tapia}, {\em Generalized iterated-sums
  signatures}, Journal of Algebra, 632 (2023), p.~801–824.

\bibitem{dudley2002real}
{\sc R.~M. Dudley}, {\em Real Analysis and Probability}, Cambridge University
  Press, 2002.

\bibitem{dyer2023approximate}
{\sc J.~Dyer, P.~Cannon, and S.~M. Schmon}, {\em Approximate {B}ayesian
  computation with path signatures}, 2023,
  \url{https://arxiv.org/abs/2106.12555}.

\bibitem{dyer2022amortised}
{\sc J.~Dyer, P.~W. Cannon, and S.~M. Schmon}, {\em Amortised likelihood-free
  inference for expensive time-series simulators with signatured ratio
  estimation}, in International Conference on Artificial Intelligence and
  Statistics, 2022, pp.~11131--11144.

\bibitem{feng2015random}
{\sc C.~Feng, Q.~Hu, and S.~Liao}, {\em Random feature mapping with signed
  circulant matrix projection}, in International Joint Conference on Artificial
  Intelligence, 2015, p.~3490–3496.

\bibitem{fermanian2021framing}
{\sc A.~Fermanian, P.~Marion, J.-P. Vert, and G.~Biau}, {\em Framing {RNN} as a
  kernel method: A neural {ODE} approach}, in Advances in Neural Information
  Processing Systems, 2021, pp.~3121--3134.

\bibitem{fliess1981fonctionnelles}
{\sc M.~Fliess}, {\em Fonctionnelles causales non lin{\'e}aires et
  ind{\'e}termin{\'e}es non commutatives}, Bulletin de la soci{\'e}t{\'e}
  math{\'e}matique de France, 109 (1981), pp.~3--40.

\bibitem{friz2020course}
{\sc P.~K. Friz and M.~Hairer}, {\em A course on rough paths}, Springer, 2020.

\bibitem{fukumizu2008characteristic}
{\sc K.~Fukumizu, A.~Gretton, B.~Sch\"{o}lkopf, and B.~K. Sriperumbudur}, {\em
  Characteristic kernels on groups and semigroups}, in Advances in Neural
  Information Processing Systems, 2008, p.~473–480.

\bibitem{giusti2023signatures}
{\sc C.~Giusti and D.~Lee}, {\em Signatures, lipschitz-free spaces, and paths
  of persistence diagrams}, 2023, \url{https://arxiv.org/abs/2108.02727}.

\bibitem{gotze2021concentration}
{\sc F.~G{\"o}tze, H.~Sambale, and A.~Sinulis}, {\em {Concentration
  inequalities for polynomials in $\alpha$-sub-exponential random variables}},
  Electronic Journal of Probability, 26 (2021), pp.~1 -- 22.

\bibitem{hambly2010uniqueness}
{\sc B.~Hambly and T.~Lyons}, {\em Uniqueness for the signature of a path of
  bounded variation and the reduced path group}, Annals of Mathematics, 171
  (2010), pp.~109--167.

\bibitem{horvath2023non}
{\sc B.~Horvath and Z.~Issa}, {\em Non-parametric online market regime
  detection and regime clustering for multidimensional and path-dependent data
  structures}, 2023.
\newblock \url{https://ssrn.com/abstract=4493344}.

\bibitem{issa2023non}
{\sc Z.~Issa, B.~Horvath, M.~Lemercier, and C.~Salvi}, {\em Non-adversarial
  training of neural {SDE}s with signature kernel scores}, 2023,
  \url{https://arxiv.org/abs/2305.16274}.

\bibitem{isserlis1918formula}
{\sc L.~Isserlis}, {\em On a formula for the product-moment coefficient of any
  order of a normal frequency distribution in any number of variables},
  Biometrika, 12 (1918), pp.~134--139.

\bibitem{janson1997gaussian}
{\sc S.~Janson}, {\em Gaussian Hilbert Spaces}, Cambridge University Press,
  1997.

\bibitem{johnson1986extensions}
{\sc W.~B. Johnson, J.~Lindenstrauss, and G.~Schechtman}, {\em Extensions of
  {L}ipschitz maps into {B}anach spaces}, Israel Journal of Mathematics, 54
  (1986), pp.~129--138.

\bibitem{kidger2021neural}
{\sc P.~Kidger, J.~Foster, X.~Li, H.~Oberhauser, and T.~J. Lyons}, {\em Neural
  {SDE}s as infinite-dimensional {GAN}s}, in International Conference on
  Machine Learning, 2021, pp.~5453--5463.

\bibitem{kidger2021signatory}
{\sc P.~Kidger and T.~Lyons}, {\em Signatory: differentiable computations of
  the signature and logsignature transforms, on both {CPU} and {GPU}}, in
  International Conference on Learning Representations, 2021,
  \url{https://openreview.net/forum?id=lqU2cs3Zca}.

\bibitem{kiraly2019kernels}
{\sc F.~J. Kir{\'a}ly and H.~Oberhauser}, {\em Kernels for sequentially ordered
  data}, Journal of Machine Learning Research, 20 (2019), pp.~1--45.

\bibitem{kolda2009tensor}
{\sc T.~G. Kolda and B.~W. Bader}, {\em Tensor decompositions and
  applications}, SIAM review, 51 (2009), pp.~455--500.

\bibitem{kpotufe2020gaussian}
{\sc S.~Kpotufe and B.~Sriperumbudur}, {\em Gaussian sketching yields a {JL}
  lemma in {RKHS}}, in International Conference on Artificial Intelligence and
  Statistics, 2020, pp.~3928--3937.

\bibitem{kuchibhotla2022moving}
{\sc A.~K. Kuchibhotla and A.~Chakrabortty}, {\em Moving beyond
  sub-{G}aussianity in high-dimensional statistics: Applications in covariance
  estimation and linear regression}, Information and Inference: A Journal of
  the IMA, 11 (2022), pp.~1389--1456.

\bibitem{lang2002algebra}
{\sc S.~Lang}, {\em Algebra}, Springer, 2002.

\bibitem{lanthaler2023error}
{\sc S.~Lanthaler and N.~H. Nelsen}, {\em Error bounds for learning with
  vector-valued random features}, in Advances in Neural Information Processing
  Systems, 2023, \url{https://openreview.net/forum?id=sLr1sohnmo}.

\bibitem{le2013fastfood}
{\sc Q.~Le, T.~Sarl{\'o}s, A.~Smola, et~al.}, {\em Fastfood-approximating
  kernel expansions in loglinear time}, in International Conference on Machine
  Learning, 2013, p.~8.

\bibitem{lee2020path}
{\sc D.~Lee and R.~Ghrist}, {\em Path signatures on {L}ie groups}, 2020,
  \url{https://arxiv.org/abs/2007.06633}.

\bibitem{lee2023signature}
{\sc D.~Lee and H.~Oberhauser}, {\em The signature kernel}, 2023,
  \url{https://arxiv.org/abs/2305.04625}.

\bibitem{lemercier2021scaling}
{\sc M.~Lemercier, C.~Salvi, T.~Cass, E.~V. Bonilla, T.~Damoulas, and T.~J.
  Lyons}, {\em Sig{GPDE}: Scaling sparse {G}aussian processes on sequential
  data}, in International Conference on Machine Learning, 2021, pp.~6233--6242.

\bibitem{li2019towards}
{\sc Z.~Li, J.-F. Ton, D.~Oglic, and D.~Sejdinovic}, {\em Towards a unified
  analysis of random {F}ourier features}, in International Conference on
  Machine Learning, 2019, pp.~3905--3914.

\bibitem{liao2020random}
{\sc Z.~Liao, R.~Couillet, and M.~W. Mahoney}, {\em A random matrix analysis of
  random {F}ourier features: beyond the {G}aussian kernel, a precise phase
  transition, and the corresponding double descent}, in Advances in Neural
  Information Processing Systems, 2020, pp.~13939--13950.

\bibitem{liu2021random}
{\sc F.~Liu, X.~Huang, Y.~Chen, and J.~A. Suykens}, {\em Random features for
  kernel approximation: A survey on algorithms, theory, and beyond}, IEEE
  Transactions on Pattern Analysis and Machine Intelligence, 44 (2021),
  pp.~7128--7148.

\bibitem{lyons2014rough}
{\sc T.~Lyons}, {\em Rough paths, signatures and the modelling of functions on
  streams}, 2014, \url{https://arxiv.org/abs/1405.4537}.

\bibitem{lyons2017sketching}
{\sc T.~Lyons and H.~Oberhauser}, {\em Sketching the order of events}, 2017,
  \url{https://arxiv.org/abs/1708.09708}.

\bibitem{lyons2007differential}
{\sc T.~J. Lyons, M.~Caruana, and T.~L{\'e}vy}, {\em Differential equations
  driven by rough paths}, Springer, 2007.

\bibitem{morrill2021generalised}
{\sc J.~Morrill, A.~Fermanian, P.~Kidger, and T.~Lyons}, {\em A generalised
  signature method for multivariate time series feature extraction}, 2021,
  \url{https://arxiv.org/abs/2006.00873}.

\bibitem{murray1936rings}
{\sc F.~J. Murray and J.~v. Neumann}, {\em On rings of operators}, Annals of
  Mathematics,  (1936), pp.~116--229.

\bibitem{ni2021sig}
{\sc H.~Ni, L.~Szpruch, M.~Sabate-Vidales, B.~Xiao, M.~Wiese, and S.~Liao},
  {\em Sig-{W}asserstein {GAN}s for time series generation}, in International
  Conference on AI in Finance, 2022, pp.~1--8.

\bibitem{scikit-learn}
{\sc F.~Pedregosa, G.~Varoquaux, A.~Gramfort, V.~Michel, B.~Thirion, O.~Grisel,
  M.~Blondel, P.~Prettenhofer, R.~Weiss, V.~Dubourg, J.~Vanderplas, A.~Passos,
  D.~Cournapeau, M.~Brucher, M.~Perrot, and E.~Duchesnay}, {\em Scikit-learn:
  Machine learning in {P}ython}, Journal of Machine Learning Research, 12
  (2011), pp.~2825--2830.

\bibitem{petitjean2012satellite}
{\sc F.~Petitjean, J.~Inglada, and P.~Gan{\c{c}}arski}, {\em Satellite image
  time series analysis under time warping}, IEEE transactions on geoscience and
  remote sensing, 50 (2012), pp.~3081--3095.

\bibitem{rahimi2007random}
{\sc A.~Rahimi and B.~Recht}, {\em Random features for large-scale kernel
  machines}, in Advances in Neural Information Processing Systems, 2007,
  pp.~1177--1184.

\bibitem{rahimi2008uniform}
{\sc A.~Rahimi and B.~Recht}, {\em Uniform approximation of functions with
  random bases}, in Allerton Conference on Communication, Control, and
  Computing, 2008, pp.~555--561.

\bibitem{rahimi2008weighted}
{\sc A.~Rahimi and B.~Recht}, {\em Weighted sums of random kitchen sinks:
  Replacing minimization with randomization in learning}, Advances in Neural
  Information Processing Systems,  (2008), pp.~1313--1320.

\bibitem{rakhshan2020tensorized}
{\sc B.~Rakhshan and G.~Rabusseau}, {\em Tensorized random projections}, in
  International Conference on Artificial Intelligence and Statistics, 2020,
  pp.~3306--3316.

\bibitem{raschka2020machine}
{\sc S.~Raschka, J.~Patterson, and C.~Nolet}, {\em Machine learning in
  {P}ython: Main developments and technology trends in data science, machine
  learning, and artificial intelligence}, Information, 11 (2020).

\bibitem{rasmussen2006gaussian}
{\sc C.~E. Rasmussen and C.~K.~I. Williams}, {\em Gaussian Processes for
  Machine Learning}, MIT Press, 2006.

\bibitem{reutenauer2003free}
{\sc C.~Reutenauer}, {\em Free {L}ie algebras}, Handbook of Algebra, 3 (2003),
  pp.~887--903.

\bibitem{rudi17generalization}
{\sc A.~Rudi and L.~Rosasco}, {\em Generalization properties of learning with
  random features}, in Advances in Neural Information Processing Systems, 2017,
  pp.~3218--3228.

\bibitem{rudin1976principles}
{\sc W.~Rudin}, {\em Principles of mathematical analysis}, McGraw-hill New
  York, 1976.

\bibitem{rudin2017fourier}
{\sc W.~Rudin}, {\em Fourier Analysis on Groups}, Dover Publications, 2017.

\bibitem{salvi2021signature}
{\sc C.~Salvi, T.~Cass, J.~Foster, T.~Lyons, and W.~Yang}, {\em The signature
  kernel is the solution of a {G}oursat {PDE}}, SIAM Journal on Mathematics of
  Data Science, 3 (2021), pp.~873--899.

\bibitem{sambale2022notes}
{\sc H.~Sambale}, {\em Some notes on concentration for $\alpha$-subexponential
  random variables}, 2022, \url{https://arxiv.org/abs/2002.10761}.

\bibitem{sasvari2013multivariate}
{\sc Z.~Sasv{\'a}ri}, {\em Multivariate Characteristic and Correlation
  Functions}, De Gruyter, 2013.

\bibitem{scholkopf2002learning}
{\sc B.~Sch{\"o}lkopf and A.~Smola}, {\em Learning with Kernels: Support Vector
  Machines, Regularization, Optimization, and Beyond}, MIT Press, 2002.

\bibitem{shawe2004kernel}
{\sc J.~Shawe-Taylor and N.~Cristianini}, {\em Kernel Methods for Pattern
  Analysis}, Cambridge University Press, 2004.

\bibitem{sriperumbudur2010relation}
{\sc B.~Sriperumbudur, K.~Fukumizu, and G.~Lanckriet}, {\em On the relation
  between universality, characteristic kernels and {RKHS} embedding of
  measures}, in International Conference on Artificial Intelligence and
  Statistics, 2010, pp.~773--780.

\bibitem{sriperumbudur2015optimal}
{\sc B.~Sriperumbudur and Z.~Szab{\'o}}, {\em Optimal rates for random
  {F}ourier features}, in Advances in Neural Information Processing Systems,
  2015, pp.~1144--1152.

\bibitem{sriperumbudur22approximate}
{\sc B.~K. Sriperumbudur and N.~Sterge}, {\em Approximate kernel {PCA} using
  random features: Computational vs. statistical trade-off}, Annals of
  Statistics,  (2022), pp.~2713--2736.

\bibitem{steinwart2008support}
{\sc I.~Steinwart and A.~Christmann}, {\em Support Vector Machines}, Springer
  Science \& Business Media, 2008.

\bibitem{sun2018but}
{\sc Y.~Sun, A.~Gilbert, and A.~Tewari}, {\em But how does it work in theory?
  {L}inear {SVM} with random features}, in Advances in Neural Information
  Processing Systems, 2018, pp.~3379--3388.

\bibitem{sun2021tensor}
{\sc Y.~Sun, Y.~Guo, J.~A. Tropp, and M.~Udell}, {\em Tensor random projection
  for low memory dimension reduction}, 2021,
  \url{https://arxiv.org/abs/2105.00105}.

\bibitem{sutherland2015error}
{\sc D.~J. Sutherland and J.~Schneider}, {\em On the error of random {F}ourier
  features}, in Conference on Uncertainty in Artificial Intelligence, 2015,
  pp.~862--871.

\bibitem{szabo2019kernel}
{\sc Z.~Szab{\'o} and B.~Sriperumbudur}, {\em On kernel derivative
  approximation with random {F}ourier features}, in International Conference on
  Artificial Intelligence and Statistics, 2019, pp.~827--836.

\bibitem{tancik2020fourier}
{\sc M.~Tancik, P.~Srinivasan, B.~Mildenhall, S.~Fridovich-Keil, N.~Raghavan,
  U.~Singhal, R.~Ramamoorthi, J.~Barron, and R.~Ng}, {\em Fourier features let
  networks learn high frequency functions in low dimensional domains}, in
  Advances in Neural Information Processing Systems, 2020, pp.~7537--7547.

\bibitem{toth2021seq2tens}
{\sc C.~Toth, P.~Bonnier, and H.~Oberhauser}, {\em Seq2{T}ens: An efficient
  representation of sequences by low-rank tensor projections}, in International
  Conference on Learning Representations, 2021,
  \url{https://openreview.net/forum?id=dx4b7lm8jMM}.

\bibitem{toth2022capturing}
{\sc C.~Toth, D.~Lee, C.~Hacker, and H.~Oberhauser}, {\em Capturing graphs with
  hypo-elliptic diffusions}, in Advances in Neural Information Processing
  Systems, 2022, pp.~38803--38817.

\bibitem{toth2020bayesian}
{\sc C.~T{\'{o}}th and H.~Oberhauser}, {\em Bayesian learning from sequential
  data using {G}aussian processes with signature covariances}, in International
  Conference on Machine Learning, 2020, pp.~9548--9560.

\bibitem{ullah2018streaming}
{\sc E.~Ullah, P.~Mianjy, T.~V. Marinov, and R.~Arora}, {\em Streaming kernel
  {PCA} with $\tilde {O}(\sqrt{n})$ random features}, in Advances in Neural
  Information Processing Systems, 2018.

\bibitem{van1996weak}
{\sc A.~van~der Vaart and J.~Wellner}, {\em Weak Convergence and Empirical
  Processes: With Applications to Statistics}, Springer, 1996.

\bibitem{vershynin2018high}
{\sc R.~Vershynin}, {\em High-Dimensional Probability: An Introduction with
  Applications in Data Science}, Cambridge University Press, 2018.

\bibitem{wacker2022local}
{\sc J.~Wacker and M.~Filippone}, {\em Local random feature approximations of
  the {G}aussian kernel}, Procedia Computer Science, 207 (2022), pp.~987--996.

\bibitem{wacker2022improved}
{\sc J.~Wacker, M.~Kanagawa, and M.~Filippone}, {\em Improved random features
  for dot product kernels}, 2022, \url{https://arxiv.org/abs/2201.08712}.

\bibitem{wacker2022complex}
{\sc J.~Wacker, R.~Ohana, and M.~Filippone}, {\em Complex-to-real sketches for
  tensor products with applications to the polynomial kernel}, in International
  Conference on Artificial Intelligence and Statistics, 2023, pp.~5181--5212.

\bibitem{williams2000nystrom}
{\sc C.~Williams and M.~Seeger}, {\em Using the {N}ystr\"{o}m method to speed
  up kernel machines}, in Advances in Neural Information Processing Systems,
  2000, pp.~682--688.

\bibitem{wu2018random}
{\sc L.~Wu, I.~E.-H. Yen, J.~Yi, F.~Xu, Q.~Lei, and M.~Witbrock}, {\em Random
  warping series: A random features method for time-series embedding}, in
  International Conference on Artificial Intelligence and Statistics, 2018,
  pp.~793--802.

\bibitem{yokonuma1992tensor}
{\sc T.~Yokonuma}, {\em Tensor Spaces and Exterior Algebra}, American
  Mathematical Society, 1992.

\bibitem{yu2016orthogonal}
{\sc F.~X.~X. Yu, A.~T. Suresh, K.~M. Choromanski, D.~N. Holtmann-Rice, and
  S.~Kumar}, {\em Orthogonal random features}, in Advances in Neural
  Information Processing Systems, 2016, pp.~1975--1983.

\bibitem{zhang2021concentration}
{\sc H.~Zhang and S.~X. Chen}, {\em Concentration inequalities for statistical
  inference}, Communications in Mathematical Research, 37 (2021), pp.~1--85.

\end{thebibliography}

\newpage

\appendix

\section{Concentration of measure} \label{apx:prob}

\paragraph{Classic inequalities} The following inequalities are classic, and their proofs are in analysis textbooks, e.g.~\cite{dudley2002real,rudin1976principles}. Firstly, Jensen's inequality is useful for convex (concave) expectations.
\begin{lemma}[Jensen's inequality] \label{lem:jensen}
    Let $X$ be an integrable random variable, and $f: \bbR \to \bbR$ a convex function, such that $f\pars{X}$ is also integrable. Then, the following inequality holds:
    \begin{align}
        f\pars{\expe{X}} \leq \expe{f\pars{X}}.
    \end{align}
\end{lemma}

Hölder's inequality is a generalization of the Cauchy-Schwarz inequality to $L^p$ spaces.

\begin{lemma}[Hölder's inequality] \label{lem:holder}
    Let $p, q \geq 1$
    such that $\frac{1}{p} + \frac{1}{q} = 1$. Let $X$ and $Y$ respectively be $L^p$ and $L^q$ integrable random variables, i.e.~$\expe{\abs{X}^p} < \infty$ and $\expe{\abs{Y}^q} < \infty$. Then, $XY$ is integrable, and it holds that
    \begin{align}
        \expe{\abs{XY}} \leq \bbE^{1/p}\bracks{\abs{X}^p} \bbE^{1/q}\bracks{\abs{Y}^q}.
    \end{align}
\end{lemma}

Although not inherently a probabilistic inequality, Young's inequality can be used to decouple products of random variables.

\begin{lemma}[Young's inequality]
    Let $p,q > 0$ with $\frac{1}{p} + \frac{1}{q} = 1$. Then, for every $a,b \geq 0$
    \begin{align}
        ab \leq \frac{a^p}{p} + \frac{b^q}{q}.
    \end{align}
\end{lemma}

\begin{lemma}[Reverse Young's inequality] \label{lem:reverse}
    Let $p,q > 0$ such that $\frac{1}{p} - \frac{1}{q} = 1$. Then, for every $a \geq 0$ and $b > 0$, it holds that
    \begin{align}
        ab \geq \frac{a^p}{p} - \frac{b^{-q}}{q}.
    \end{align}
\end{lemma}
\begin{proof} Apply Young's inequality with $p^\p = \frac{1}{p}$ and $q^\p = \frac{q}{p}$ to $a^\p = (ab)^p$ and $b^\p = b^{-p}$.
\end{proof}

\paragraph{Subexponential concentration}

Next, we state a variation on the well-known Bernstein inequality, which holds for  random variables in the subexponential class. The condition \eqref{eq:bernstein_cond_twot} on a random variable $X$, in this case, is formulated as a moment-growth bound, called a Bernstein moment condition. We show in Lemmas \ref{lem:alpha_bernstein_cond} and \ref{lem:alpha_exp_norm_to_bernstein} (in a more general setting) that \eqref{eq:bernstein_cond_twot} is equivalent to the random variable being subexponential. Further, note that the condition itself is given in terms of non-centered random variables, while the statement of the theorem is about their centered counterparts. For similar results, see \cite[Sec.~2.2.2]{van1996weak}.
\begin{theorem}[Bernstein inequality - one-tailed] \label{thm:bernstein_onetail}
Let $X_1, \dots, X_n$ be independent random variables satisfying the following moment-growth condition for some $S, R > 0$,
    \begin{align} \label{eq:bernstein_cond_onet}
        \expe{X_i^k} \leq \frac{k! S^2 R^{k-2}}{2} \quad \text{for all} \spc k \geq 2.
    \end{align}
Then, it holds for $\tilde X_i \coloneqq X_i - \expe{X_i}$ that
\begin{align} \label{eq:bernstein_ineq_onet}
    \prob{\sumnolim_{i=1}^n \tilde X_i \geq t} \leq \exp\pars{\frac{-t^2}{2(nS^2 + Rt)}}.
\end{align}
\end{theorem}
\begin{proof}
    We have for $\lambda > 0$
    \begin{align}
        \prob{\sumnolim_{i=1}^n \tilde X_i \geq t}
        & \stackrel{(a)}{\leq}
        \exp(-\lambda t) \expe{\exp\pars{\lambda \sumnolim_{i=1}^n \tilde X_i}}
        \\&\stackrel{(b)}{=}
        \exp\pars{-\lambda t - \lambda \sumnolim_{i=1}^n \expe{X_i}} \prod_{i=1}^n \expe{\exp(\lambda X_i)},
        \label{eq:chernoff}
    \end{align}
    where (a) holds for any $\lambda  >0$ by the  Chernoff bound applied to $\sum_{i=1}^n \tilde X_i$, and in (b) we used the independence of $X_i$-s. Bounding the moment-generating function of $X_i$, one gets
    \begin{align}
        \expe{\exp(\lambda X_i)}
        &\stackrel{\text{(d)}}{=}
        1 + \lambda \expe{X_i} + \sum_{k=2}^\infty \frac{\lambda^k}{k!} \expe{X_i^k}
        \stackrel{\text{(e)}}{\leq}
        1 + \lambda \expe{X_i} + \frac{S^2 \lambda^2}{2} \sum_{k=0}^\infty \lambda^k R^k
        \\
        &\stackrel{\text{(f)}}{=} 1 + \lambda \expe{X_i} + \frac{S^2 \lambda^2}{2(1- \lambda R)}
         \stackrel{\text{(g)}}{\leq}
        \exp\pars{\lambda \expe{X_i} + \frac{S^2\lambda^2}{2(1 - \lambda R)}}, \label{eq:mgf_bound}
    \end{align}
    where (d) is due to the Taylor expansion of the exponential function, (e) implied by \eqref{eq:bernstein_cond_onet}, (f) holds by the sum of geometric series for any $\lambda < 1/R$, (g) follows from the inequality $1+x \leq \exp(x)$ for all $x \in \bbR$. Using this to bound \eqref{eq:chernoff},
    \begin{align}
        \prob{\sumnolim_{i=1}^n \tilde X_i \geq t}
        &\leq
        \exp\pars{-\lambda t - \lambda \sumnolim_{i=1}^n \expe{X_i}} \prod_{i=1}^n \exp\pars{\lambda \expe{X_i} + \frac{S^2\lambda^2}{2(1 - \lambda R)}}
        \\
        &=
        \exp\pars{-\lambda t + \frac{n S^2\lambda^2}{2(1 - \lambda R)}}. \label{eq:bernstein_ineq_lambda}
    \end{align}
    Finally, choosing $\lambda = t / (nS^2 + Rt)$ gives the statement.
\end{proof}

\begin{corollary}[Bernstein inequality - two-tailed] \label{thm:bernstein_twotail}
Let $X_1, \dots, X_n$ be independent random variables satisfying the the following moment-growth condition for some $S, R > 0$,
    \begin{align} \label{eq:bernstein_cond_twot}
        \expe{\abs{X_i}^k} \leq \frac{k! S^2 R^{k-2}}{2} \quad \text{for all} \spc k \geq 2.
    \end{align}
Then, it holds for $\tilde X_i \coloneqq X_i - \expe{X_i}$ that
\begin{align} \label{eq:bernstein_ineq_twot}
    \prob{\abs{\sumnolim_{i=1}^n \tilde X_i} \geq t} \leq 2  \exp\pars{\frac{-t^2}{2(nS^2 + Rt)}}.
\end{align}
\end{corollary}
\begin{proof}
Applying Theorem \ref{thm:bernstein_onetail} to $\sumnolim_{i=1}^n \tilde X_i$ and $-\sumnolim_{i=1}^n \tilde X_i$, and combining the two bounds gives the two-tailed result.
\end{proof}

\paragraph{$\alpha$-exponential concentration}

In this section, we introduce a specific class of Orlicz norms for heavy-tailed random variables, which generalizes subexponential distributions.
\begin{definition}[$\alpha$-exponential Orlicz norm] \label{def:alpha_subexp_norm}
    Let $\alpha > 0$ and define the function
    \begin{align}
    \Psi_\alpha: \bbR \to \bbR, \quad \Psi_\alpha(x) \coloneqq \exp\pars{x^\alpha} - 1 \quad \text{for all} \spc x \in \bbR.
    \end{align}
    The $\alpha$-exponential Orlicz (quasi-)norm of a random variable $X$ is defined as
    \begin{align}
        \norm{X}_{\Psi_\alpha} \coloneqq \inf\curls{t > 0 : \expe{\Psi_\alpha\pars{\frac{\abs{X}}{t}}} \leq 1},
    \end{align}
    adhering to the standard convention that $\inf \emptyset = \infty$.
\end{definition}
If a random variable $X$ satisfies $\norm{X}_{\Psi_\alpha} < \infty$, it is either called an $\alpha$-(sub)exponential random variable (if $\alpha < 1$) \cite{sambale2022notes,gotze2021concentration,chamakh2020orlicz}, or sub-Weibull of order-$\alpha$ \cite{kuchibhotla2022moving,zhang2021concentration}.

An alternative characterization of the $\alpha$-exponential norm is the following tail-bound.

\begin{remark}[Tail bound] \label{remark:orlicz_tail} Let $\alpha > 0$ and $X$ be a random variable such that $\norm{X}_{\Psi_\alpha} < \infty$. Then,
    \begin{align}
        \prob{\abs{X} \geq \epsilon} \leq 2\exp\pars{- \frac{\epsilon^\alpha}{\norm{X}_{\Psi_\alpha}^\alpha}}.
    \end{align}
\end{remark}
\begin{proof}
    Due to Definition \ref{def:alpha_subexp_norm}, $\expe{\exp\pars{\frac{\abs{X}^\alpha}{\norm{X}^\alpha_{\Psi_\alpha}}}} \leq 2$, hence by Markov's inequality
    \begin{align}
        \bbP\bracks{\abs{X} \geq \epsilon}
        =
        \prob{\frac{\abs{X}^\alpha}{\norm{X}_{\Psi_\alpha}^\alpha}
        \geq
        \frac{\epsilon^\alpha}{\norm{X}_{\Psi_\alpha}^\alpha}}
        \leq
        e^{-\frac{\epsilon^\alpha}{\norm{X}_{\Psi_\alpha}^\alpha}} \expe{\exp\pars{\frac{\abs{X}^\alpha}{\norm{X}_{\Psi_\alpha}^\alpha)}}}
        \leq
        2 e^{-\frac{\epsilon^\alpha}{\norm{X}_{\Psi_\alpha}^\alpha}}.
    \end{align}
\end{proof}

Note that although $\norm{\cdot}_{\Psi_\alpha}$ is often referred to as a norm, it does not satisfy the triangle inequality for $\alpha < 1$, although we can still relate the norm of the sum to the sum of the norms.
\begin{lemma}[Generalized triangle inequality for Orlicz norm] \label{lem:alpha_exp_triangle}
    It holds for any random variables $X, Y$ and $\alpha > 0$ that
    \begin{align}
        \norm{X + Y}_{\Psi_\alpha} \leq C_\alpha \pars{\norm{X}_{\Psi_\alpha} + \norm{Y}_{\Psi_\alpha}},
    \end{align}
    where $C_\alpha = 2^{1/\alpha}$ if $\alpha < 1$ and $1$ otherwise.
\end{lemma}
\begin{proof}
    See Lemma~A.3. in \cite{gotze2021concentration}.
\end{proof}

A useful property of $\alpha$-exponential norms is that they satisfy a Hölder-type inequality.
\begin{lemma}[Hölder inequality for Orlicz norm] \label{lem:alpha_exp_holder}
    It holds for any random variables $X_1, \dots, X_k$ and $\alpha_1, \dots, \alpha_k > 0$ that
    \begin{align}
        \norm{\prod\nolimits_{i=1}^k X_i}_{\Psi_\alpha} \leq \prod\nolimits_{i=1}^k \norm{X_i}_{\Psi_{\alpha_i}},
    \end{align}
    where $\alpha \coloneqq \pars{\sum_{i=1}^k \alpha_i^{-1}}^{-1}$.
\end{lemma}
\begin{proof}
    See Lemma~A.1. in \cite{gotze2021concentration}.
\end{proof}

Next, we show in the following two lemmas that a random variable $X$ is $\alpha$-exponential if and only if $\abs{X}^\alpha$ satisfies a Bernstein moment-growth condition.

\begin{lemma}[Bernstein condition implies Orlicz norm bound] \label{lem:alpha_bernstein_cond}
    Let $\alpha > 0$ and $X$ be a random variable that satisfies for $S, R > 0$ that
    \begin{align} \label{eq:alpha_bernstein_cond}
        \expe{\abs{X}^{k \alpha}} \leq \frac{k! S^2 R^{k-2}}{2} \quad \text{for all} \spc k \geq 2.
    \end{align}
    Then, it holds that
    \begin{align}
        \norm{X}_{\Psi_\alpha} \leq \pars{S \vee R }^{1/\alpha}.
    \end{align}
\end{lemma}
\begin{proof}
    Firstly, due to Jensen's inequality (Lemma \ref{lem:jensen}) and \eqref{eq:alpha_bernstein_cond}, we have
    \begin{align}
        \expe{\abs{X}^\alpha} \leq \bbE^{1/2}\bracks{\abs{X}^{2\alpha}} \leq S. \label{eq:1st_alpha_moment}
    \end{align}
    We proceed similarly to the proof of Theorem \ref{thm:bernstein_onetail}. For $t > 0$, we have
    \begin{align}
        \expe{\exp{\frac{\abs{X}^\alpha}{t^\alpha}}}
        &\stackrel{\text{(a)}}{=}
        1 + \frac{\expe{\abs{X}^\alpha}}{t^\alpha} + \sum_{m=2}^\infty \frac{\expe{\abs{X}^{m \alpha}}}{t^{m \alpha} m!}
        \stackrel{\text{(b)}}{\leq}
        1 + \frac{S}{t^\alpha} + \frac{S^2}{2t^{2\alpha}} \sum_{m=0}^\infty \pars{\frac{R}{t^{\alpha}}}^m
        \\
        &\stackrel{\text{(c)}}{=}
        \underbrace{1 + \frac{S}{t^\alpha} + \frac{S^2}{2t^{\cancel{2}\alpha}} \frac{\cancel{t^\alpha}}{t^\alpha - R}}_{f(t)},
    \end{align}
    where (a) is due to the Taylor expansion of the exponential function, (b) is due to \eqref{eq:alpha_bernstein_cond} and \eqref{eq:1st_alpha_moment}, (c) is the sum of a geometric series for $R < t^\alpha$. Defining $f(t) \coloneqq 1 + \frac{S}{t^\alpha}\pars{1 + \frac{S}{2(t^\alpha - R)}}$ and solving for $f(t) \leq 2$ leads to the quadratic inequality
    \begin{align}
        0 \leq t^{2\alpha} - (S + R)t^\alpha + SR - \frac{S^2}{2},
    \end{align}
    which has roots $t^\alpha_{1,2} = \frac{1}{2}\pars{S + R \pm \sqrt{\pars{S - R}^2 + 2S^2}}$. We discard the left branch, which violates the condition $R < t^\alpha$ since
    \begin{align}
        t_2^\alpha = \frac{1}{2}\pars{S + R - \sqrt{(S - R)^2 + 2S^2}} \leq \frac{1}{2}\pars{S + R - \abs{S - R}} = S \wedge R \leq R.
    \end{align}
    Now, as the inequality $\sqrt{x^2 + y^2} \leq \abs{x} + \abs{y}$ holds for all $x, y \in \bbR$, we get
    \begin{align}
        t_1^\alpha &\leq \frac{1}{2}\pars{S + R + \sqrt{(S-R)^2 + 2S^2}} \leq \frac{1}{2}\pars{S + R + \abs{S-R} + \sqrt{2}S}
        \\
        &= S \vee R + \frac{\sqrt{2}}{2}S \leq 2 (S \vee R)
    \end{align}
    and hence choosing $t \geq \pars{2 (S \vee R)}^{1/\alpha} \geq t_1$ implies $f(t) \leq 2$, and we are done.
\end{proof}

The other direction is proven in the following lemma.

\begin{lemma}[Finite Orlicz norm implies Bernstein condition]  \label{lem:alpha_exp_norm_to_bernstein}
    Let $\alpha > 0$ and $X$ be a random variable such that $\norm{X}_{\Psi_\alpha} < \infty$. Then,
    \begin{align}
        \expe{\abs{X}^{k \alpha}} \leq \frac{k! S^2 R^{k-2}}{2} \quad \text{for all } k \geq 2,
    \end{align}
    where  $S \coloneqq \sqrt{2} \norm{X}_{\Psi_\alpha}^\alpha$ and $R \coloneqq \norm{X}_{\Psi_\alpha}^\alpha$.
\end{lemma}
\begin{proof}
    We have that
    \begin{align}
        \expe{\abs{X}^{k \alpha}} 
        & \stackrel{\text{(a)}}{=}
        k! \norm{X}_{\Psi_\alpha}^{k\alpha} \expe{\frac{1}{k!} \pars{\frac{\abs{X}}{\norm{X}_{\Psi_\alpha}}} ^{k \alpha}}
        \stackrel{\text{(b)}}{\leq}
        k! \norm{X}_{\Psi_\alpha}^{k \alpha} \expe{\exp\pars{\frac{\abs{X}^\alpha}{\norm{X}^\alpha_{\Psi_\alpha}}} - 1}
        \\
        &\stackrel{\text{(c)}}{\leq}
        k! \norm{X}_{\Psi_\alpha}^{k \alpha}
        \stackrel{\text{(d)}}{=}
        \frac{k! \pars{\sqrt{2}\norm{X}_{\Psi_\alpha}^{\alpha}}^2 \pars{\norm{X}_{\Psi_\alpha}^\alpha}^{k-2}}{2},
    \end{align}
    where (a) is simply multiplying and dividing by the same values, (b) is the inequality $1 + x^k / k! \leq e^x$ for $x > 0$, (c) follows from Definition \ref{def:alpha_subexp_norm}, (d) is reorganizing terms to the required form.
\end{proof}

We adapt the following concentration inequality from \cite{kuchibhotla2022moving} for $\alpha$-exponential summation.

\begin{theorem}[Concentration inequality for $\alpha$-subexponential summation] \label{thm:alpha_subexp_concentration}
Let $\alpha \in (0, 1)$ and $X_1, \dots, X_n$ be independent, centered random variables with $\norm{X_i}_{\Psi_\alpha} \leq M_\alpha$ for all $i \in [n]$. Then, it holds for $t> 0$ that
\begin{align} \label{eq:alpha_ineq1}
    \prob{\abs{\sumnolim_{i=1}^n X_i} \geq  C_\alpha \pars{2 \sqrt{nt} + \sqrt{2} 4^{1/\alpha} t^{1/\alpha}}} \leq 2e^{-t},
\end{align}
where $C_\alpha \coloneqq \sqrt{8} e^4 (2\pi)^{1/4} e^{1/24} (2e/\alpha)^{1/\alpha} M_\alpha$. Alternatively, it holds for $\epsilon > 0$ that
\begin{align} \label{eq:alpha_ineq2}
    \prob{\abs{\sumnolim_{i=1}^n X_i} \geq \epsilon} \leq 2\exp\pars{-\frac{1}{4}\min\curls{\pars{\frac{\epsilon}{2\sqrt{n}C_\alpha}}^2, \pars{\frac{\epsilon}{\sqrt{8}C_\alpha }}^{\alpha}}}.
\end{align}
\end{theorem}
\begin{proof}
    The inequality \eqref{eq:alpha_ineq1} follows directly from \cite[Theorem~3.1]{kuchibhotla2022moving}.

    To show \eqref{eq:alpha_ineq2}, let $g_1(t) \coloneqq 2C_\alpha \sqrt{nt}$ and $g_2(t) \coloneqq \sqrt{2}C_\alpha4^{1/\alpha} t^{1/\alpha}$. Now, for $t > 0$
    \begin{align}
        \prob{\abs{\sumnolim_{i=1}^n X_i} \geq 2\pars{g_1(t) \vee g_2(t)}} \leq \prob{\abs{\sumnolim_{i=1}^n X_i} \geq g_1(t) + g_2(t)} \leq 2e^{-t},
    \end{align}
    which is equivalently written as
    \begin{align}
        \prob{\abs{\sumnolim_{i=1}^n X} \geq \epsilon} \leq 2\exp\pars{- \pars{g^{-1}_1(\epsilon/2) \wedge g^{-1}_2(\epsilon/2)}}.
    \end{align}
\end{proof}

\paragraph{Hypercontractivity}
Here we provide an alternative approach for the concentration of heavy-tailed random variables, specifically for polynomials of Gaussian random variables. The following lemma states a moment bound for such Gaussian chaoses, considered in e.g.~\cite{janson1997gaussian,boucheron2013concentration}.

\begin{lemma}[Moment bound for Gaussian polynomial] \label{lem:hyper_moments}
    Consider a degree-$p$ polynomial $f(X) = f(X_1, \dots, X_n)$ of independent centered Gaussian random variables, $X_1, \dots X_n \stackrel{\iid}{\sim} \cN(0, 1)$. Then, for all $k \geq 2$
    \begin{align}
        \bbE^{1/k}\bracks{\abs{f(X)}^k} \leq (k-1)^{p/2} \bbE^{1/2}\bracks{\abs{f(X)}^2}.
    \end{align}
\end{lemma}

\begin{theorem}[Concentration inequality for Gaussian polynomial] \label{thm:hyper_concentration} Consider a degree-$p$ polynomial $f(X) = f(X_1, \dots, X_n)$ of independent centered Gaussian random variables, $X_1, \dots, X_n \stackrel{\iid}{\sim} \cN(0, 1)$. Then, for $p \geq 2$ and $\epsilon > 0$
\begin{align}
    \prob{\abs{f(X) - \expe{f(X)}} \geq \epsilon} \leq 2\exp\pars{-\frac{\epsilon^{2/p}}{2\sqrt{2}e\bbV^{1/p}\bracks{f(X)}}},
\end{align}
where $\bbV\bracks{\cdot}$ denotes the variance of a random variable.
\end{theorem}
\begin{proof}
    Without loss of generality, we may assume that $\expe{f(X)} = 0$ and $\bbV\bracks{f(X)} = 1$.
    Since Lemma \ref{lem:hyper_moments} holds, we have for $p, k \geq 2$
    \begin{align}
        \expe{\abs{f(X)}^{2k/p}} &\stackrel{\text{(a)}}{\leq} \bbE^{2/p}\bracks{\abs{f(X)}^k} 
        \stackrel{\text{(b)}}{\leq} (k-1)^{k} \leq k^{k} \stackrel{\text{(c)}}{\leq} k! e^k
        \leq \frac{k! (\sqrt{2}e)^2 e^{k-2}}{2},
    \end{align}
    where (a) holds due to Jensen inequality since $(\cdot)^{2/p}$ is concave, (b) is Lemma \ref{lem:hyper_moments}, and (c) is due to Stirling's approximation. Hence, $\abs{f(X)}^{2/p}$ satisfies a Bernstein condition with $S = \sqrt{2}e$ and $R = e$. Therefore, by Lemma \ref{lem:alpha_bernstein_cond}, we have $\norm{f(X)}_{\Psi_{2/p}} \leq (2\sqrt{2}e)^{p/2}$.

    Then, by Remark \ref{remark:orlicz_tail}, it holds that
    \begin{align}
        \prob{\abs{f(X)} \geq \epsilon} \leq 2 \exp\pars{- \frac{\epsilon^{2/p}}{\norm{f(X)}_{\Psi_{2/p}}^{2/p}}} \leq 2 \exp\pars{- \frac{\epsilon^{2/p}}{2\sqrt{2}e}}.
    \end{align}
\end{proof}

\section{Random Fourier Features} \label{apx:RFF}

\cite{rahimi2007random} provides a uniform bound for the \RFF{} error over a compact and convex domain $\cX \subset \bbR^d$ with diameter $\abs{\cX}$ for some absolute constant $C > 0$,
\begin{align} \label{eq:rff_conv_prob}
    \prob{\sup_{\bx, \by \in \cX}\abs{\rffkernel(\bx, \by) - \kernel(\bx, \by)} \geq \epsilon} \leq C \pars{\frac{\sigma_{{\Lambda}} \abs{\cX}}{\epsilon}}^2 \exp\pars{-\frac{-{\dimRFF}\epsilon^2}{4(d+2)}},
\end{align}
where $\sigma_\Lambda^2 = \bbE_{\bw \sim {\Lambda}}\bracks{\norm{\bw}^2}$, and \cite{sutherland2015error} shows that $C \leq 66$. Equation \eqref{eq:rff_conv_prob} implies \cite[Sec.~2]{sriperumbudur2015optimal} 
\begin{align}
    \sup_{\bx, \by \in \cX} \abs{\rffkernel(\bx, \by) - \kernel(\bx, \by)} = O_p\pars{\abs{\cX}\sqrt{\dimRFF^{-1}\log \dimRFF}},
\end{align}
where the $X_n = O_p(a_n)$ notation for $a_n > 0$ refers to $\lim_{C \to 0} \limsup_{n \to \infty} \prob{\frac{\abs{X_n}}{a_n} > C} = 0$.

The converse result\footnote{\label{footnote:exp-improved-guarantee}These error guarantees can also be improved exponentially both for the approximation of kernel values \cite{sriperumbudur2015optimal} and kernel derivatives \cite{szabo2019kernel,chamakh2020orlicz} in terms of the size of the domain where it holds, i.e.~$\abs{\cX}$.} for \RFF{} derivatives was shown in \cite[Thm.~5]{sriperumbudur2015optimal}. The idea of the proof is to cover the input domain by an $\epsilon$-net, and control the approximation error on the centers, while simultaneously controlling the Lipschitz constant of the error to get the bound to hold uniformly. We provide an adapted version with two main differences: \begin{enumerate*}[label=(\arabic*)] \item using the Bernstein inequality from Corollary \ref{thm:bernstein_twotail}, where the Bernstein condition is given in terms of non-centered random variables, and \item using the covering numbers of \cite{cucker2002mathematical}.\end{enumerate*} We will use this theorem in proving Theorem \ref{thm:main} for controlling the approximation error of the derivatives of \RFF s.

Given $\bp \in \bbN^d$, we denote $\abs{\bp} \coloneqq p_1 + \ldots + p_d$, for a function $f: \bbR^d \to \bbR$ the $\bp$-th partial derivative by $\partial^\bp f(\bz) \coloneqq \frac{\partial^{\abs{\bp}} f(\bz)}{\partial^{p_1} z_1 \ldots \partial^{p_d} z_d}$, for a vector $\bw \in \bbR^d$ the $\bp$-th power by $\bw^\bp \coloneqq w_1^{p_1}\cdots w_d^{p_d}$.

\begin{theorem}[Concentration inequality for \RFF{} kernel derivatives] \label{thm:rff_derivative_approx}
    Let $\bp, \bq \in \bbN^d$ and $\kernel: \bbR^d \times \bbR^d \to \bbR$ be a continuous, bounded, translation-invariant kernel such that $\bz \mapsto \nabla \bracks{\partial^{\bp, \bq} \kernel(\bz)}$ is continuous. Let $\cX \subset \bbR^d$ be a compact and convex domain with diameter $\abs{\cX}$, and denote by $D_{\bp, \bq, \cX} \coloneqq \sup_{\bz \in \cX_{\Delta}} \norm{\nabla \bracks{\partial^{\bp, \bq} \kernel(\bz)}}_{2}$, where $\cX_{\Delta} \coloneqq \curls{\bx - \by \given \bx, \by \in \cX}$.
    Let $\Lambda$ be the spectral measure of $\kernel$ satisfying that $E_{\bp, \bq} \coloneqq \bbE_{\bw \sim \Lambda}\bracks{\abs{\bw^{\bp + \bq}} \norm{\bw}_2} < \infty$, and the Bernstein moment condition for some $S, R > 0$,
    \begin{align} \label{eq:moment_growth}
        \bbE_{\bw \sim \Lambda}\bracks{\abs{\bw^{\bp + \bq}}^k} \leq \frac{k! S^2 R^{k-2}}{2} \quad \text{for all} \spc k \geq 2.
    \end{align}
    Then, for $C_{\bp, \bq, \cX} \coloneqq \abs{\cX} (D_{\bp, \bq, \cX} + E_{\bp, \bq})$ and $\epsilon > 0$, 
    \begin{align} \label{eq:rff_derivative_ineq}
        \bbP&\bracks{
            \sup_{\bx, \by \in \cX}\abs{\partial^{\bp, \bq}\rffkernel(\bx, \by) - \partial^{\bp, \bq}\kernel(\bx, \by)}
            \geq
            \epsilon}
        \leq
        16\pars{\frac{C_{\bp,\bq,\cX}}{\epsilon}}^{\frac{d}{d+1}}\exp\pars{\frac{-\dimRFF \epsilon^2}{4(d+1)(2S^2 + R\epsilon)}}.
    \end{align}
\end{theorem}
\begin{proof} We adapt the proof of \cite{rahimi2007random,sriperumbudur2015optimal}. Note that as $\cX$ is compact, so is $\cX_\Delta$, and it can be covered by an $\epsilon$-net of at most $T\coloneqq (4\abs{\cX}/r)^d$ balls of radius $r > 0$ \cite[Prop.~5]{cucker2002mathematical} with centers $\bz_1, \dots, \bz_T \in \cX_\Delta$. Since for all $\bz \in \cX_\Delta$ there exists $i \in \{1, \dots, T\}$ such that $\norm{\bz-\bz_i}_2 \leq r$, it holds for $f(\bz) \coloneqq \partial^{\bp,\bq}\rffkernel(\bz) - \partial^{\bp, \bq}\kernel(\bz)$ and $L_f \coloneqq \sup_{\bs \in \cX_\Delta} \norm{\nabla f(\bs)}_2$ that
    \begin{align} \label{eq:delta_error}
        \abs{f(\bz) - f(\bz_i)} \stackrel{\text{(a)}}{\leq} \sup_{\bs \in \cX_\Delta} \norm{\nabla f(\bs)}_2 \norm{\bz - \bz_i}_2 \stackrel{\text{(b)}}{\leq} \sup_{\bs \in \cX_\Delta} \norm{\nabla f(\bs)}_2 r = L_f r,
    \end{align}
    where (a) is due to the mean-value theorem followed by Cauchy-Schwarz inequality, and (b) is since $\norm{\bz - \bz_i} \leq r$. Now, by triangle inequality, it holds for any $\bz \in \cX_\Delta$ that 
    \begin{align} \label{eq:grad_error}
    \norm{\nabla f(\bz)}_2 \leq \norm{\nabla \bracks{\partial^{\bp, \bq}\rffkernel(\bz)}}_2 +\norm{\nabla \bracks{\partial^{\bp, \bq}\kernel(\bz)}}_2 \leq \norm{\nabla \bracks{\partial^{\bp, \bq}\rffkernel(\bz)}}_2 + D_{\bp, \bq, \cX}.
    \end{align}
    Differentiating $\rffkernel(\bz)$ as defined in \eqref{eq:rffkernel_def}, we get by the chain rule and triangle inequality that
    \begin{align} \label{eq:grad_rff}
        \norm{\nabla \bracks{\partial^{\bp, \bq}\rffkernel(\bz)}}_2 = \frac{1}{\dimRFF}\norm{\sum_{j=1}^{\dimRFF} \bw_j \bw_j^{\bp + \bq} \cos^{(\abs{\bp + \bq})}(\bw_j^\top \bz)}_2 \leq \frac{1}{\dimRFF} \sum_{j=1}^{\dimRFF} \abs{\bw_j^{\bp + \bq}} \norm{\bw_j}_2,
    \end{align}
    where $\cos^{(n)}$ denotes the $n$-th derivative of $\cos$ for $n \in \bbN$.
    The main idea of the proof is then
    \begin{align}
        \bigcap_{i=1}^T \{\abs{f(\bz_i)} < \epsilon/2\} \bigcap \{L_f < \epsilon/2r\} \subseteq \{\abs{f(\bz)} < \epsilon \given \forall \bz \in \cX_\Delta\},
    \end{align}
    since $\abs{f(\bz)} \leq \abs{f(\bz) - f(\bz_i)} + \abs{f(\bz_i)}$ and \eqref{eq:delta_error} holds. Therefore, taking the complement and bounding the union, we get our governing inequality for the uniform error over $\cX_\Delta$:
    \begin{align}
        \prob{\sup_{\bz \in \cX_\Delta} \abs{f(\bz)} \geq \epsilon} \leq \sum_{i=1}^T\prob{\abs{f(\bz_i)} \geq \epsilon/2} + \prob{L_f \geq \epsilon/2r}. \label{eq:error_decomp}
    \end{align}
    Now, we need to bound all probabilities on the RHS. First, we deal with $L_f$, 
    \begin{align}
        \prob{L_f \geq \epsilon/2r} 
        &\stackrel{\text{(c)}}{\leq}
        \expe{L_f} \frac{2r}{\epsilon} \stackrel{\text{(d)}}{\leq} \pars{D_{\bp,\bq,\cX} + \frac{1}{\dimRFF}\sumnolim_{j=1}^{\dimRFF}\bbE_{\bw_j \sim \Lambda}\bracks{\abs{\bw_j^{\bp + \bq}} \norm{\bw_j}_2}}\frac{2r}{\epsilon}
        \\
        &\leq
        \pars{D_{\bp,\bq, \cX} + E_{\bp, \bq}} \frac{2r}{\epsilon},
        \label{eq:lipschitz_bound}
    \end{align}
    where (c) is Markov's inequality, (d) is \eqref{eq:grad_error} and \eqref{eq:grad_rff}.
    To deal with the centers in \eqref{eq:error_decomp}, note $\partial^{\bp, \bq} \rffkernel(\bz)$ can be written as a sample average of $\dimRFF$ $\iid$ terms as per \eqref{eq:rffkernel_def} since
    \begin{align}
        \partial^{\bp, \bq} \rffkernel(\bz) = \partial^{\bp, \bq} \pars{\frac{1}{\dimRFF} \sum_{i=1}^{\dimRFF} \cos(\bw_j^\top \bz)} =  \frac{1}{\dimRFF} \sum_{i=1}^{\dimRFF} \partial^{\bp, \bq} \cos(\bw_j^\top \bz)
    \end{align}
    so that the Bernstein inequality (Cor.~\ref{thm:bernstein_twotail}) is applicable. For $j = 1, \dots, \dimRFF$, we have
    \begin{align}
    \bbE_{\bw_j \sim \Lambda}{\abs{\partial^{\bp, \bq} \cos\pars{\bw_j^\top \bz}}^m}
    =
    \bbE_{\bw_j \sim \Lambda}{\abs{\bw_j^{\bp + \bq} \cos^{(\abs{\bp + \bq})}(\bw_j^\top \bz)}^m}
    \leq
    \bbE_{\bw_j \sim \Lambda}{\abs{\bw_j^{\bp + \bq}}^m}
    \leq
    \frac{k!S^2R^{k-2}}{2},
    \end{align}
    and that $f(\bz) = \partial^{\bp, \bq} \rffkernel(\bz) - \partial^{\bp, \bq} \kernel(\bz) = \partial^{\bp, \bq} \rffkernel(\bz) - \expe{\partial^{\bp, \bq} \rffkernel(\bz)}$ by the dominated convergence theorem.
    Hence, we may call the Bernstein inequality (Cor.~\ref{thm:bernstein_twotail}) to control $f(\bz_i)$ so that
    \begin{align} \label{eq:center_bound}
        \prob{\abs{f(\bz_i)} \geq \epsilon / 2}
        =
        \prob{\abs{\partial^{\bp, \bq} \rffkernel(\bz) - \partial^{\bp, \bq} \kernel(\bz)} \geq \epsilon / 2}
        \leq
        2 \exp\pars{\frac{-\dimRFF\epsilon^2}{4(2S^2 + R\epsilon)}}.
    \end{align}
    Combining the bounds for $\abs{f(\bz_i)}$ \eqref{eq:center_bound}, and for $L_f$ \eqref{eq:lipschitz_bound}, into \eqref{eq:error_decomp} yields
    \begin{align}
        \prob{\sup_{\bz \in \cM_\Delta} \abs{f(\bz)} \geq \epsilon} \leq 2 \pars{\frac{4 \abs{\cX}}{r}}^d\exp\pars{\frac{-\dimRFF \epsilon^2}{4(2S^2 + R\epsilon)}} + \pars{D_{\bp, \bq, \cX} + E_{\bp, \bq}} \frac{2r}{\epsilon},
    \end{align}
    which has the form $g(r) = \tau_1 r^{-d} + \tau_2 r$, that is minimized by choosing $r^\star = (d\tau_1 / \tau_2)^{\frac{1}{d+1}}$. This choice sets it to the form $g(r^\star) = \tau_1^{\frac{1}{d+1}}\tau_2^{\frac{d}{d+1}} \pars{d^{\frac{1}{d+1}} + d^{\frac{-d}{d+1}}}$, so that by substituting back in
    \begin{align}
        \prob{\sup_{\bz \in \cM_\Delta} \abs{f(\bz)} \geq \epsilon}
        &\leq
        F_d 2^\frac{3d+1}{d+1} \pars{\frac{\abs{\cX}(D_{\bp, \bq, \cX} + E_{\bp, \bq})}{\epsilon}}^\frac{d}{d+1} \exp\pars{\frac{-\dimRFF \epsilon^2}{4(d+1)(2S^2 + R\epsilon)}}. \label{eq:rff_proof_final_bound}
    \end{align}
    Finally, we note that for $d \geq1$, $F_d \coloneqq d^\frac{1}{d+1} + d^\frac{-d}{d+1} \leq 2$ and $2^\frac{3d+1}{d+1} \leq 8$.
\end{proof}

\section{Bounds on Signature Kernels} \label{apx:sig_bounds}
We first set the ground for proving our main theorems by introducing the notion of $L$-Lipschitz kernels to help control distance distortions in the feature space. This subclass of kernels will be useful for us in relating the $1$-variation of sequences in feature space to that in the input space.
After this, we will prove various smaller lemmas and supplementary results for signature kernels, which will lead up to Lemma \ref{lem:RFSF_approx}, which is our main tool for proving Theorem \ref{thm:main}, and it relates the concentration of our \RFSF{} kernel $\rffsigkernel$ to the second derivatives of the \RFF{} kernel $\rffkernel$. Then, the proof of Theorem \ref{thm:main} quantifying the concentration of the \RFSF{} kernel, $\rffsigkernel$, will follow from putting together Lemma \ref{lem:RFSF_approx} with Theorem \ref{thm:rff_derivative_approx}, and we will also make use of the Bernstein inequality from Theorem \ref{thm:bernstein_onetail}. The proof of Theorem \ref{thm:main2} for the RFSF-DP kernel, $\rffsigkernelDP$, will follow by combining the results of this section with $\alpha$-exponential concentration, in particular, Theorem \ref{thm:alpha_subexp_concentration}. Finally, Theorem \ref{thm:main3} for the RFSF-TRP kernel, $\rffsigkernelTRP$, will be proven using lemmas from this section, and the hypercontractivity concentration result from Theorem \ref{thm:hyper_concentration}.

\paragraph{Distance bounds in the RKHS}
\begin{definition}[Lipschitz kernel]\label{def:lipschitz kernel} Let $(\cX, d)$ be a metric space. We call a kernel $\kernel: \cX \times \cX \to \bbR$ with RKHS $\Hil$ an $L$-Lipschitz kernel over $\cX$ for some $L > 0$ if it holds for all $\bx, \by \in \cX$ that
    \begin{align}
        \norm{\kernel_\bx - \kernel_\by}_{\Hil} = \sqrt{\kernel(\bx, \bx) + \kernel(\by, \by) - 2 \kernel(\bx, \by)}  \leq L d(\bx, \by),
    \end{align}
    where $\kernel_\bx\coloneqq \kernel(\bx,\cdot), \kernel_\by\coloneqq \kernel(\by,\cdot) \in \Hil$.
\end{definition}

\begin{example}[Finite 2\textsuperscript{nd} spectral moment implies Lipschitz]\label{example:2ndmoment_lip}
    Let $\kernel: \bbR^d \times \bbR^d \to \bbR$ be a continuous, bounded and translation-invariant kernel with RKHS $\Hil$ and spectral measure $\Lambda$, such that $\sigma_\Lambda^2 \coloneqq \bbE_{\bw \sim \Lambda}\bracks{\norm{\bw}^2_2} < \infty$. Then, it holds that $\kernel$ is $\norm{\bbE_{\bw \sim \Lambda}\bracks{\bw \bw^\top}}_2^{1/2}$-Lipschitz, so for any $\b x,\b y\in\R^d$ one has
    \begin{align}
        \norm{\kernel_\bx - \kernel_\by}_\Hil \leq \norm{\bbE_{\bw \sim \Lambda}\bracks{\bw\bw^\top}}_2^{1/2}\norm{\bx - \by}_2.
    \end{align}
\end{example}
\begin{proof}
    We have for $\bx, \by \in \bbR^d$ that
    \begin{align}
        \norm{\kernel_\bx - \kernel_\by}_\Hil^2 &\stackrel{\text{(a)}}{=} \kernel(\bx, \bx) + \kernel(\by, \by) - 2 \kernel(\bx, \by) \stackrel{\text{(b)}}{=} \int_{\bbR^d} 2 - 2\exp\pars{i \bw^\top(\bx - \by)} \d \Lambda(\bw)
        \\
        &\stackrel{\text{(c)}}{=} \int_{\bbR^d} 2 - 2\cos\pars{\bw^\top(\bx - \by)} \d \Lambda(\bw)
        \stackrel{\text{(d)}}{\leq} \int_{\bbR^d} \pars{\bw^\top(\bx - \by)}^2 \d \Lambda(\bw)
        \\
        &\stackrel{\text{(e)}}{=} (\bx - \by)^\top \bbE_{\bw \sim \Lambda}\bracks{\bw \bw^\top} (\bx - \by) \stackrel{\text{(f)}}{\leq} \norm{\bbE_{\bw \sim \Lambda}\bracks{\bw \bw^\top}}_2 \norm{\bx - \by}_2^2,
    \end{align}
    where (a) holds due to the reproducing property, (b) is Bochner's theorem, (c) is because the imaginary part of the integral evaluates to $0$ as the kernel is real-valued, (d) is due to the inequality $1 - t^2/2 \leq \cos(t)$ for all $t \in \bbR$, (e) is because $\pars{\bw^\top(\bx-\by)}^2 = (\bx-\by)^\top \bw \bw^\top (\bx - \by)$, and (f) is Cauchy-Schwarz inequality combined with the definition of the spectral norm.
\end{proof}

\begin{example}[Random Lipschitz bound for \RFF{}] \label{example:rff_lip}
    Let $\rffkernel:\bbR^d \times \bbR^d \to \bbR$ be an \RFF{} kernel defined as in \eqref{eq:rff_def} corresponding to some spectral measure $\Lambda$, and let $\HilRFF$ denote its feature space corresponding to the \RFF{} map $\rff: \bbR^d \to \HilRFF$ defined as in \eqref{eq:rff_def}, so that given $\bw_1, \dots, \bw_{\dimRFF} \stackrel{\iid}{\sim} \Lambda$, we have for $\bx, \by \in \bbR^d$ that
    \begin{align}
        \rffkernel(\bx, \by) = \frac{1}{\dimRFF} \sum_{j=1}^{\dimRFF} \cos\pars{\bw_j^\top(\bx - \by)}.
    \end{align}
    Let $\bW = (\bw_1, \dots, \bw_{\dimRFF}) \in \bbR^{d \times \dimRFF}$ be the random matrix with column vectors $\bw_1, \dots, \bw_{\dimRFF} \stackrel{\iid}{\sim} \Lambda$. Then, $\rffkernel$ is $\pars{\frac{\norm{\bW}_2}{\sqrt{\dimRFF}}}$-Lipschitz, so that for any $\bx, \by \in \bbR^d$, we have the inequality
    \begin{align}
        \norm{\rff(\bx) - \rff(\by)}_{\HilRFF} \leq \frac{\norm{\bW}_2}{\sqrt{\dimRFF}} \norm{\bx - \by}_2.
    \end{align} 
\end{example}
\begin{proof}
    The proof follows analogously to that of Example \ref{example:2ndmoment_lip}. Let $\bx, \by \in \bbR^d$, then
    \begin{align}
        \norm{\rff(\bx) - \rff(\by)}_{\HilRFF}^2 &\stackrel{\text{(a)}}{=} 2 - \frac{2}{\dimRFF} \sum_{i=1}^{\dimRFF} \cos(\bw_i^\top (\bx-\by)) \stackrel{\text{(b)}}{\leq} \frac{1}{\dimRFF} \sum_{i=1}^{\dimRFF} \pars{\bw_i^\top (\bx - \by)}^2 \\
        &\stackrel{\text{(c)}}{=} \frac{1}{\dimRFF} (\bx - \by)^\top \bW \bW^\top (\bx - \by) \stackrel{\text{(d)}}{\leq} \frac{1}{\dimRFF} \norm{\bW}^2_2 \norm{\bx - \by}^2_2,
    \end{align}
    where (a) is due to the cosine identity $\cos(a-b) = \cos(a)\cos(b) + \sin(a)\sin(b)$ for all $a,b \in \bbR$, (b) is due to the inequality $1 - t^2/2 \leq \cos(t)$ for all $t \in \bbR$, (c) is because $\pars{\bw_i^\top(\bx-\by)}^2 = (\bx-\by)^\top \bw_i \bw_i^\top (\bx - \by)$ and $\bW \bW^\top = \sum_{i=1}^{\dimRFF} \bw_i \bw_i^\top$, (d) is due to the Cauchy-Schwarz inequality combined with the definition of the spectral norm.
\end{proof}
\paragraph{Bounds for the Signature Kernel}
A well-known property of signature features that they decay factorially fast with respect to the tensor level $m \in \bbN$.
\begin{lemma}[Norm bound for signature features] \label{lem:1var_sig_norm}
    Let $L > 0$ and $\kernel: \cX \times \cX \to \bbR$ be an $L$-Lipschitz kernel with RKHS $\Hil$. Then, we have for the level-$m$ signature features $\signature[m](\bx)$ of the sequence $\bx  \in \seq$ that
    \begin{align}
        \norm{\signature[m](\bx)}_{\Hil^{\otimes m}}
        \leq
        \frac{\pars{L \norm{\bx}_{\onevar}}^m}{m!}.
    \end{align}
\end{lemma}
\begin{proof}
    We have
    \begin{align}
        \norm{\signature[m](\bx)}_{\Hil^{\otimes m}} 
        &=
        \norm{\sum_{\bi \in \Delta_m(\ell_\bx-1)} \delta \kernel_{\bx_{i_1}} \otimes \cdots \otimes \delta\kernel_{\bx_{i_m}}}_{\cH^{\otimes m}}
        \\
        &\stackrel{\text{(a)}}{\leq}
        \sum_{\bi \in \Delta_m(\ell_\bx-1)} \norm{\delta\kernel_{\bx_{i_1}}}_\cH \cdots \norm{\delta \kernel_{\bx_{i_m}}}_\cH
        \\
        &\stackrel{\text{(b)}}{\leq} L^m \sum_{\bi \in \Delta_m(\abs{\bx-1})} \norm{\delta\bx_{i_1}}_2 \cdots \norm{\delta \bx_{i_m}}_2
        \\
        &\stackrel{\text{(c)}}{\leq}
        \frac{\pars{L \sum_{i=1}^{\ell_\bx-1} \norm{\delta \bx_i}_2}^m}{m!}
        \stackrel{\text{(d)}}{=} \frac{\pars{L \norm{\bx}_\onevar}^m}{m!}, \label{line:signorm2}
    \end{align}
    where (a) follows from triangle inequality followed by factorizing the tensor norm, (b) is the $L$-Lipschitzness property, (c) from completing the multinomial expansion and normalizing by the number of permutations (note that $\Delta_m(\ell_\bx-1)$ contains a single permutation of all such multi-indices that have nonrepeating entries), and (d) is the definition of sequence $1$-variation.
\end{proof}

The following bound for the signature kernel is a direct consequence of the previous lemma. 
\begin{corollary}[Upper bound for signature kernel] \label{lem:ksig_bound}
Let $\kernel: \cX \times \cX \to \bbR$ be an $L$-Lipschitz kernel, and $\sigkernel[m]: \seq \times \seq \to \bbR$ the level-$m$ signature kernel built from $\kernel$. Then, we have the following bound for $\bx, \by \in \seq$
    \begin{align}
        \abs{\sigkernel[m](\bx, \by)}
        \leq \frac{\pars{L^2 \norm{\bx}_\onevar \norm{\by}_\onevar}^m}{(m!)^2}.
    \end{align}
\end{corollary}
\begin{proof}
    Note that without a kernel trick, $\sigkernel[m]$ is written for $\bx, \by \in \seq$ as the inner product
    \begin{align}
        \abs{\sigkernel[m](\bx, \by)}
        &=
        \abs{\inner{\sum_{\bi \in \Delta_m(\ell_\bx-1)} \delta \kernel_{\bx_{i_1}} \otimes \cdots \otimes \delta \kernel_{\bx_{i_m}}}{\sum_{\bj \in \Delta_m(\ell_\by-1)} \delta \kernel_{\by_{j_1}} \otimes \cdots \otimes \delta \kernel_{\by_{j_m}}}_{\Hil^{\otimes m}}}
        \\
        &\stackrel{\text{(a)}}{\leq} \norm{\sum_{\bi \in \Delta_m(\ell_\bx-1)} \delta \kernel_{\bx_{i_1}} \otimes \cdots \otimes \delta \kernel_{\bx_{i_m}}}_{\Hil^{\otimes m}}\norm{\sum_{\bj \in \Delta_m(\ell_\by-1)} \delta \kernel_{\by_{j_1}} \otimes \cdots \otimes \delta \kernel_{\by_{j_m}}}_{\Hil^{\otimes m}}
        \\
        &\stackrel{\text{(b)}}{\leq} \frac{\pars{L^2 \norm{\bx}_\onevar \norm{\by}_\onevar}^m}{(m!)^2},
    \end{align}
    where (a) follows from the Cauchy-Schwarz inequality, and (b) is implied by Lemma \ref{lem:1var_sig_norm}.
\end{proof}

A similar upper bound to Lemma \ref{lem:1var_sig_norm} also holds for the \RFSF{} kernel $\rffsigkernel[m]$, that now depends on the norms of the random matrices $\bW^{(1)}, \dots, \bW^{(m)}$, hence is itself random.

\begin{lemma}[Random norm bound for \RFSF{}] \label{lem:1var_rffsig_norm}
    Let $\rffsig[m]: \seq \to \HilRFFT$ be the level-$m$ \RFSF{} map defined as in \eqref{eq:rffsigdef} built from some spectral measure $\Lambda$. Then, we have for $\bx \in \seq$ that
    \begin{align}
        \norm{\rffsig[m](\bx)}_{{\HilRFF}^{\otimes m}} \leq \frac{\norm{\bW^{(1)}}_2 \cdots \norm{\bW^{(m)}}_2}{m!} \pars{\frac{\norm{\bx}_\onevar}{\sqrt{\dimRFF}}}^m,
    \end{align}
    where $\bW^{(1)}, \dots, \bW^{(m)} \stackrel{\iid}{\sim} \Lambda^{\dimRFF}$ are random matrices sampled from $\Lambda^{\dimRFF}$.
\end{lemma}
\begin{proof}
    The proof follows analogously to Lemma \ref{lem:1var_sig_norm}. We have for $\bx \in \seq$ that
    \begin{align}
        \norm{\rffsig[m](\bx)}_{{\HilRFF}^{\otimes m}}
        &=
        \norm{\sum_{\bi \in \Delta_m(\ell_\bx - 1)} \delta \rff_1(\bx_{i_1}) \otimes \cdots \otimes \delta \rff_m(\bx_{i_m})}_{\HilRFF^{\otimes m}}
        \\
        &\stackrel{\text{(a)}}{\leq}
        \sum_{\bi \in \Delta_m(\ell_\bx - 1)} \norm{\delta \rff_1(\bx_{i_1})}_{\HilRFF} \cdots \norm{\delta \rff_m(\bx_{i_m})}_{\HilRFF}
        \\
        &\stackrel{\text{(b)}}{\leq}
        \frac{\norm{\bW^{(1)}}_2 \cdots \norm{\bW^{(m)}}_2}{\dimRFF^{m/2}} \sum_{\bi \in \Delta_m(\ell_\bx - 1)} \norm{\delta \bx_{i_1}}_2 \cdots \norm{\delta \bx_{i_m}}_2
        \\
        &\stackrel{\text{(c)}}{\leq}
        \frac{\norm{\bW^{(1)}}_2 \cdots \norm{\bW^{(m)}}_2}{\dimRFF^{m/2}} \frac{\norm{\bx}_{\onevar}^m}{m!},
    \end{align}
    where (a) is the triangle inequality and factorization of tensor norm, (b) is using the Lipschitzness of \RFF s from Example \ref{example:rff_lip}, (c) is the same as steps (c)-(d) in Lemma \ref{lem:1var_sig_norm}.
\end{proof}
Then, the following is again an application of the Cauchy-Schwarz inequality.
\begin{corollary}[Random upper bound for \RFSF{} kernel]
\label{lem:krffsig_bound}
    Let $\rffsigkernel[m]: \seq \times \seq \to \bbR$ be the level-$m$ \RFSF{} kernel defined as in \eqref{eq:rffsigkernel_def} built from some spectral measure $\Lambda$.  Then, for all $\bx, \by \in \seq$
    \begin{align}
        \abs{\rffsigkernel[m](\bx, \by)}
        &\leq
        \frac{\norm{\bW^{(1)}}_2^2 \cdots \norm{\bW^{(m)}}_2^2}{(m!)^2}  \pars{\frac{\norm{\bx}_\onevar \norm{\by}_\onevar}{\dimRFF}}^m
    \end{align}
    where $\bW^{(1)}, \dots, \bW^{(m)} \stackrel{\iid}{\sim} \Lambda^{\dimRFF}$ are random matrices sampled from $\Lambda^{\dimRFF}$.
\end{corollary}

The following lemma does not concern signatures, but will be useful to us later in the proof of Lemma \ref{lem:RFSF_approx} by providing a mean-value theorem for the cross-differencing operator $\delta^2_{i,j}$.
\begin{lemma}[2\textsuperscript{nd} order mean-value theorem] \label{lem:second_mvt}
Let $f: \bbR^d \times \bbR^d \to \bbR$ be a twice differentiable function, and $\cX \subset \bbR^d$ be a convex and compact set. Then, we have for any $\bu, \bv, \bx, \by \in \cX$ that
\begin{align}
    f(\bx, \by) - f(\bx, \bv) - f(\bu, \by) + f(\bu, \bv) \leq \sup_{\bs, \bt \in \cX} \norm{\partial^2_{\bs, \bt} f(\bs, \bt)}_2 \norm{\bx - \bu}_2 \norm{\by - \bv}_2,
\end{align}
where $\partial^2_{\bs, \bt} f(\bs, \bt) \coloneqq \pars{\frac{\partial^2 f(\bs, \bt)}{\partial s_i \partial t_j}}_{i,j=1}^d$ refers to the submatrix of the Hessian of cross-derivatives.
\begin{proof}
    Keeping $\bv, \by \in \cX$ as fixed, we may define $g: \bbR^d \to \bbR$ as $g(\cdot) \coloneqq f(\cdot, \by) - f(\cdot, \bv)$, so that the expression above can be written as $g(\bx) - g(\bu)$, and by the convexity of $\cX$, we may apply the mean-value theorem to find that $\exists s \in (0, 1)$ such that
    \begin{align}
        g(\bx) - g(\bu) = \inner{\nabla g(s \bx + (1-s) \bu)}{ \bx - \bu} \leq \sup_{\bs \in \cX}\norm{\nabla g(\bs)}_2 \norm{\bx - \bu}_2, \label{eq:mvt1}
    \end{align}
    where $\nabla g(\bs) \coloneqq \pars{\frac{\partial g(\bs)}{\partial s_i}}_{i=1}^d$ denotes the gradient of $g$, while the second inequality follows from the Cauchy-Schwarz inequality and the compactness of $\cX$ (so that the $\sup$ exists). Also note that $\nabla g(\bs) = \partial_\bs f(\bs, \by) - \partial_\bs f(\bs, \bv)$, so defining $h: \bbR^d \to \bbR^d$ as $h(\cdot) \coloneqq \partial_\bs f (\bs, \cdot)$ and applying the vector-valued mean-value inequality \cite[Thm.~9.19]{rudin1976principles} to $h$ gives that $\exists t \in (0, 1)$ such that
    \begin{align}
        \norm{h(\by) - h(\bv)}_2 \leq \norm{\Jac_h(t\by + (1-t)\bv)}_2 \norm{\by - \bv}_2 \leq \sup_{\bt \in \cX} \norm{\Jac_h(\bt)}_2 \norm{\by - \bv}_2, \label{eq:mvt2}
    \end{align}
    where $\Jac_h(\bt) \coloneqq \pars{\frac{\partial h_i (\bt)}{\partial t_j}}_{i,j=1}^d$ refers to the Jacobian of $h$. Putting the inequalities \eqref{eq:mvt1} and \eqref{eq:mvt2} together and substituting back the function $f$ gives the desired result.
\end{proof}
\end{lemma}

This is our final lemma in our exposition of supplementary results about the (random) signature kernels, and it will be our main tool for proving Theorem \ref{thm:main}.
\begin{lemma}[Uniform upper bound for deviation of \RFSF{} kernel] \label{lem:RFSF_approx}
    Let $\cX \subset \bbR^d$ be a convex and compact set, $\kernel: \bbR^d \times \bbR^d \to \bbR$ a continuous, bounded, translation-invariant $L$-Lipschitz kernel and $\rffkernel: \bbR^d \times \bbR^d \to \bbR$ the corresponding \RFF{} kernel.
    Then, the level-$m$ ($m\in \bbN$) signature and \RFSF{} kernels are uniformly close for $V>0$ by
    \begin{align}
        &\sup_{\substack{\bx, \by \in \seq \\ \norm{\bx}_\onevar, \norm{\by}_\onevar \leq V}} \abs{\rffsigkernel[m](\bx, \by) - \sigkernel[m](\bx, \by)}
        \\
        &\leq V^{2m} \sum_{k=1}^m \frac{L^{2(m-k)}}{\dimRFF^{k-1}((k-1)!)^2} \norm{\bW^{(1)}}_2^2 \cdots \norm{\bW^{(k-1)}}_2^2 \sup_{\bs, \bt \in \cX} \norm{\partial^2_{\bs,\bt} \rffkernel_{k}(\bs, \bt) - \partial^2_{\bs,\bt} \kernel(\bs, \bt)}_2, 
    \end{align}
    where $\rffkernel_{1}, \ldots, \rffkernel_{m}$ are independent \RFF{} kernels with weights $\bW^{(1)}, \dots, \bW^{(m)} \stackrel{\iid}{\sim} \Lambda^{\dimRFF}$, and $\partial^2_{\bs, \bt} f(\bs, \bt) \coloneqq \pars{\frac{\partial^2 f(\bs, \bt)}{\partial s_i \partial t_j}}_{i,j=1}^d$ for a twice-differentiable function $f:\bbR^d \times \bbR^d \to \bbR$.
\end{lemma}
\begin{proof}
First of all, by Lemma \ref{lem:ksig_bound} and Lemma \ref{lem:krffsig_bound}, the supremum exists.
In the following, given a sequence $\bx \in \seq$, we denote its $1:l$ slice for some $l \in [\ell_\bx]$ $\bx_{1:l} \coloneqq (\bx_1, \dots, \bx_l)$.
Then, it holds for any $m \geq 1$ recursively for the signature kernel that
\begin{align} \label{eq:sigkernel_rec}
    \sigkernel[m](\bx, \by) = \sum_{k=1}^{\ell_\bx-1} \sum_{l=1}^{\ell_\by-1} \sigkernel[m-1](\bx_{1:k}, \by_{1:l}) \delta^2_{k, l} \kernel(\bx_k, \by_l),
\end{align}
and analogously for the \RFSF{} kernel that
\begin{align} \label{eq:rffsigkernel_rec}
    \rffsigkernel[m](\bx, \by) = \sum_{k=1}^{\ell_\bx-1} \sum_{l=1}^{\ell_\by-1} \rffsigkernel[m-1](\bx_{1:k}, \by_{1:l}) \delta^2_{k,l} \rffkernel_{m}(\bx_k, \by_l).
\end{align}
Combining these recursions together, we have for the uniform error that
\begin{align}
    \epsilon_m
    =& \sup_{\substack{\bx, \by \in \seq\\ \norm{\bx}_\onevar, \norm{\by}_\onevar \leq V}} \abs{\rffsigkernel[m](\bx, \by) - \sigkernel[m](\bx, \by)}
    \\
    =& \sup_{\substack{\bx, \by \in \seq\\ \norm{\bx}_\onevar, \norm{\by}_\onevar \leq V}} \abs{\sum_{k=1}^{\ell_\bx-1} \sum_{l=1}^{\ell_\by-1} \rffsigkernel[m-1](\bx_{1:k}, \by_{1:l}) \delta^2_{k,l}\rffkernel_{m}(\bx_k, \by_l) - \sigkernel[m-1](\bx_{1:k}, \by_{1:l}) \delta^2_{k,l}\kernel(\bx_k, \by_l)}
    \\
    \stackrel{\text{(a)}}{\leq}& \sup_{\substack{\bx, \by \in \seq\\ \norm{\bx}_\onevar, \norm{\by}_\onevar \leq V}} \abs{\sum_{k=1}^{\ell_\bx-1} \sum_{l=1}^{\ell_\by-1} \rffsigkernel[m-1](\bx_{1:k}, \by_{1:l}) (\delta^2_{k,l}\rffkernel_{m}(\bx_k, \by_l) - \delta^2_{k,l}\kernel(\bx_k, \by_l))}
    \\
    &+ \sup_{\substack{\bx, \by \in \seq\\ \norm{\bx}_\onevar, \norm{\by}_\onevar \leq V}} \abs{\sum_{k=1}^{\ell_\bx-1} \sum_{l=1}^{\ell_\by-1} (\rffsigkernel[m-1](\bx_{1:k}, \by_{1:l}) - \rffsigkernel[m-1](\bx_{1:k}, \by_{1:l})) \delta^2_{k,l}\kernel(\bx_k, \by_l)}
    \\
    \stackrel{\text{(b)}}{\leq}& \sup_{\substack{\bx, \by \in \seq\\ \norm{\bx}_\onevar, \norm{\by}_\onevar \leq V}} \sum_{k=1}^{\ell_\bx-1} \sum_{l=1}^{\ell_\by-1} \abs{\rffsigkernel[m-1](\bx_{1:k}, \by_{1:l})}\abs{\delta^2_{k,l}\rffkernel_{m}(\bx_k, \by_l) - \delta^2_{k,l}\kernel(\bx_k, \by_l)}
    \\
    &+ \sup_{\substack{\bx, \by \in \seq\\ \norm{\bx}_\onevar, \norm{\by}_\onevar \leq V}} \sum_{k=1}^{\ell_\bx-1} \sum_{l=1}^{\ell_\by-1} \abs{\rffsigkernel[m-1](\bx_{1:k}, \by_{1:l}) - \sigkernel[m-1](\bx_{1:k}, \by_{1:l})}\abs{\delta^2_{k,l}\kernel(\bx_k, \by_l)}
    \\
    \stackrel{\text{(c)}}{\leq}& \underbrace{\sup_{\substack{\bx, \by \in \seq\\ \norm{\bx}_\onevar, \norm{\by}_\onevar \leq V}} \abs{\rffsigkernel[m-1](\bx, \by)}}_{\text{(i)}}  \underbrace{\sup_{\substack{\bx, \by \in \seq\\ \norm{\bx}_\onevar, \norm{\by}_\onevar \leq V}} \sum_{k=1}^{\ell_\bx-1} \sum_{l=1}^{\ell_\by-1} \abs{\delta^2_{k,l}\rffkernel_{m}(\bx_k, \by_l) - \delta^2_{k,l}\kernel(\bx_k, \by_l)}}_{\text{(ii)}}
    \\
    &+ \underbrace{\sup_{\substack{\bx, \by \in \seq\\ \norm{\bx}_\onevar, \norm{\by}_\onevar \leq V}} \abs{\rffsigkernel[m-1](\bx, \by) - \sigkernel[m-1](\bx, \by)}}_{\text{(iii)}} \underbrace{\sup_{\substack{\bx, \by \in \seq\\ \norm{\bx}_\onevar, \norm{\by}_\onevar \leq V}} \sum_{k=1}^{\ell_\bx-1} \sum_{l=1}^{\ell_\by-1} \abs{\delta^2_{k,l}\kernel(\bx_k, \by_l)}}_{\text{(iv)}}, \label{eq:term_splitting}
\end{align}
where (a) follows from adding and subtracting the cross-terms and applying triangle inequality, (b) follows from applying triangle inequality over the summations, (c) follows from noting that if $\norm{\bx}_\onevar, \norm{\by}_\onevar \leq V$ then so is $\norm{\bx_{1:k}}_\onevar, \norm{\by_{1:l}}_\onevar \leq V$ for $k \in [\ell_\bx]$ and $l \in [\ell_\by]$, and thus justifiably pulling out the supremums.

Now, we deal with terms (i)--(iv) individually. For (i), we have Corollary \ref{lem:krffsig_bound}, so
\begin{align}
    \sup_{\substack{\bx, \by \in \seq\\ \norm{\bx}_\onevar, \norm{\by}_\onevar \leq V}} \abs{\rffsigkernel[m-1](\bx, \by)} \leq \frac{\norm{\bW^{(1)}}_2^2 \cdots \norm{\bW^{(m-1)}}_2^2}{((m-1)!)^2} \pars{\frac{V^2}{\dimRFF}}^{m-1}. \label{eq:term1}
\end{align}

To deal with (ii), we can apply Lemma \ref{lem:second_mvt} with $f = \rffkernel_{m} - \kernel$ to get
\begin{align}
    &\sup_{\substack{\bx, \by \in \seq\\ \norm{\bx}_\onevar, \norm{\by}_\onevar \leq V}} \sum_{k=1}^{\ell_\bx-1} \sum_{l=1}^{\ell_\by-1} \abs{\delta^2_{k,l} \rffkernel_{m}(\bx_k, \by_l) - \delta^2_{k,l} \kernel(\bx_k, \by_l)}
    \\
    &\leq \sup_{\bs, \bt \in \cX} \norm{\partial^2_{\bs,\bt} \rffkernel_{m}(\bs, \bt) - \partial^2_{\bs, \bt} \kernel(\bs, \bt)}_2 \sup_{\substack{\bx, \by \in \seq\\ \norm{\bx}_\onevar, \norm{\by}_\onevar \leq V}} \sum_{k=1}^{\ell_\bx-1} \sum_{l=1}^{\ell_\by-1} \norm{\delta \bx_k}_2 \norm{\delta \by_l}_2
    \\
    &\leq V^2 \sup_{\bs, \bt \in \cX} \norm{\partial^2_{\bs, \bt} \rffkernel_{m}(\bs, \bt) - \partial^2_{\bs,\bt} \kernel(\bs, \bt)}_2. \label{eq:term2}
\end{align}
For (iii), we note that it is simply $\epsilon_{m-1}$.
Finally, we can write (iv) as an inner product and apply Cauchy-Schwarz and $L$-Lipschitzness of $\kernel$ so that
\begin{align}
    &\sup_{\substack{\bx, \by \in \seq\\ \norm{\bx}_\onevar, \norm{\by}_\onevar \leq V}} \sum_{k=1}^{\ell_\bx-1} \sum_{l=1}^{\ell_\by-1} \abs{\delta^2_{k,l} \kernel(\bx_k, \by_l)}
    = \sup_{\substack{\bx, \by \in \seq\\ \norm{\bx}_\onevar, \norm{\by}_\onevar \leq V}} \sum_{k=1}^{\ell_\bx-1} \sum_{l=1}^{\ell_\by-1} \abs{\inner{\delta \kernel_{\bx_k}} {\delta \kernel_{\by_l}}}
    \\
    &\leq \sup_{\substack{\bx, \by \in \seq\\ \norm{\bx}_\onevar, \norm{\by}_\onevar \leq V}} \sum_{k=1}^{\ell_\bx-1} \sum_{l=1}^{\ell_\by-1} \norm{\delta \kernel_{\bx_k}}_{\Hil}\norm{\delta \kernel_{\by_l}}_{\Hil} \leq L^2 V^2. \label{eq:term3}
\end{align}
Putting equations \eqref{eq:term1}, \eqref{eq:term2}, \eqref{eq:term3} together in \eqref{eq:term_splitting}, we get that
\begin{align}
    \epsilon_m
    =& \sup_{\substack{\bx, \by \in \seq\\ \norm{\bx}_\onevar, \norm{\by}_\onevar \leq V}} \abs{\rffsigkernel[m](\bx, \by) - \sigkernel[m](\bx, \by)}
    \\
    \leq& \frac{V^{2m}}{\dimRFF^{m-1} ((m-1)!)^2} \sup_{\bs, \bt \in \cX} \norm{\partial^2_{\bs,\bt} \rffkernel_{m}(\bs, \bt) - \partial^2_{\bs,\bt} \kernel(\bs, \bt)}_2 \prod_{p=1}^{m-1} \norm{\bW^{(p)}}_2^2 + L^2 V^2 \epsilon_{m-1}, \label{eq:terms_rec}
\end{align}
which gives us a recursion for estimating $\epsilon_m$. The initial step, $m=1$, can be estimated by
\begin{align}
    \epsilon_1
    &=
    \sup_{\substack{\bx, \by \in \seq\\ \norm{\bx}_\onevar, \norm{\by}_\onevar \leq V}} \abs{\rffsigkernel[1](\bx, \by) - \sigkernel[1](\bx, \by)}
    \\
    &=
    \sup_{\substack{\bx, \by \in \seq\\ \norm{\bx}_\onevar, \norm{\by}_\onevar \leq V}} \abs{\sum_{k=1}^{\ell_\bx-1} \sum_{l=1}^{\ell_\by-1} \delta^2_{k, l} \rffkernel_{1}(\bx_k, \by_l) - \delta^2_{k,l} \kernel(\bx_k, \by_l)}
    \\
    &=
    \sup_{\substack{\bx, \by \in \seq\\ \norm{\bx}_\onevar, \norm{\by}_\onevar \leq V}} \sum_{k=1}^{\ell_\bx-1} \sum_{l=0}^{\ell_\by-1} \abs{\delta^2_{k,l} \rffkernel_{1}(\bx_k, \by_l) - \delta^2_{k,l} \kernel(\bx_k, \by_l)}
    \\
    &\leq V^2 \sup_{\bs, \bt \in \cX} \norm{\partial^2_{\bs,\bt} \rffkernel_{1}(\bs, \bt) - \partial^2_{\bs, \bt} \kernel(\bs, \bt)}_2, \label{eq:terms_init}
\end{align}
which is actually analogous to \eqref{eq:terms_rec} since $\epsilon_0 = 0$. Now, we may unroll the recursion \eqref{eq:terms_rec} with the initial condition \eqref{eq:terms_init}, and we get
\begin{align}
    \epsilon_m \leq V^{2m} \sum_{k=1}^m \frac{L^{2(m-k)}}{\dimRFF^{k-1} ((k-1)!)^2} \sup_{\bs, \bt \in \cX} \norm{\partial^2_{\bs,\bt} \rffkernel_{k}(\bs, \bt) - \partial^2_{\bs,\bt} \kernel(\bs, \bt)}_2 \prod_{p=1}^{k-1} \norm{\bW^{(p)}}_2^2.
\end{align}
\end{proof}
\paragraph{Proofs of main concentration results} Here, we provide proofs of the main concentration results, i.e.~Theorems \ref{thm:main}, \ref{thm:main2}, \ref{thm:main3}, respectively under Theorems \ref{thm:rfsf_approx}, \ref{thm:rfsf_dp_approx}, \ref{thm:rfsf_trp_approx}.
\begin{theorem}[Concentration inequality for \RFSF{} kernel] \label{thm:rfsf_approx}
    Let $\cX \subset \bbR^d$ be a compact and convex set with diameter $\abs{\cX}$, and $\cX_\Delta \coloneqq \{\bx - \by : \bx, \by \in \cX \}$. Let $\kernel: \bbR^d \times \bbR^d \to \bbR$ be a continuous, bounded, translation-invariant kernel with spectral measure $\Lambda$, which satisfies for some $S, R > 0$ that
    \begin{align} \label{eq:rfsf_approx_cond}
        \bbE_{\bw \sim \Lambda}\bracks{\abs{w_i}^{2k}} \leq \frac{k! S^2 R^{k-2}}{2} \quad \text{for all} \spc i \in [d] \spc \text{and} \spc k \geq 2.
    \end{align}

    Then, the following quantities are finite: $\sigma_\Lambda^2 \coloneqq \bbE_{\bw \sim \Lambda}\bracks{\norm{\bw}_2^2}$, $L \coloneqq \norm{\bbE_{\bw \sim \Lambda}\bracks{\bw \bw^\top}}_2^{1/2}$, $E_{i,j} \coloneqq \bbE_{{\bw \sim \Lambda}}\bracks{\abs{w_i w_j} \norm{\bw}_2}$ and $D_{i,j} \coloneqq \sup_{\bz \in \cX_\Delta} \norm{\nabla \bracks{\frac{\partial^2\kernel(\bz)}{\partial z_i \partial z_j}}}_2$ for $i,j \in [d]$. Further, for any maximal sequence $1$-variation $V>0$, and signature level $m \in \mathbb{Z}_+$, it holds for the level-$m$ \RFSF{} kernel $\rffsigkernel[m]: \seq \times \seq \to \bbR$ defined as in \eqref{eq:rffsigkernel_def} and the signature kernel $\sigkernel[m]: \seq \times \seq \to \bbR$ defined as in \eqref{eq:sigkernel_def} for $\epsilon > 0$ that
    \begin{align} \label{eq:rfsf_approx_bound}
        \bbP & \bracks{\sup_{\substack{\bx, \by \in \seq \\ \norm{\bx}_\onevar, \norm{\by}_\onevar \leq V}} \abs{\sigkernel[m](\bx, \by) - \rffsigkernel[m](\bx, \by)} \geq \epsilon} \le
        \\
        &\leq
        m
        \begin{cases}
        \pars{C_{d, \cX} \pars{\frac{\beta_{d, m, V}}{\epsilon}}^\frac{d}{d+1} + d}
        \exp\pars{- \frac{\dimRFF}{2(d+1)(S^2 + R)} \pars{\frac{\epsilon}{\beta_{d, m, V}}}^{2}} \quad
        &\text{for} \spc \epsilon < \beta_{d, m, V}  \\
        \pars{C_{d, \cX} \pars{\frac{\beta_{d, m, V}}{\epsilon}}^{\frac{d}{(d+1)m}} + d}
        \exp\pars{- \frac{\dimRFF}{2(d+1)(S^2 + R)} \pars{\frac{\epsilon}{\beta_{d, m, V}}}^\frac{1}{m}} \quad
        &\text{for} \spc \epsilon \geq \beta_{d, m, V},
        \end{cases}
    \end{align}
    where $C_{d, \cX} \coloneqq 2^\frac{1}{d+1} 16 \abs{\cX}^\frac{d}{d+1}\sum_{i,j=1}^d \pars{D_{i,j} + E_{i, j}}^\frac{d}{d+1}$ and $\beta_{d, m, V} \coloneqq m \pars{2 V^{2} \pars{L^2 \vee 1} \pars{\sigma_\Lambda^2 \vee d}}^m$.
\end{theorem}
\begin{proof}
    \emph{Finite quantities.}
    To start off with, due to \eqref{eq:rfsf_approx_cond} with $m=2$ and Jensen's inequality (Lemma \ref{lem:jensen}), we get
    \begin{align}
        \expe{\abs{w_i}^2} \leq \bbE^{1/2}\bracks{\abs{w_i}^4} < \infty \quad \text{for all} \spc i \in \bracks{d}.
    \end{align}
    Hence, by linearity of the expectation $\sigma_\Lambda^2 = \expe{\norm{\bw}^2} = \sum_{i=1}^d \expe{w_i^2} < \infty$. Next, due to Hölder's inequality (Lemma \ref{lem:holder}), it holds that
    \begin{align}
    \expe{w_i w_j} \leq \expe{\abs{w_i w_j}} \leq \bbE^{1/2}\bracks{\abs{w_i}^2}\bbE^{1/2}\bracks{\abs{w_j}^2} < \infty \quad \text{for all} \spc i,j \in \bracks{d},
    \end{align}
    therefore $L < \infty$. Further, applying Hölder's inequality twice followed by Jensen's inequality,
    \begin{align}
    \expe{\abs{w_i w_j w_k}} &\leq \bbE^{2/3}\bracks{\abs{w_i w_j}^{3/2}} \bbE^{1/3}\bracks{\abs{w_k}^3} \leq \bbE^{1/3}\bracks{\abs{w_i}^3} \bbE^{1/3}\bracks{\abs{w_j}^3} \bbE^{1/3}\bracks{\abs{w_k}^3} \\
    &\leq \bbE^{1/6}\bracks{\abs{w_i}^6} \bbE^{1/6}\bracks{\abs{w_j}^6} \bbE^{1/6}\bracks{\abs{w_k}^6} < \infty \quad \text{for all} \spc i, j, k \in \bracks{d}, \label{eq:ijk_moment}
    \end{align} which is finite due to \eqref{eq:rfsf_approx_cond} with $m=3$. Now, because of the $\ell^1$-$\ell^2$ norm inequality,
    \begin{align}
        \expe{\abs{w_i w_j} \norm{\bw}_2} \leq \expe{\abs{w_i w_j} \norm{\bw}_1} = \sum_{k=1}^d \expe{\abs{w_i w_j w_k}} < \infty,
    \end{align}
    hence $\bar E < \infty$. Next, as per \cite[Thm.~1.2.1.(iii)]{sasvari2013multivariate}, as \eqref{eq:ijk_moment} holds for all $i,j,k \in \bracks{d}$, $\kernel$ is $3$-times continuously differentiable, which combined with the compactness of $\cX$, hence that of $\cX_\Delta$, gives $\sup_{\bz \in \cX_\Delta} \abs{\frac{\partial^3 \kernel(\bz)}{\partial z_i \partial z_j \partial z_k}} < \infty$. Finally, from the $\ell^1$-$\ell^2$ inequality again, we get that
    \begin{align}
    \sup_{\bz \in \cX_\Delta} \norm{\nabla \bracks{\frac{\partial^2 \kernel(\bz)}{\partial z_i \partial z_j}}}_2 \leq \sup_{\bz \in \cX_\Delta} \norm{\nabla \bracks{\frac{\partial^2 \kernel(\bz)}{\partial z_i \partial z_j}}}_1
    \leq \sum_{k=1}^d \sup_{\bz \in \cX_\Delta} \abs{\frac{\partial^3 \kernel(\bz)}{\partial z_i \partial z_j \partial z_k}} < \infty,
    \end{align}
    which shows the finiteness of $\bar D$. This finishes showing that the stated quantities are finite.
    
    \emph{Splitting the bound.}
    To start proving our main inequality, first note that as per Example \ref{example:2ndmoment_lip}, $\kernel$ is $L$-Lipschitz (see Def.~\ref{def:lipschitz kernel}). Hence, Lemma \ref{lem:RFSF_approx} yields that
    \begin{align}
        \sup_{\substack{\bx, \by \in \seq \\ \norm{\bx}_\onevar, \norm{\by}_\onevar \leq V}} &\abs{\rffsigkernel[m](\bx, \by) - \rffsigkernel[m](\bx, \by)}
        \\
        &\leq V^{2m} \sum_{k=1}^m \frac{L^{2(m-k)}}{\dimRFF^{k-1}((k-1)!)^2} \sup_{\bs, \bt \in \cX} \norm{\partial^2_{\bs,\bt} \rffkernel_{k}(\bs, \bt) - \partial^2_{\bs,\bt} \kernel(\bs, \bt)}_2 \prod_{p=1}^{k-1} \norm{\bW^{(p)}}_2^2. \label{eq:rfsf_bound_rhs}
    \end{align}
    We bound the summand in the previous line in probability for each $k \in  \bracks{m}$.
    For brevity, denote $\alpha_{m, k} \coloneqq \frac{V^{2m} L^{2(m-k)}}{((k-1)!)^2}$, and first consider the case $k \geq 2$, so that we have
\begin{align}
    P_k(\epsilon) &\coloneqq \prob{\frac{\alpha_{m, k}}{\dimRFF^{k-1}} \norm{\bW^{(1)}}_2^2 \cdots \norm{\bW^{(k-1)}}_2^2 \sup_{\bs, \bt \in \cX} \norm{\partial_{\bs, \bt} \rffkernel_{k}(\bs, \bt) - \partial_{\bs, \bt} \kernel(\bs, \bt)}_2 \geq \epsilon}
    \\
    &\stackrel{\text{(a)}}{\leq}
    \prob{\alpha_{m, k} \pars{\frac{\norm{\bW^{(1)}}_2^2 + \ldots + \norm{\bW^{(k-1)}}_2^2}{\dimRFF (k-1)}}^{k-1} \sup_{\bs, \bt \in \cX} \norm{\partial_{\bs, \bt} \rffkernel_{k}(\bs, \bt) - \partial_{\bs, \bt} \kernel(\bs, \bt)}_2 \geq \epsilon}
    \\
    &\stackrel{\text{(b)}}{=}
    \prob{\underbrace{\frac{\norm{\bW^{(1)}}_2^2 + \ldots + \norm{\bW^{(k-1)}}_2^2}{\dimRFF (k-1)}}_{(A_k)} \underbrace{\sup_{\bs, \bt \in \cX} \norm{\partial_{\bs, \bt} \rffkernel_{k}(\bs, \bt) - \partial_{\bs, \bt} \kernel(\bs, \bt)}_2^{\frac{1}{k-1}}}_{(B_k)} \geq \pars{\frac{\epsilon}{\alpha_{m, k}}}^{\frac{1}{k-1}}},
\end{align}
where in (a) we used the arithmetic-geometric mean inequality, and in (b) we divided both sides by $\alpha_{m, k}$ and took the $(k-1)$th root.
Further, setting $t \coloneqq \pars{\frac{\epsilon}{\alpha_{m, k}}}^{\frac{1}{k-1}}$, we have for $\gamma > 0$
\begin{align}
    P_k(\epsilon)
    &\leq &
    \prob{A_k \cdot B_k \ge t}
    \stackrel{\text{(c)}}{\leq}
    \prob{\pars{A_k - \gamma} B_k \geq \frac{t}{2}}
     +
    \prob{B_k \geq \frac{t}{2\gamma}}
    \\
    &\stackrel{\text{(d)}}{\leq} &
    \inf_{\tau > 0} \curls{
    \prob{A_k - \gamma \geq \frac{\tau}{2}}
    +
    \prob{B_k \geq \frac{t}{\tau}}}
     +
    \prob{B_k \geq \frac{t}{2\gamma}}, \label{eq:union bound}
\end{align}
where in (c) we added and subtracted $\gamma B_k$ and applied a union bound, while in (d) we combined a union bound with the relation $\{XY \geq \epsilon\} \subseteq \{X \geq \tau\} \bigcup \{Y \geq \epsilon / \tau\}$ which holds for any $\tau > 0$.

Our aim is now to obtain good probabilistic bounds on $A_k$ and $B_k$ to use in \eqref{eq:union bound} with the specific choice of $\gamma = \expe{A_k} = \sigma_\Lambda^2$.

\emph{Bounding $A_k$.}
By the inequality between the spectral and Frobenius norms, we have
\begin{align}
    \sum_{p=1}^{k-1} \norm{\bW^{(p)}}_2^2 \leq \sum_{p=1}^{k-1} \norm{\bW^{(p)}}_F^2 = \sum_{i=1}^d \sum_{p=1}^{k-1} \sum_{j=1}^{\dimRFF} \pars{w_{i, j}^{(p)}}^2.
\end{align}
Now note that $w_{i,j}^{(p)}$ are $\iid$ copies of the $i$\textsuperscript{th} marginal of $\Lambda$ for all $j \in \bracks{\dimRFF}$ and $p \in \bracks{k-1}$, so that via \eqref{eq:rfsf_approx_cond} $\pars{w_{i,j}^{(p)}}^2$ satisfies the Bernstein moment condition
\begin{align}
    \expe{\pars{w_{i,j}^{(p)}}^{2k}} \leq \frac{k! S^2 R^{k-2}}{2} \quad \text{for all} \spc i \in [d], j \in [\dimRFF], p \in [m], k \geq 2.
\end{align}
Hence, we may apply the Bernstein inequality from Theorem \ref{thm:bernstein_onetail} so that for $i \in \bracks{d}$
\begin{align}
    \prob{\frac{1}{\dimRFF(k-1)} \sumnolim_{p=1}^{k-1}\sumnolim_{j=1}^{\dimRFF} \pars{w_{i, j}^{(p)}}^2 - \sigma_i^2 \geq \epsilon}
    \leq
    \exp\pars{\frac{-\dimRFF(k-1) \epsilon^2}{2(S^2 + R\epsilon)}}, \label{eq:w_coord_prob_bound}
\end{align}
where  $\sigma_i^2 = \bbE_{\bw \sim \Lambda}\bracks{w_{i}^2}$.
Combining these bounds for all $i \in \bracks{d}$ and denoting $\sigma_\Lambda^2 = \sum_{i=1}^d \sigma_i^2$,
\begin{align}
    \prob{A_k - \sigma_\Lambda^2\geq \epsilon}
    &=
    \prob{\frac{1}{\dimRFF(k-1)}\sum_{i=1}^d \sum_{p=1}^{k-1} \sum_{j=1}^{\dimRFF} \pars{w_{i, j}^{(p)}}^2 - \sigma_\Lambda^2 \geq \epsilon}
    \\
    &\leq
    \sum_{i=1}^d \prob{\frac{1}{\dimRFF(k-1)}\sumnolim_{p=1}^{k-1} \sumnolim_{j=1}^{\dimRFF} \pars{w_{i, j}^{(p)}}^2 - \sigma_i^2 \geq \frac{\epsilon}{d}}
    \\
    &\leq
    d \exp\pars{\frac{-\dimRFF(k-1) \pars{\frac{\epsilon}{d}}^2}{2\pars{S^2 + R\frac{\epsilon}{d}}}}. \label{eq:w_full_prob_bound}
\end{align}
Hence, we have the required probabilistic bound for the term in \eqref{eq:union bound} containing $A_k$.

\emph{Bounding $B_k$.}
    One can bound the spectral norm by the max norm so that
    \begin{align}
        B_k^{k-1} = \sup_{\bs, \bt \in \cX} \norm{\partial_{\bs, \bt} \rffkernel_{k}(\bs, \bt) - \partial_{\bs,\bt} \kernel(\bs, \bt)}_2
        &\leq
        \sup_{\bs, \bt \in \cX} \norm{\partial_{\bs, \bt} \rffkernel_{k}(\bs, \bt) - \partial_{\bs,\bt} \kernel(\bs, \bt)}_{\max}
        \\
        &= \max_{i,j=1,\dots, d} \sup_{\bs, \bt \in \cX} \abs{\frac{\partial^2 \rffkernel_{k}(\bs, \bt)}{\partial s_i \partial t_j} - \frac{\partial^2 \kernel(\bs, \bt)}{\partial s_i \partial t_j}} \label{eq:rff_derivs_max}.
    \end{align}
    Let $i, j \in \bracks{d}$ and denote $E_{i,j} \coloneqq \bbE_{\bw \sim \Lambda}\bracks{\abs{w_i w_j} \norm{\bw}_2}$ and $D_{i, j} \coloneqq \sup_{\bz \in \cX_\Delta} \norm{\nabla \bracks{\partial^{\be_i, \be_j} \kernel(\bz)}}_2$, which are finite as previously shown. Due to Hölder's inequality (Lemma \ref{lem:holder}) and \eqref{eq:rfsf_approx_cond},
    \begin{align}
        \expe{\abs{w_i w_j}^m} \leq \bbE^{1/2}\bracks{w_i^{2m}} \bbE^{1/2}\bracks{w_j^{2m}} < \infty.
    \end{align}
    Recall that $\kernel$ is $3$-times continuously differentiable, so that the conditions required by Theorem \ref{thm:rff_derivative_approx} are satisfied, that we now call to our aid in controlling the \RFF{} kernel derivatives,
    \begin{align}
        \prob{\sup_{\bs, \bt \in \cX} \abs{\frac{\partial^2 \rffkernel_{k}(\bs, \bt)}{\partial s_i \partial t_j} - \frac{\partial^2 \kernel(\bs, \bt)}{\partial s_i \partial t_j}}
        \geq \epsilon}
        \leq 16 C^\prime_{d, \cX, i, j} \epsilon^{-\frac{d}{d+1}} \exp\pars{\frac{-\dimRFF \epsilon^2}{4(d+1)(2S^2 + R\epsilon)}},
    \end{align}
    where we defined $C^\prime_{d, \cX, i, j} \coloneqq \pars{\abs{\cX} (D_{i, j} + E_{i, j})}^\frac{d}{d+1}$. 
    Hence, noting that the max satisfies the relation $\{\max_i \xi_i \geq \epsilon\} = \bigcup_i \{\xi_i \geq \epsilon\}$
    and union bounding \eqref{eq:rff_derivs_max} in probability, we get that
    \begin{align}
        \prob{B_k^{k-1} \geq \epsilon}
        &\leq
        \sum_{i,j=1}^d
        \prob{\sup_{\bs, \bt \in \cX} \abs{\frac{\partial^2 \rffkernel_{k}(\bs, \bt)}{\partial s_i \partial t_j} - \frac{\partial^2 \kernel(\bs, \bt)}{\partial s_i \partial t_j}} \geq \epsilon}
        \\
        &\leq
        16 C^\prime_{d, \cX} \epsilon^{-\frac{d}{d+1}} \exp\pars{\frac{-\dimRFF \epsilon^2}{4(d+1)(2S^2 + R\epsilon)}}, \label{eq:rff_deriv_prob_bound}
    \end{align}
    where we denote $C^\prime_{d, \cX} \coloneqq \sum_{i,j=1}^d C^\prime_{d, \cX, i, j} = \abs{\cX}^\frac{d}{d+1} \sum_{i,j=1}^d \pars{D_{i,j} + E_{i,j}}^\frac{d}{d+1}$.

\emph{Putting it together.} Now that we  have our bounds for $A_k$ and $B_k$, we put everything together, that is, plug the bounds \eqref{eq:w_full_prob_bound} and \eqref{eq:rff_deriv_prob_bound} into \eqref{eq:union bound}, so that we get
\begin{align}
    P_k(\epsilon)
    \leq&
    \inf_{\tau > 0} \curls{
    d \exp\pars{\frac{-\dimRFF(k-1) \pars{\frac{\tau}{2d}}^2}{2\pars{S^2 + R\frac{\tau}{2d}}}}
    +
    16 C^\prime_{d, \cX} \pars{\frac{\tau}{t}}^\frac{d(k-1)}{d+1} \exp\pars{\frac{-\dimRFF \pars{\frac{t}{\tau}}^{2(k-1)}}{4(d+1) \pars{2S^2 + R\pars{\frac{t}{\tau}}^{k-1}}}}
    }
    \\
    & +
    16 C^\prime_{d, \cX} \pars{\frac{2\sigma_\Lambda^2}{t}}^\frac{d(k-1)}{d+1}  \exp\pars{\frac{-\dimRFF \pars{\frac{t}{2\sigma_\Lambda^2}}^{2(k-1)}}{4(d+1) \pars{2 S^2 + R\pars{\frac{t}{2\sigma_\Lambda^2}}^{k-1}}}}
    \\
    \stackrel{\text{(e)}}{\leq} &
    d \exp\pars{\frac{-\dimRFF(k-1) \pars{\frac{t^{\frac{k-1}{k}}}{2d}}^2}{2\pars{S^2 + R\frac{t^\frac{k-1}{k}}{2d}}}}
    +
    16 C^\prime_{d, \cX} \pars{\frac{1}{t}}^\frac{d(k-1)}{(d+1)k} \exp\pars{\frac{-\dimRFF t^\frac{2(k-1)}{k}}{4(d+1) \pars{2S^2 + Rt^\frac{k-1}{k}}}}
    \\
    & +
    16 C^\prime_{d, \cX} \pars{\frac{2\sigma_\Lambda^2}{t}}^\frac{d(k-1)}{d+1}  \exp\pars{\frac{-\dimRFF \pars{\frac{t}{2\sigma_\Lambda^2}}^{2(k-1)}}{4(d+1) \pars{2 S^2 + R\pars{\frac{t}{2\sigma_\Lambda^2}}^{k-1}}}}
    \\
    \stackrel{\text{(f)}}{=} &
    d \exp\pars{\frac{-\dimRFF(k-1) \pars{\frac{\pars{\epsilon/\alpha_{m,k}}^\frac{1}{k}}{2d}}^2}{2\pars{S^2 + R\frac{\pars{\epsilon/\alpha_{m, k}}^\frac{1}{k}}{2d}}}}
    +
    16 C^\prime_{d, \cX} \pars{\frac{\alpha_{m, k}}{\epsilon}}^\frac{d}{(d+1)k} \exp\pars{\frac{-\dimRFF \pars{\frac{\epsilon}{\alpha_{m,k}}}^\frac{2}{k}}{4(d+1) \pars{2S^2 + R\pars{\frac{\epsilon}{\alpha_{m,k}}}^\frac{1}{k}}}}
    \\
    & +
    16 C^\prime_{d, \cX} \pars{\frac{\alpha_{m,k}\pars{2\sigma_\Lambda^2}^{k-1}}{\epsilon}}^\frac{d}{d+1} \exp\pars{\frac{-\dimRFF \pars{\frac{\epsilon}{\alpha_{m,k} \pars{2\sigma_\Lambda^2}^{k-1}}}^2}{4(d+1) \pars{2 S^2 + R\frac{\epsilon}{\alpha_{m,k} \pars{2\sigma_\Lambda^2}^{k-1}}}}},
\end{align}
where (a) follows from substituting \eqref{eq:rff_deriv_prob_bound} and \eqref{eq:w_full_prob_bound} into \eqref{eq:union bound} with the choice of $\gamma = \sigma_\Lambda^2$, (b) from choosing $\tau = t^\frac{k-1}{k}$, and (c) from putting back $t = (\epsilon/\alpha_{m, k})^\frac{1}{k-1}$.

Note that the previous applies for all $k \geq 2$. For $k=1$, we have by \eqref{eq:rff_deriv_prob_bound} that
\begin{align}
    P_1(\epsilon) &= \prob{\sup_{\bs, \bt \in \cX} \norm{\partial_{\bs, \bt} \rffkernel_{1}(\bs, \bt) - \partial_{\bs, \bt} \kernel(\bs, \bt)}_2 \geq \frac{\epsilon}{\alpha_{m, 1}}}
    \\
    &\leq
    16 C^\prime_{d, \cX} \pars{\frac{\alpha_{m, 1}}{\epsilon}}^\frac{d}{d+1} \exp\pars{\frac{-\dimRFF \pars{\frac{\epsilon}{\alpha_{m, 1}}}^2}{4(d+1)\pars{2S^2 + R\frac{\epsilon }{\alpha_{m, 1}}}}}.
\end{align}
\emph{Combining and simplifying.} We can now combine the bounds for $P_1, \dots, P_m$ into \eqref{eq:rfsf_bound_rhs},
\begin{align}
    \bbP & \bracks{\sup_{\substack{\bx, \by \in \seq\\\norm{\bx}_\onevar,\norm{\by}_\onevar \leq V}} \abs{\rffsigkernel[m](\bx, \by) - \sigkernel[m](\bx, \by)} \geq \epsilon} \leq \sum_{k=1}^m P_k\pars{\frac{\epsilon}{m}}
    \\
    \stackrel{\text{(g)}}{\leq} &
    2^\frac{1}{d+1} 8 C^\prime_{d, \cX}\sum_{k=1}^{m} \pars{\frac{2^k {\sigma_\Lambda^2}^{k-1} m \alpha_{m, k}}{\epsilon}}^\frac{d}{d+1}
    \exp\pars{- \frac{\dimRFF}{2(d+1)} \cdot \frac{\pars{\frac{\epsilon}{2^k \sigma_\Lambda^{2(k-1)} m\alpha_{m, k}}}^2}{S^2 + R \frac{\epsilon}{2^k \sigma_\Lambda^{2(k-1)} m\alpha_{m, k} }}}
    \\
    & +
    2^\frac{1}{d+1} 8 C^\prime_{d, \cX} \sum_{k=2}^m \pars{\frac{2^k m\alpha_{m, k}}{\epsilon}}^\frac{d}{(d+1)k}
    \exp\pars{- \frac{\dimRFF}{2(d+1)} \cdot \frac{\pars{\frac{\epsilon}{2^k m\alpha_{m,k}}}^\frac{2}{k}}{ S^2 + R\pars{\frac{\epsilon}{2^k m\alpha_{m,k}}}^\frac{1}{k}}}
    \\
    & +
    d \sum_{k=2}^M
    \exp\pars{-\frac{\dimRFF(k-1)}{2} \frac{\pars{\frac{\epsilon}{2^k d^k m\alpha_{m,k}}}^\frac{2}{k}}{S^2 + R \pars{\frac{\epsilon}{2^k d^k m\alpha_{m, k}}}^\frac{1}{k}}}
    \\
    \stackrel{\text{(h)}}{\leq} &
    2^\frac{1}{d+1} 8 C^\prime_{d, \cX} \sum_{k=1}^{m} \pars{\frac{2^k \pars{\sigma_\Lambda^2 \vee d}^k m \alpha_{m, k}}{\epsilon}}^\frac{d}{d+1}
    \exp\pars{- \frac{\dimRFF}{2(d+1)} \cdot \frac{\pars{\frac{\epsilon}{2^k \pars{\sigma_\Lambda^2 \vee d}^k m\alpha_{m, k}}}^2}{S^2 + R \frac{\epsilon}{2^k \pars{\sigma_\Lambda^2 \vee d}^k m\alpha_{m, k} }}}
    \\
    & +
    2^\frac{1}{d+1} 8 C^\prime_{d, \cX} \sum_{k=1}^m \pars{\frac{2^k \pars{\sigma_\Lambda^2 \vee d}^k m\alpha_{m, k}}{\epsilon}}^\frac{d}{(d+1)k}
    \exp\pars{- \frac{\dimRFF}{2(d+1)} \cdot \frac{\pars{\frac{\epsilon}{2^k \pars{\sigma_\Lambda^2 \vee d}^k m\alpha_{m,k}}}^\frac{2}{k}}{ S^2 + R\pars{\frac{\epsilon}{2^k \pars{\sigma_\Lambda^2 \vee d}^k m\alpha_{m,k}}}^\frac{1}{k}}}
    \\
    & +
    d \sum_{k=1}^m
    \exp\pars{-\frac{\dimRFF}{2(d+1)} \frac{\pars{\frac{\epsilon}{2^k \pars{\sigma_\Lambda^2 \vee d}^k m\alpha_{m,k}}}^\frac{2}{k}}{S^2 + R \pars{\frac{\epsilon}{2^k \pars{\sigma_\Lambda^2 \vee d}^k m\alpha_{m, k}}}^\frac{1}{k}}}
    \\
    \stackrel{\text{(f)}}{\leq} &
    2^\frac{1}{d+1} 8 C^\prime_{d, \cX} \sum_{k=1}^{m} \pars{\frac{\beta_{d, m, V}}{\epsilon}}^\frac{d}{d+1}
    \exp\pars{- \frac{\dimRFF}{2(d+1)} \frac{\pars{\frac{\epsilon}{\beta_{d, m, V}}}^2}{S^2 + R \frac{\epsilon}{\beta{d, m, V}}}}
    \\
    & +
    2^\frac{1}{d+1} 8 C^\prime_{d, \cX} \sum_{k=1}^m \pars{\frac{\beta_{d, m, V}}{\epsilon}}^\frac{d}{(d+1)k}
    \exp\pars{- \frac{\dimRFF}{2(d+1)} \frac{\pars{\frac{\epsilon}{\beta_{d, m, V}}}^\frac{2}{k}}{S^2 + R\pars{\frac{\epsilon}{\beta_{d, m, V}}}^\frac{1}{k}}}
    \\
    & +
    d \sum_{k=1}^m
    \exp\pars{-\frac{\dimRFF}{2(d+1)} \frac{\pars{\frac{\epsilon}{\beta_{d, m, V}}}^\frac{2}{k}}{S^2 + R\pars{\frac{\epsilon}{\beta_{d, m, V}}}^\frac{1}{k}}},
\end{align}
where (g) follows from rearranging the expressions from (f), while (h) from unifying the coefficients and that $1 \leq d \leq \max(\sigma_\Lambda^2, d)$ and $f(x) = x^2 / (a + bx)$ is monotonically increasing in $x$ on the positive half-line for $a,b > 0$, while (f) from $\alpha_{m, k} = V^{2m} L^{2(m-k)} / ((k-1)!)^2 \leq (VL)^{2m}$, using that $f(x)$ is increasing, and defining $\beta_{d, m, V} \coloneqq m\pars{2 V^{2} (L^2 \vee 1) (\sigma_\Lambda^2 \vee d)}^m$.

\emph{Conclusion.}
Finally, we split the bound into two cases: the first case is if the error is big, i.e. $\epsilon \geq \beta_{d, m, V} = m\pars{2 V^{2} (L^2 \vee 1) (\sigma_\Lambda^2 \vee d)}^m$, when we decrease all the exponents to $\frac{1}{m}$,
\begin{align}
    \bbP & \bracks{\sup_{\substack{\bx, \by \in \seq\\\norm{\bx}_\onevar, \norm{\by}_\onevar \leq V}} \abs{\rffsigkernel[m](\bx, \by) - \sigkernel[m](\bx, \by)} \geq \epsilon}
    \\
    \leq &
    m \cdot \pars{2^\frac{1}{d+1} 16 C^\prime_{d, \cX} \pars{\frac{\beta_{d, m, V}}{\epsilon}}^{\frac{d}{(d+1)m}} + d}
    \exp\pars{-\frac{\dimRFF}{2(d+1)} \frac{\pars{\frac{\epsilon}{\beta_{d, m, V}}}^\frac{2}{m}}
    {S^2 + R \pars{\frac{\epsilon}{\beta_{d, m, V}}}^{\frac{1}{m}}}}
\end{align}
The other when the error is small, i.e. $\epsilon < \beta_{d, m, V}$, when we increase all the exponents to $1$
\begin{align}
    \bbP & \bracks{\sup_{\substack{\bx, \by \in \seq\\\norm{\bx}_\onevar, \norm{\by}_\onevar \leq V}} \abs{\rffsigkernel[m](\bx, \by) - \sigkernel(\bx, \by)} \geq \epsilon}
    \\
    \leq &
    m \cdot \pars{2^\frac{1}{d+1} 16 C^\prime_{d, \cX} \pars{\frac{\beta_{d, m, V}}{\epsilon}}^\frac{d}{d+1} + d}
    \exp\pars{-\frac{\dimRFF}{2(d+1)} \frac{\pars{\frac{\epsilon}{\beta_{d, m, V}}}^{2}}
    {S^2 + R \pars{\frac{\epsilon}{\beta_{d, m, V}}}}}.
\end{align}
The claimed estimate follows by denoting $C_{d, \cX} \coloneqq 2^\frac{1}{d+1}16 C^\prime_{d, \cX}$ and simplifying.
\end{proof}

Next, we prove Theorem \ref{thm:main2} to show an approximation bound for the \RFSFD{} kernel.

\begin{theorem}[Concentration inequality for \RFSFD{} kernel] \label{thm:rfsf_dp_approx}
    Let $\kernel: \bbR^d \times \bbR^d \to \bbR$ be a continuous, bounded, translation-invariant kernel with spectral measure $\Lambda$, which satisfies for some $S, R > 0$ that
    \begin{align} \label{eq:rfsf_dp_approx_cond}
        \bbE_{\bw \sim \Lambda}\bracks{\abs{w_i}^{2k}} \leq \frac{k! S^2 R^{k-2}}{2} \quad \text{for all} \spc i \in [d] \spc \text{and} \spc k \geq 2.
    \end{align}

    Then, for signature level $m \in \mathbb{Z}_+$ and $\bx, \by \in \seq$, it holds for $\epsilon > 0$ that:
    \begin{align}
    \bbP\bracks{\abs{\rffsigkernelDP[m](\bx, \by) - \sigkernel[m](\bx, \by)} \geq \epsilon}
        \leq
        2\exp\pars{-\frac{1}{4}\min
        \curls{\begin{array}{c}
        \pars{\frac{\sqrt{\dimRFF} \epsilon}{2C_{d, m}\norm{\bx}_\onevar^m \norm{\by}_\onevar^m}}^2,
        \\
        \pars{\frac{\dimRFF \epsilon}{\sqrt{8}C_{d, m} \norm{\bx}_\onevar^m \norm{\by}_\onevar^m}}^{\frac{1}{m}}
        \end{array}}
        },
    \end{align}
    where the absolute constant $C_{d, m} > 0$ satisfies that
    \begin{align}
    C_{d, m} \leq
    \sqrt{8} e^4 (2\pi)^{1/4} e^{1/24} (4e^3/m)^m \pars{\pars{2d\max(S, R)}^m + \pars{L^2/\ln 2}^m}.
    \end{align}
\end{theorem}
\begin{proof}
    Let $\rffsigkernelhat[m]^{(1)}, \dots, \rffsigkernelhat[m]^{(\dimRFF)}$ be independent copies of the \RFSF{} kernel, each with internal \RFF{} sample size $\hat d = 1$, such that $\rffsigkernelDP[m] = \frac{1}{\dimRFF} \sum_{k=1}^{\dimRFF} \rffsigkernelhat[m]^{(k)}$.
    Our goal is to call Theorem \ref{thm:alpha_subexp_concentration} with $\alpha = \frac{1}{m}$, and therefore, we compute an upper bound on the $\Psi_{1/m}$-norm of $\rffsigkernel[m]^{(k)}(\bx, \by) - \sigkernel[m](\bx, \by)$ for all $k \in [\dimRFF]$; for a definition of the $\alpha$-exponential norm, see Definition~\ref{def:alpha_subexp_norm}.

    By Lemma \ref{lem:1var_rffsig_norm}, it holds for any $\bx, \by \in \seq$ and $k \in [m]$ that
    \begin{align}
        \abs{\rffsigkernelhat[m]^{(k)}(\bx, \by)} \leq \frac{\pars{\norm{\bx}_\onevar \norm{\by}_\onevar}^m}{(m!)^2} \norm{\bw_k^{(1)}}_2^2 \cdots \norm{\bw_k^{(m)}}_2^2,
    \end{align}
    where $\bw_k^{(1)}, \dots, \bw_k^{(m)} \stackrel{\iid}{\sim} \Lambda$ are the random weights that parametrize $\rffsigkernelhat[m]^{(k)}$ for all $k \in [\dimRFF]$.
    Now, calling Lemma \ref{lem:alpha_bernstein_cond} with $\alpha=1$ yields that, due to \eqref{eq:rfsf_dp_approx_cond}, the following holds:
    \begin{align}
        \norm{{w^{(p)}_{k, i}}^2}_{\Psi_1} \leq 2\max(S, R) \quad \text{for all} \spc i \in [d], k \in [\dimRFF], p \in [m].
    \end{align}
    Note that for $\alpha = 1$, $\norm{\cdot}_{\Psi_\alpha}$ satisfies the triangle inequality (see Lemma \ref{lem:alpha_exp_triangle}), and hence
    \begin{align} \label{eq:bw_alpha_norm_bound}
        \norm{\norm{\bw_k^{(p)}}_2^2}_{\Psi_1} \leq \sum_{i=1}^d \norm{{w_{k,i}^{(p)}}^2}_{\Psi_1} \leq 2d \max(S, R) \quad \text{for all} \spc k \in [\dimRFF], p \in [m].
    \end{align}
    As $\norm{\cdot}_{\Psi_\alpha}$ is positive homogenous and satisfies a Hölder-type inequality (see Lemma \ref{lem:alpha_exp_holder}):
    \begin{align}
        \norm{\rffsigkernelhat[m]^{(k)}(\bx, \by)}_{\Psi_{1/m}} &\stackrel{\text{(a)}}{\leq} \frac{\pars{\norm{\bx}_\onevar \norm{\by}_\onevar}^m}{(m!)^2} \norm{\norm{\bw_k^{(1)}}_2^2}_{\Psi_1} \cdots \norm{\norm{\bw_k^{(m)}}_2^2}_{\Psi_1}
        \\
        &\stackrel{\text{(b)}}{\leq} \frac{\pars{2d\norm{\bx}_\onevar \norm{\by}_\onevar \max(S, R)}^m}{(m!)^2} \quad\text{for all}\spc k \in [\dimRFF], \label{eq:rfsf_1dim_alpha_norm}
    \end{align}
    where in (a) we used Corollary \ref{lem:krffsig_bound}, in (b) we used \eqref{eq:rfsf_1dim_alpha_norm}.
    We are almost ready to use Theorem \ref{thm:alpha_subexp_concentration}, but it requires the $\norm{\cdot}_{\Psi_{1/m}}$ bound in terms of centered random variables.
    Although $\norm{\cdot}_{\Psi_\alpha}$ does not satisfy the triangle inequality for $\alpha \in (0, 1)$, it obeys that (see \cite[Lemma~A.3.]{gotze2021concentration}) $\norm{X + Y}_{\Psi_\alpha}\leq 2^{1/\alpha}\pars{\norm{X}_{\Psi_\alpha} + \norm{Y}_{\Psi_\alpha}}$ for any random variables $X$ and $Y$. For a constant $c \in \bbR$, we have $\norm{c}_{\Psi_{1/m}} = \frac{\abs{c}}{\ln^m 2}$, and hence by Lemma \ref{lem:ksig_bound} we have that
    \begin{align}
        \norm{\sigkernel[m](\bx, \by)}_{\Psi_{1/m}} \leq \frac{\pars{L^2\norm{\bx}_\onevar \norm{\by}_\onevar / \ln2}^m}{(m!)^2},
    \end{align}
    where $L = \norm{\bbE_{\bw \sim \Lambda}\bracks{\bw\bw^\top}}_2$ is the Lipschitz constant of the kernel $\kernel: \cX \times \cX \to \bbR$. This gives
    \begin{align}
        \norm{\rffsigkernelhat[m]^{(k)}(\bx, \by) - \sigkernel[m](\bx, \by)}_{\Psi_{1/m}}
        &\leq
        2^m\pars{\norm{\rffsigkernelhat[m]^{(k)}(\bx, \by)}_{\Psi_{1/m}} + \norm{\sigkernel[m](\bx, \by)}_{\Psi_{1/m}}}
        \\
        &\leq
        \frac{\pars{2\norm{\bx}_\onevar\norm{\by}_\onevar}^m}{(m!)^2}\pars{\pars{2d\max(S, R)}^m + \pars{L^2/\ln 2}^m}.
    \end{align}
    Finally, we have the required Orlicz norm bound for invoking Theorem \ref{thm:alpha_subexp_concentration}, so that we get
    \begin{align}
        \bbP&\bracks{\abs{\rffsigkernelDP[m](\bx, \by) - \sigkernel[m](\bx, \by)} \geq \epsilon}
        \\&\leq
        2\exp\pars{-\frac{1}{4}\min\curls{
        \pars{\frac{\sqrt{\dimRFF} \epsilon}{2C_{d, m}}}^2,
        \pars{\frac{\dimRFF \epsilon}{\sqrt{8}C_{d, m}}}^{\frac{1}{m}}}},
    \end{align}
    where the constant $C_{d, m} > 0$ is defined as
    \begin{align}
    C_{d, m}
    \coloneqq
    \sqrt{8} e^4 (2\pi)^{1/4} e^{1/24} \frac{(4em \norm{\bx}_\onevar \norm{\by}_\onevar )^m}{(m!)^2} \pars{\pars{2d\max(S, R)}^m + \pars{L^2/\ln 2}^m},
    \end{align}
    and invoking Stirling's approximation $\frac{1}{m!} \leq \pars{\frac{e}{m}}^m$ gives the stated result.
\end{proof}

Now, we prove the analogous result for $\rffsigkernelTRP[m]$.
\begin{theorem}[Concentration inequality for \RFSFT{} kernel] \label{thm:rfsf_trp_approx}
    Let $\kernel: \bbR^d \times \bbR^d \to \bbR$ be a continuous, bounded, translation-invariant kernel with spectral measure $\Lambda$, which satisfies for some $S, R > 0$ that
    \begin{align} \label{eq:rfsf_trp_approx_cond}
        \bbE_{\bw \sim \Lambda}\bracks{\abs{w_i}^{2k}} \leq \frac{k! S^2 R^{k-2}}{2} \quad \text{for all} \spc i \in [d] \spc \text{and} \spc k \geq 2.
    \end{align}
    Then, for the level-$m$ \RFSFT{} kernel as defined in \eqref{eq:rffsigtrpkernel_def}, we have for $\bx, \by \in \seq$ and $\epsilon > 0$
    \begin{align}
        \bbP\bracks{\abs{\rffsigkernelTRP[m](\bx, \by) - \rffsigkernel[m](\bx, \by)} \geq \epsilon}
        \leq
        C_{d, \Lambda}
        \exp\pars{- \pars{\frac{m^2 \dimTRP^{\frac{1}{2m}} \epsilon^{\frac{1}{m}}}{2\sqrt{2}e^3 R \norm{\bx}_\onevar \norm{\by}_\onevar}}^\frac{1}{2}},
    \end{align}
    where $C_{d, \Lambda} \coloneqq 2\pars{1 + \frac{S}{2R} + \frac{S^2}{4R^2}}^d$.
\end{theorem}
\begin{proof}
    First, we consider the conditional probability $\prob{\abs{\rffsigkernelTRP[m](\bx, \by) - \rffsigkernel[m](\bx, \by)} \geq \epsilon \middle\vert \bW}$ by conditioning on the \RFSF{} weights $\bW \coloneqq (\bW^{(1)}, \dots, \bW^{(m)})$, so that the only source of randomness comes from the \TRP{} weights $\bP \coloneqq (\bP^{(1)}, \dots, \bP^{(m)})$. The idea is to call Theorem \ref{thm:hyper_concentration} to estimate the conditional probability, and then take expectation over $\bW$. Since Theorem \ref{thm:hyper_concentration} quantifies the concentration of a Gaussian polynomial around its mean in terms of its variance, we first compute the conditional statistics of $\rffsigkernelTRP[m](\bx, \by)$.
    
    \emph{Conditional expectation.}
    Recall the definition of $\rffsigTRP[m](\bx)$ \eqref{eq:rffsigtrpdef}, where $\bp^{(1)}_1, \dots, \bp^{(m)}_{\dimTRP} \stackrel{\iid}{\sim} \cN(0, \b I_{2\dimRFF})$, and $\bx \in \seq$, so that
    \begin{align}
        \rffsigTRP[m](\bx) &= \frac{1}{\sqrt{\dimTRP}} \pars{\sum_{\bi \in \Delta_m(\ell_\bx)} \prod_{p=1}^m \inner{\bp_i^{(p)}}{\delta \rff_p(\bx_{i_p})}}_{i=1}^{\dimTRP}.
    \end{align}
    
    Hence, $\rffsigkernelTRP[m](\bx, \by)$ is written for $\bx, \by \in \seq$ as
    \begin{align}
        \rffsigkernelTRP[m](\bx, \by) = \frac{1}{\dimTRP} \sum_{i=1}^{\dimTRP} \underbrace{\sum_{\substack{\bi \in \Delta_m(\ell_\bx) \\ \bj \in \Delta_m(\ell_\by)}} \prod_{p=1}^m \inner{\bp_i^{(p)}}{\delta \rff_p(\bx_{i_p})} \inner{\bp_i^{(p)}}{\delta \rff_p(\by_{i_p})}}_{A_i},
    \end{align}
    which is a sample average of $\dimTRP$ $\iid$ terms, i.e.~$\rffsigkernelTRP[m](\bx, \by) = \frac{1}{\dimTRP} \sum_{i=1}^{\dimTRP} A_i$. We only have to verify that $A_i$ is conditionally an unbiased approximator of $\rffsigkernel[m]$ given $\bW$. 
    \begin{align}
        \expe{A_i \cond \bW} &\stackrel{\text{(a)}}{=} \sum_{\substack{\bi \in \Delta_m(\ell_\bx) \\ \bj \in \Delta_m(\ell_\by)}} \prod_{p=1}^m  \expe{\inner{\bp^{(p)}_i}{\delta \rff_p(\bx_{i_p})} \inner{\bp^{(p)}_i}{\delta \rff_p(\by_{j_p})} \cond \bW}
        \\
        &\stackrel{\text{(b)}}{=}
        \sum_{\substack{\bi \in \Delta_m(\ell_\bx) \\ \bj \in \Delta_m(\ell_\by)}} \prod_{p=1}^m \inner{\expe{\bp^{(p)}_i \otimes \bp^{(p)}_i}}{\delta \rff_p(\bx_{i_p}) \otimes \delta \rff_p(\by_{j_p})}
        \\
        &\stackrel{\text{(c)}}{=}
        \sum_{\substack{\bi \in \Delta_m(\ell_\bx) \\ \bj \in \Delta_m(\ell_\by)}} \prod_{p=1}^m \inner{I_{\dimRFF}}{\delta \rff_p(\bx_{i_p}) \otimes \delta \rff_p(\by_{j_p})}
        \\
        &\stackrel{\text{(d)}}{=}
        \sum_{\substack{\bi \in \Delta_m(\ell_\bx) \\ \bj \in \Delta_m(\ell_\by)}} \prod_{p=1}^m \inner{\delta \rff_p(\bx_{i_p})}{\delta \rff_p(\by_{j_p})},
    \end{align}
    where (a) follows from linearity of expectation and independence of the $\bp_i^{(p)}$'s for $p \in [m]$, (b) from bilinearity of inner product, and independence of $\bP$ and $\bW$, (c) from substituting the covariance, (d) is since the outer product is projected onto the diagonal.

    \emph{Conditional variance.} We compute the conditional variance of $A_i$ given $\bW$:
    \begin{align}
        \bbE&\bracks{A_{m, i}^2 \cond \bW}&&
        \\
        &\stackrel{\text{(e)}}{=} \sum_{\substack{\bi,\bk \in \Delta_m({\abs{\bx}})\\\bj,\bl \in \Delta_m({\abs{\by}})}} \prod_{p=1}^m
        &&\bbE\bracks{\inner{\bp^{(p)}_i}{\delta \rff_p(\bx_{i_p})} \inner{\bp^{(p)}_i}{\delta \rff_p(\by_{j_p})}\inner{\bp^{(p)}_i}{\delta \rff_p(\bx_{k_p})} \inner{\bp^{(p)}_i}{\delta \rff_p(\by_{l_p})} \cond \bW}
        \\
        &\stackrel{\text{(f)}}{=}
        \sum_{\substack{\bi,\bk \in \Delta_m({\abs{\bx}})\\\bj,\bl \in \Delta_m({\abs{\by}})}} \prod_{p=1}^m
        &&\bigg(
        \bbE\bracks{\inner{\bp^{(p)}_i}{\delta \rff_p(\bx_{i_p})} \inner{\bp^{(p)}_i}{\delta \rff_p(\by_{j_p})} \cond \bW} \bbE\bracks{\inner{\bp^{(p)}_i}{\delta \rff_p(\bx_{k_p})} \inner{\bp^{(p)}_i}{\delta \rff_p(\by_{l_p})} \cond \bW}
        \\
        &&&+
        \bbE\bracks{\inner{\bp^{(p)}_i}{\delta \rff_p(\bx_{i_p})} \inner{\bp^{(p)}_i}{\delta \rff_p(\bx_{k_p})}\cond \bW} \bbE\bracks{\inner{\bp^{(p)}_i}{\delta \rff_p(\by_{j_p})} \inner{\bp^{(p)}_i}{\delta \rff_p(\by_{l_p})} \cond \bW}
        \\
        &&&+
        \bbE\bracks{\inner{\bp^{(p)}_i}{\delta \rff_p(\bx_{i_p})} \inner{\bp^{(p)}_i}{\delta \rff_p(\by_{l_p})} \cond \bW} \bbE\bracks{\inner{\bp^{(p)}_i}{\delta \rff_p(\bx_{k_p})} \inner{\bp^{(p)}_i}{\delta \rff_p(\by_{j_p})} \cond \bW}\bigg)
        \\
        &\stackrel{\text{(g)}}{=}
        \sum_{\substack{\bi,\bk \in \Delta_m({\abs{\bx}})\\\bj,\bl \in \Delta_m({\abs{\by}})}} \prod_{p=1}^m
        &&\bigg(
        \inner{\delta \rff_p(\bx_{i_p})}{\delta \rff_p(\by_{j_p})} \inner{\delta \rff_p(\bx_{k_p})}{\delta \rff_p(\by_{l_p})}
        \\
        &&&+
        \inner{\delta \rff_p(\bx_{i_p})}{\delta \rff_p(\bx_{k_p})} \inner{\delta \rff_p(\by_{j_p})}{\delta \rff_p(\by_{l_p})}
        \\
        &&&+
        \inner{\delta \rff_p(\bx_{i_p})}{\delta \rff_p(\by_{l_p})} \inner{\delta \rff_p(\bx_{k_p})}{\delta \rff_p(\by_{j_p})} \bigg)
        \\
        &\stackrel{\text{(h)}}{\leq} \mathrlap{
        \sum_{\substack{\bi,\bk \in \Delta_m({\abs{\bx}})\\\bj,\bl \in \Delta_m({\abs{\by}})}} 3^m \prod_{p=1}^m
        \norm{\delta \rff_p(\bx_{i_p})} \norm{\delta \rff_p(\by_{j_p})} \norm{\delta \rff_p(\bx_{k_p})}\norm{\delta \rff_p(\by_{l_p})}}
        \\
        &\stackrel{\text{(i)}}{=} \mathrlap{
        3^m \pars{\sum_{\bi \in \Delta_m({\abs{\bx}})} \prod_{p=1}^m \norm{\delta \rff_p(\bx_{i_p})}}^2
        \pars{\sum_{\bj \in \Delta_m({\abs{\by}})} \prod_{p=1}^m
        \norm{\delta \rff_p(\by_{j_p})}}^2}
        \\
        &\stackrel{\text{(j)}}{\leq} \mathrlap{
        \frac{1}{(m!)^4} {\pars{\frac{3 \norm{\bx}_\onevar^2 \norm{\by}_\onevar^2}{\dimRFF^2}}}^m \prod_{p=1}^m \norm{\bW^{(p)}}^4_2,}
    \end{align}
    where (e) follows from linearity of expectation and independence of the $\bp_i^{(p)}$'s for $p \in [m]$, (f) from Isserlis' theorem \cite{isserlis1918formula}, (g) is the same as (a)-(d), (h) is the Cauchy-Schwarz inequality, (i) from factorizing the summation, (j) is the same as Lemma \ref{lem:1var_rffsig_norm}. 
    
    Therefore, we have due to Lemma \ref{lem:krffsig_bound} for the variance that
    \begin{align}
        \bbV\bracks{A_{m, i} \,\vert\, \bW} 
        = \expe{A_{m, i}^2 \,\vert\, \bW} - \bbE^2\bracks{A_{m, i} \,\vert\, \bW}
        \leq 
        \frac{3^m + 1}{(m!)^4} {\pars{\frac{\norm{\bx}_\onevar^2 \norm{\by}_\onevar^2}{\dimRFF^2}}}^m \prod_{p=1}^m \norm{\bW^{(p)}}^4_2.
    \end{align}
    Let $\beta_m(\bx, \by) \coloneqq \frac{3^m + 1}{(m!)^4} \norm{\bx}_\onevar^{2m} \norm{\by}_\onevar^{2m}$. Then, as $\rffsigkernelTRP[m](\bx, \by) \,\vert\, \bW$ is a sample average,
    \begin{align}
        \bbV\bracks{\rffsigkernelTRP[m](\bx, \by) \cond \bW} \leq \frac{\beta_m(\bx, \by)}{\dimTRP \dimRFF^{2m}} \prod_{p=1}^m \norm{\bW^{(p)}}^4_2. \label{eq:cond_variance}
    \end{align}

    \emph{Conditional bound.} Since $\rffsigkernelTRP[m](\bx,\by) \,\vert\, \bW$ is a Gaussian polynomial of degree-$2m$, with expectation $\rffsigkernel[m](\bx, \by)$, and variance \eqref{eq:cond_variance}, we have by Theorem \ref{thm:hyper_concentration} for $\epsilon > 0$ that
    \begin{align}
        \bbP\bracks{\abs{\rffsigkernelTRP[m](\bx, \by) - \rffsigkernel[m](\bx, \by)} \geq \epsilon \cond \bW} 
        &\leq 2 \exp\pars{-\frac{\epsilon^{\frac{1}{m}}}{2\sqrt{2}e \bbV^{\frac{1}{2m}}\bracks{\rffsigkernelTRP[m](\bx, \by) \cond \bW}}}
        \\
        &\leq 2 \exp\pars{- \frac{\dimTRP^{\frac{1}{2m}} \dimRFF \epsilon^{\frac{1}{m}}}{2\sqrt{2}e \beta_m^{\frac{1}{2m}}(\bx, \by) \prod_{p=1}^m \norm{\bW^{(p)}}^\frac{2}{m}_2}}. \label{eq:conditional_bound}
    \end{align}
    
    \emph{Undoing the conditioning.} We take the expectation in \eqref{eq:conditional_bound} so that
    \begin{align}
        \bbP\bracks{\abs{\rffsigkernelTRP[m](\bx, \by) - \rffsigkernel[m](\bx, \by)} \geq \epsilon} &= \bbE\bracks{\prob{\abs{\rffsigkernelTRP[m](\bx, \by) - \rffsigkernel[m](\bx, \by)} \geq \epsilon \cond \bW}}
        \\
        &\leq 2\expe{\exp\pars{- \frac{\dimTRP^{\frac{1}{2m}} \dimRFF \epsilon^{\frac{1}{m}}}{2\sqrt{2}e \beta_m(\bx, \by)^{\frac{1}{2m}} \prod_{p=1}^m \norm{\bW^{(p)}}^\frac{2}{m}_2}}}
        \\
        &\stackrel{\text{(k)}}{\leq} 2\expe{\exp\pars{- \frac{\lambda \dimTRP^{\frac{1}{2m}} \epsilon^{\frac{1}{m}}}{2\sqrt{2}e \beta_m(\bx, \by)^{\frac{1}{2m}}} \frac{\dimRFF}{  \lambda \prod_{p=1}^m \norm{\bW^{(p)}}^\frac{2}{m}_2}}} 
        \\
        &\stackrel{\text{(l)}}{\leq} 2\expe{\exp\pars{- 2\pars{\frac{\lambda \dimTRP^{\frac{1}{2m}} \epsilon^{\frac{1}{m}}}{2\sqrt{2}e \beta_m(\bx, \by)^{\frac{1}{2m}}}}^\frac{1}{2} + \frac{\lambda \prod_{p=1}^m \norm{\bW^{(p)}}^\frac{2}{m}_2}{\dimRFF}}}
        \\
        &\stackrel{\text{(m)}}{\leq} 2\expe{\exp\pars{- \pars{\frac{\sqrt{2}\lambda \dimTRP^{\frac{1}{2m}} \epsilon^{\frac{1}{m}}}{e \beta_m(\bx, \by)^{\frac{1}{2m}}}}^\frac{1}{2} + \frac{\lambda \sum_{p=1}^m \norm{\bW^{(p)}}^2_2}{m\dimRFF}}}, \label{eq:expected_trp_prob_decomp}
    \end{align}
    where (k) follows form multiplying and dividing with a $\lambda > 0$, (l) from applying Lemma \ref{lem:reverse} with $p=\frac{1}{2}$ and $q = 1$, and (m) from the arithmetic-geometric mean inequality. 

    Bounding the MGF of ${w_{i,j}^{(p)}}^2$ for $p \in [m], i \in [d], j \in [\dimRFF]$, we have that 
    \begin{align}
    \bbE\bracks{\exp\pars{\lambda{w^{(p)}_{i,j}}^2}}
    &\stackrel{\text{(n)}}{\leq} \sum_{k \geq 0} \bbE\bracks{{w_{i,j}^{(p)}}^k} \frac{\lambda^k}{k!}
    \stackrel{\text{(o)}}{\leq} 1 + \lambda S + \frac{\lambda^2 S^2}{2} \sum_{k \geq 0} (\lambda R)^k
    \\
    &\stackrel{\text{(p)}}{=} 1 + \lambda S + \frac{\lambda^2 S^2}{2} \frac{1}{1 - \lambda R} \stackrel{\text{(q)}}{=} 1 + \frac{S}{2R} + \frac{S^2}{4R^2},
    \end{align}
    where (n) is the Taylor expansion, (o) is the condition \eqref{eq:rfsf_trp_approx_cond} and applying Jensen inequality to the degree-$1$ term, (p) is the geometric series for $\lambda < \frac{1}{R}$, and (q) is choosing $\lambda = \frac{1}{2R}$. Hence,
    \begin{align}
        \expe{\exp\pars{\frac{\lambda \sum_{p=1}^m \norm{\bW^{(p)}}^2_2}{m\dimRFF}}}
        &=
        \expe{\exp\pars{\frac{\lambda \sum_{p=1}^m \sum_{i=1}^d \sum_{j=1}^{\dimRFF} {w_{i, j}^{(p)}}^2}{m \dimRFF}}}
        \\
        &\stackrel{\text{(r)}}{\leq} 
        \bbE^{1/(m \dimRFF)}\bracks{\exp\pars{\lambda \sum_{p=1}^m \sum_{i=1}^d \sum_{j=1}^{\dimRFF} {w_{i, j}^{(p)}}^2}}
        \\
        &\stackrel{\text{(s)}}{\leq} 
        \pars{1 + \frac{S}{2R} + \frac{S^2}{4R^2}}^d,
    \end{align}
    where (r) is due to the Jensen inequality (Lemma \ref{lem:jensen}), and (s) follows from the independence of the $w_{i,j}^{(p)}$'s for $p \in [m], i \in [d], j \in [\dimRFF]$.
    Finally, plugging this into \eqref{eq:expected_trp_prob_decomp}, we get that
    \vspace{-10pt}
    \begin{align}
        \bbP\bracks{\abs{\rffsigkernelTRP[m](\bx, \by) - \rffsigkernel[m](\bx, \by)} \geq \epsilon}
        \leq
        2\pars{1 + \frac{S}{2R} + \frac{S^2}{4R^2}}^d
        \exp\pars{- \pars{\frac{ \dimTRP^{\frac{1}{2m}} \epsilon^{\frac{1}{m}}}{\sqrt{2}e R \beta_m(\bx, \by)^{\frac{1}{2m}}}}^\frac{1}{2}}.
    \end{align}
    \vspace*{-10pt}
    
    Finally, note that $\beta_m(\bx, \by)^\frac{1}{2m} = \pars{\frac{3^m + 1}{(m!)^4}}^\frac{1}{2m} \norm{\bx}_{\onevar} \norm{\by}_{\onevar} \leq \frac{2e^2}{m^2} \norm{\bx}_{\onevar} \norm{\by}_{\onevar}$, since $3^m + 1 \leq 4^m$ for $m \geq 1$, and $\frac{1}{m!} \leq \pars{\frac{e}{m}}^m$ due to Stirling's approximation.
\end{proof}

\section{Algorithms} \label{apx:algs}
We adopt the following notation for describing vectorized algorithms from \cite{kiraly2019kernels,toth2021seq2tens}. For arrays, $1$-based indexing is used. Let $A$ and $B$ be $k$-fold arrays with shape $(n_1 \times \dots \times n_k)$, and let $i_j \in [n_j]$ for $j \in [k]$. We define the following array operations:
  \begin{enumerate}[label=(\roman*)]
  	\item  The cumulative sum along axis $j$:
  	\begin{align}
  		A[\dots, :, \boxplus, :, \dots][\dots, i_{j-1}, i_j,, i_{j+1} \dots] := \sum_{\kappa=1}^{i_j} A[\dots, i_{j-1}, \kappa, i_{j+1}, \dots].
  	\end{align}
  	\item The slice-wise sum along axis $j$:
  	\begin{align}
	  	A[\dots, :, \Sigma, :, \dots][\dots, i_{j-1}, i_{j+1}, \dots] := \sum_{\kappa=1}^{n_j} A[\dots, i_{j-1}, \kappa, i_{j+1}, \dots].
  	\end{align}
  	\item The shift along axis $j$ by $+m$ for $m \in \bbZ+$:
  	\begin{align}
	  	A&[\dots, :, +m, :, \dots][\dots, i_{j-1}, i_j, i_{j+1}, \dots] 
            \\
            &:=
            \left\lbrace\begin{array}{ll} A[\dots, i_{j-1}, i_j-m, i_{j+1}, \dots], \quad&\text{if}\spc i_j > m \\ 0 \quad&\text{if}\spc i_j \leq m.  \end{array}\right.
  	\end{align}
  	\item The Hadamard product of arrays $A$ and $B$:
  	\begin{align}
  	 (A \odot B) [i_1, \dots, i_k] := A[i_1, \dots, i_k] B[i_1, \dots, i_k].
    \end{align}
        \item Now, if $A$ has shape $(n_1 \times \cdots \times n_j \times \cdots \times n_k)$ and $B$ has shape $(n_1 \times \cdots \times n_j^\prime \times \cdots \times n_k)$, then their (batch) outer product along axis $j$ is defined for $i_j \in [n_j n_j^\prime]$ as
        \begin{align}
            (A \boxtimes_j B) [i_1, \dots, i_j, \dots, i_k] := A[i_1, \cdots, \lceil i_j / n_j^\prime\rceil, \cdots i_k] B[i_1, \dots, i_j \bmod n_j \dots, i_k],
        \end{align}
        where $\lceil \cdot \rceil: \bbR \to \bbZ$ refers to the ceiling operation, and $\pars{\cdot \bmod n}: \bbZ \to [n]$ to the modulo $n$ operation that maps onto $[n]$ for $n \in \bbZ$.  
  \end{enumerate}

\begin{algorithm}[H]
    \begin{footnotesize}
	\caption{Computing the \RFSF{} map $\rffsig[\leq M]$.}
	\label{alg:rfsf}
	\begin{algorithmic}[1]
		\STATE {\bfseries Input:}  Sequences $\bX=(\bx_i)_{i=1}^N \subset \seq$, measure $\Lambda$, truncation $M \in \bbZ_+$, \RFF{} sample size $\dimRFF \in \bbZ_+$
		\STATE Optional: Add time-parametrization $\bx_i \gets (\bx_{i, t}, t / \ell_{\bx_i})_{t=1}^{\ell_{\bx_i}}$ for all $i \in [N]$
		\STATE Tabulate to uniform length $\ell \coloneqq \max_{j \in [N]} \ell_{\bx_j}$ by $\bx_i \gets (\bx_{i, 1}, \ldots, \bx_{i, \ell_{\bx_i}}, \ldots, \bx_{i, \ell_{\bx_i}})$ for all $i \in [N]$
		\STATE Sample independent \RFF{} weights $\bW^{(1)}, \dots, \bW^{(M)} \stackrel{\iid}{\sim} \Lambda^{\dimRFF}$
		\STATE Initialize an array $U$ with shape $[M, N, \ell-1, 2\dimRFF]$
		\STATE Compute increments $U[m, i, t, :] \gets \rff_m(\bx_{i,t+1}) - \rff_m(\bx_{i, t})$ for $m \in [M]$, $i \in [N]$, $t \in [\ell-1]$
        \STATE Initialize array $V \gets U[1, :, :, :]$
        \STATE Collapse into level-$1$ features $P_1 \gets V[:, \Sigma, :]$
		\FOR{$m=2$ {\bfseries to} $M$}
		\STATE Update with next increment $V \gets V[:, \boxplus+1, :] \boxtimes_{3} U[m, :, :, :]$
		\STATE Collapse into level-$m$ features $P_m \gets V[:, \Sigma, :]$
		\ENDFOR
		\STATE {\bfseries Output:} Arrays of \RFSF{} features per signature level $P_1, \dots, P_M$.
	\end{algorithmic}
    \end{footnotesize}
\end{algorithm}

\vspace{-5pt}

\begin{algorithm}[H]
    \begin{footnotesize}
	\caption{Computing the \RFSFD{} map $\rffsigDP[\leq M]$.}
	\label{alg:rsfsdp}
	\begin{algorithmic}[1]
		\STATE {\bfseries Input:} Sequences $\bX=(\bx_i)_{i=1}^{N} \subset \seq$, measure $\Lambda$, truncation $M \in \bbZ_+$, \RFSFD{} sample size $\dimRFF \in \bbZ_+$
		\STATE Optional: Add time-parametrization $\bx_i \gets (\bx_{i, t}, t / \ell_{\bx_i})_{t=1}^{\ell_{\bx_i}}$ for all $i \in [N]$
		\STATE Tabulate to uniform length $\ell \coloneqq \max_{j \in [N]} \ell_{\bx_j}$ by $\bx_i \gets (\bx_{i, 1}, \ldots, \bx_{i, \ell_{\bx_i}}, \ldots, \bx_{i, \ell_{\bx_i}})$ for all $i \in [N]$
		\STATE Sample independent \RFF{} weights $\bW^{(1)}, \dots, \bW^{(M)} \stackrel{\iid}{\sim} \Lambda^{\dimRFF}$
		\STATE Initialize an array $U$ with shape $[M, N, \ell-1, \dimRFF, 2]$
		\STATE Compute increments $U[m, i, t, k, :] \gets \hat\kernelfeatures_{m, k}(\bx_{i,t+1}) - \hat\kernelfeatures_{m, k}(\bx_{i, t})$ for $m \in [M]$, $i \in [N]$, $t \in [\ell-1]$, $k \in [\dimRFF]$
        \STATE Initialize array $V \gets \frac{1}{\sqrt{\dimRFF}} U[1, :, :, :, :]$
        \STATE Collapse into level-$1$ features $P_1 \gets V[:, \Sigma, :, :]$
		\FOR{$m=2$ {\bfseries to} $M$}
		\STATE Update with next increment $V \gets V[:, \boxplus+1, :, :] \boxtimes_4 U[m, :, :, :, :]$
		\STATE Collapse into level-$m$ features $P_m \gets V[:, \Sigma, :, :]$
		\ENDFOR
		\STATE {\bfseries Output:} Arrays of \RFSFD{} features per signature level $P_1, \dots, P_M$.
	\end{algorithmic}
    \end{footnotesize}
\end{algorithm}

\vspace{-5pt}

\begin{algorithm}[H]
    \begin{footnotesize}
	\caption{Computing the \RFSFT{} map $\rffsigTRP[\leq M]$.}
	\label{alg:rsfstrp}
	\begin{algorithmic}[1]
		\STATE {\bfseries Input:} Sequences $\bX=(\bx_i)_{i=1}^N \subset \seq$, measure $\Lambda$, truncation $M \in \bbZ_+$, \RFSF{} and \TRP{} sample size $\dimRFF \in \bbZ_+$
		\STATE Optional: Add time-parametrization $\bx_i \gets (\bx_{i, t}, t / \ell_{\bx_i})_{t=1}^{\ell_{\bx_i}}$ for all $i \in [N]$
		\STATE Tabulate to uniform length $\ell \coloneqq \max_{j \in [N]} \ell_{\bx_j}$ by $\bx_i \gets (\bx_{i, 1}, \ldots, \bx_{i, \ell_{\bx_i}}, \ldots, \bx_{i, \ell_{\bx_i}})$ for all $i \in [N]$
		\STATE Sample independent \RFF{} weights $\bW^{(1)}, \dots, \bW^{(M)} \stackrel{\iid}{\sim} \Lambda^{\dimRFF}$
		\STATE Sample standard normal matrices  $\bP^{(1)}, \dots, \bP^{(M)} \stackrel{\iid}{\sim} \cN^{\dimRFF}(0, \b I_{2\dimRFF})$
            \STATE Initialize an array $U$ with shape $[M, N, \ell-1, \dimRFF]$
		\STATE Compute projected increments $U[m, i, t, :] \gets {\bP^{(m)}}^\top \pars{\rff_m(\bx_{i,t+1}) - \rff_m(\bx_{i, t})}$ for $m \in [M]$, $i \in [N]$, $t \in [\ell-1]$
        \STATE Initialize array $V \gets \frac{1}{\sqrt{\dimRFF}} U[1, :, :, :]$
        \STATE Collapse into level-$1$ features $P_1 \gets V[:, \Sigma, :]$
		\FOR{$m=2$ {\bfseries to} $M$}
		\STATE Update with next increment $V \gets V[:, \boxplus+1, :] \odot U[m, :, :, :]$
		\STATE Collapse into level-$m$ features $P_m \gets V[:, \Sigma, :]$
		\ENDFOR
		\STATE {\bfseries Output:} Arrays of \RFSFT{} features per signature level $P_1, \dots, P_M$.
	\end{algorithmic}
    \end{footnotesize}
\end{algorithm}

\end{document}